%% file: main.tex
\documentclass[english]{article}
\pdfoutput=1

     \usepackage[nonatbib,final]{neurips_2021}

\usepackage[T1]{fontenc}
\usepackage[latin9]{inputenc}
\usepackage[colorlinks = true,
            linkcolor = red,
            urlcolor  = blue,
            citecolor = blue,
            anchorcolor = blue]{hyperref}
\usepackage{url}            %
\usepackage{booktabs}       %
\usepackage{amsfonts}       %
\usepackage{nicefrac}       %
\usepackage{microtype}      %
\usepackage{lipsum}
\usepackage[dvipdfmx]{graphicx}
\usepackage{bmpsize}
\usepackage{amsmath}
\usepackage{cleveref}
\usepackage{empheq}
\usepackage[dvipsnames]{xcolor}
\usepackage{bbold}
\usepackage[numbers, compress, sort]{natbib}
\usepackage{subcaption}
\usepackage{amssymb}
\usepackage{url}
\usepackage{babel}
\usepackage[shortlabels]{enumitem}
\usepackage{verbatim}
\usepackage{amsthm}
\usepackage{thmtools, thm-restate}
\usepackage{kky}
\usepackage{tikz}
\usetikzlibrary{shapes.geometric, arrows, shadows, bayesnet}
\usepackage{nccmath}
\usepackage{titlesec}
\usepackage{colour-blind}
\usepackage{wrapfig}
\usepackage[abs]{overpic}
\usepackage{xcolor,varwidth}
\usepackage{relsize}

\titlespacing*{\section}{0pt}{0.1ex plus .025ex minus .05ex}{0.05ex plus .025ex minus .05ex}
\titlespacing*{\subsection}{0pt}{0.1ex plus .05ex minus .05ex}{0.05ex plus .05ex minus .05ex}
\setlist[itemize]{leftmargin=*,itemsep=0em}
\setlist[enumerate]{leftmargin=*,itemsep=0em}
\crefname{figure}{Fig.}{Figs.}
\crefname{definition}{Defn.}{Defns.}
\crefname{corollary}{Corollary}{Corollaries}
\crefname{lemma}{Lemma}{Lemmas}
\crefname{proposition}{Prop.}{Props.}
\crefname{theorem}{Thm.}{Thms.}
\crefname{remark}{Remark}{Remarks}
\crefname{principle}{Principle}{Principles}
\crefname{lemma}{Lemma}{Lemmas}
\crefname{table}{Tab.}{Tabs.}
\crefname{section}{\S}{\S\S}
\crefname{subsection}{\S}{\S\S}
\crefname{subsubsection}{\S}{\S\S}

\definecolor{lightgray}{RGB}{196,196,196}
\definecolor{lightblue}{RGB}{34,94,196}
\definecolor{darkred}{RGB}{133,46,43}
\definecolor{darkgreen}{RGB}{0,163,0}
\definecolor{blue_cblind}{HTML}{1A85FF}
\definecolor{red_cblind}{HTML}{D41159}

\newcommand{\luigi}[1]{{\color{darkgreen} ~\textbf{Luigi:} #1}}
\newcommand{\julius}[1]{\textcolor{purple}{~\textbf{Julius:} #1}}

\newcommand{\tochange}[1]{{\color{black} { #1} }}

\newcommand{\paul}[1]{\textcolor{blue}{~\textbf{Paul:} #1}}
\renewcommand{\paul}[1]{}

\makeatletter
\newcommand{\printfnsymbol}[1]{%
  \textsuperscript{\@fnsymbol{#1}}%
}
\makeatother

\title{Independent mechanism analysis, a new concept?}
\author{
Luigi Gresele\thanks{Equal contribution. Code available at: \href{https://github.com/lgresele/independent-mechanism-analysis}{https://github.com/lgresele/independent-mechanism-analysis}}\, $^{1}$
\And Julius von K\"ugelgen\footnotemark[1]\, $^{1,2}$
\And Vincent Stimper $^{1,2}$
\AND Bernhard Sch\"olkopf $^{1}$
\quad\quad \quad  Michel Besserve $^{1}$ \\[0.5em]
$^1$ Max Planck Institute for Intelligent Systems, T\"ubingen, Germany
\quad 
$^2$ University of Cambridge \\[.25em]
\texttt{\{luigi.gresele,jvk,vincent.stimper,bs,besserve\}@tue.mpg.de}
}

\begin{document}

\maketitle
\vspace{-1.5em}
\begin{abstract}
\vspace{-0.5em}
\input{sections/0_abstract}
\end{abstract}
\setcounter{footnote}{0} 

\vspace{-.5em}

\section{Introduction}
\label{sec:introduction}
\input{sections/1_introduction}

\section{Background and preliminaries}
\label{sec:background}
\input{sections/2_1_background_ICA}
\input{sections/2_2_background_causality}
\section{Existing ICM measures are insufficient for nonlinear ICA}
\label{sec:unsuitability_of_existing_ICM_measures}
\input{sections/3_unsuitability_of_existing_ICM_measures}
\section{Independent mechanism analysis (IMA)}
\label{sec:IMA}
\input{sections/4_IMA}

\section{Experiments}
\label{sec:experiments}
\input{sections/5_experiments}

\section{Discussion}
\label{sec:discussion}
\input{sections/6_discussion}
\clearpage
\section*{Acknowledgements}
The authors thank Aapo Hyv\"arinen, Adri\'an Javaloy Born\'as, Dominik Janzing, Giambattista Parascandolo, Giancarlo Fissore, Nasim Rahaman, Patrick Burauel, Patrik Reizinger, Paul Rubenstein, Shubhangi Ghosh, and the anonymous reviewers for helpful comments and discussions. 

\section*{Funding Transparency Statement}
This work was supported by the German Federal Ministry of Education and Research (BMBF): T\"ubingen AI Center, FKZ: 01IS18039B; and by the Machine Learning Cluster of Excellence, EXC number 2064/1 - Project number 390727645.

{\small
\bibliographystyle{plainnat}
\bibliography{main}
}
\input{sections/checklist}

\clearpage
\appendix
\begin{center}
{\centering \LARGE APPENDIX}
\vspace{0.8cm}
\sloppy
\end{center}

\section*{Overview}
\begin{itemize}
    \item \Cref{app:identifiability_and_linear_ICA} contains further elaboration on the notion of identifiability as used in the present work, as well as connections to linear ICA.
    \item \Cref{app:trace} contains additional discussion of existing ICM criteria and their relation to IMA.
    \item \Cref{app:proofs} presents the full proofs for all theoretical results from the main paper.
    \item \Cref{app:examples} contains a worked out computation of the value of $C_\IMA$ for the mapping from radial to Cartesian coordinates.
    \item \Cref{app:experiments} contains experimental details and additional results.
    \item \Cref{app:confmaps} contains additional background on conformal maps and M\"obius transformations
\end{itemize}

\section{Additional background on identifiability and linear ICA}
\label{app:identifiability_and_linear_ICA}
\input{appendix/A_identifiability_and_linear_ICA}
\clearpage
\section{Existing ICM criteria and their relationship to ICA and IMA }
\label{app:trace}

\input{appendix/B_trace_method_igci}

\clearpage
\section{Proofs}
\label{app:proofs}
\input{appendix/C_proofs}

\clearpage
\section{Worked out example}
\label{app:examples}
\input{appendix/D_examples}

\clearpage
\section{Experiments}
\label{app:experiments}

\input{appendix/E_experiments}

\clearpage
\section{Additional background on conformal maps and M\"obius transformations}
\label{app:confmaps}

\input{appendix/F_background_conf_map}
\begin{comment}
\clearpage
\textbf{WORK IN PROGRESS THEORY:}

%
%

\input{appendix/polar}
\input{appendix/identifiability}
\bigbreak

\clearpage
\textbf{OTHER OLD APPENDICES WITH POTENTIALLY USEFUL NOTES:}

\input{appendix/examples_darmois}
\bigbreak

\input{appendix/proofs_adm_identif}
\bigbreak
\end{comment}

%
%

%
%

%
%

%
%

%
%

%
%

%
%

\end{document}

%% file: sections/0_abstract.tex
Independent component analysis provides a principled framework for %
unsupervised representation learning, with solid theory on the identifiability %
of the 
latent code that generated the data, given only observations of mixtures thereof. 
Unfortunately, when the mixing is nonlinear, the model is provably nonidentifiable, since statistical independence alone does not sufficiently constrain the problem. 
Identifiability can be recovered in settings where additional, typically observed variables are included in the generative process. %
We investigate an alternative path and consider instead including assumptions reflecting the principle of \textit{independent causal mechanisms} exploited in the field of causality. 
Specifically, our approach is motivated by thinking of each source as independently influencing  
the mixing process.
This gives rise to a framework which we term independent mechanism analysis. We provide theoretical and empirical evidence that our approach circumvents a number of nonidentifiability issues arising in nonlinear
blind source separation.

%% file: sections/1_introduction.tex
One of the goals of unsupervised learning is to uncover properties of the data generating process, such as latent structures
giving rise to the observed data.
Identifiability~\cite{lehmann2006theory} 
formalises this desideratum: 
under suitable assumptions, a model learnt from observations should match the ground truth, 
up to 
well-defined ambiguities.
Within
representation learning, identifiability has been 
studied mostly in the context of
independent component analysis
(ICA)~\cite{comon1994independent, ICAbook},
which assumes that the observed data~$\xb$ results from mixing
unobserved \textit{independent} random variables~$s_i$ referred to as \textit{sources}. The aim is to recover the sources based on the observed mixtures alone, also termed \textit{blind source separation}~(BSS).
A major obstacle to BSS 
is that, in the
nonlinear case, independent component estimation does not necessarily correspond to recovering the \textit{true} sources: %
it is 
possible to give counterexamples where the observations are transformed into components $y_i$ which %
are independent, yet still mixed with respect to %
the true sources $s_i$~\cite{darmois1951construction, hyvarinen1999nonlinear, taleb1999source}.
In other words, nonlinear ICA is not identifiable.

In order to achieve identifiability, a growing body of research postulates 
additional supervision or structure in the data generating process, often in the form of \textit{auxiliary variables}~\cite{hyvarinen2016unsupervised, hyvarinen2017nonlinear, hyvarinen2019nonlinear, gresele2020incomplete, halva2020hidden}. 
In the present work, we investigate a different route to identifiability by drawing inspiration from the field of \textit{causal inference}~\cite{pearl2009causality,peters2017elements} which has provided useful insights for a number of machine learning tasks, including semi-supervised~\cite{scholkopf2012causal,kugelgen2020semi}, transfer~\cite{kugelgen2019semi,greenfeld2020robust,subbaswamy2019preventing,rothenhausler2018anchor,arjovsky2019invariant,heinze2021conditional,magliacane2018domain,gong2016domain,pearl2014external,rojas2018invariant,ZhaSchMuaWan13},
reinforcement~\cite{bareinboim2015bandits,zhang2019near,lu2018deconfounding,lu2020sample,buesing2018woulda,lee2018structural,forney2017counterfactual,goyal2019recurrent}, and unsupervised~\cite{parascandolo2018learning,besserve2018group,von2020towards,scholkopf2021toward,besserve2020counterfactuals,shen2020disentangled,leeb2020structural,von2021self} learning.
To this end, we \textit{interpret the ICA mixing as a causal process} and apply the principle of independent causal mechanisms (ICM) which postulates that the generative process consists of independent modules which do not share information~\cite{janzing2010causal,scholkopf2012causal,peters2017elements}.
In this context, ``independent'' does not refer to \textit{statistical} independence 
of random variables, but rather to %
the notion that 
the distributions and functions composing the generative process are chosen independently by Nature \cite{janzing2010causal,JanChaSch16}.
While a formalisation of ICM~\cite{janzing2010causal, lemeire2013replacing} in terms of algorithmic (Kolmogorov) complexity~\cite{kolmogorov1963tables} exists, it is not computable, and hence applying ICM in practice requires 
assessing such non-statistical independence with suitable domain specific criteria \cite{SteJanSch10}.
The goal of our work is thus to \textit{constrain the nonlinear ICA problem, in particular the mixing function, via suitable ICM measures}, thereby ruling out common counterexamples to identifiability which intuitively violate the ICM principle.

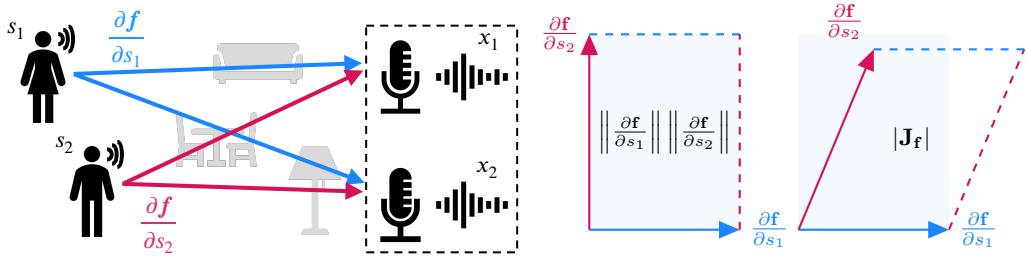
\begin{figure}
    \vspace{-.5em}
    \begin{subfigure}[b]{0.5\textwidth}
      \centering
      \includegraphics[width=\textwidth,
      ]{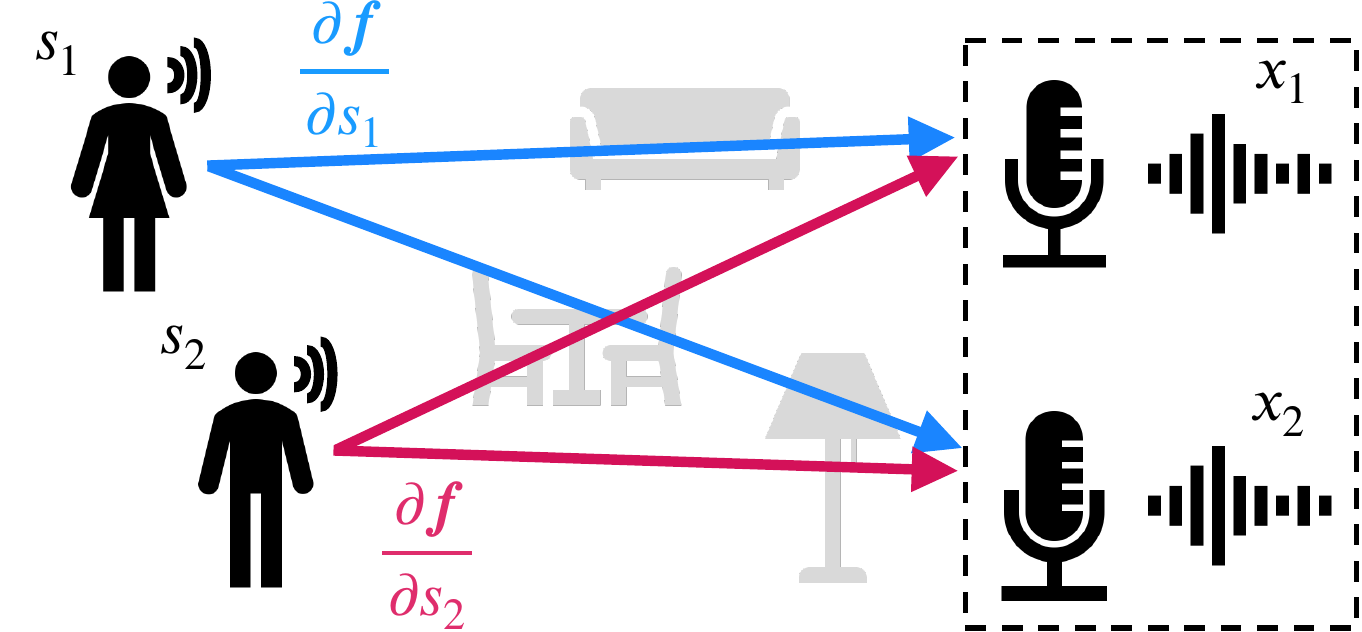}
      \vspace{-1.5em}
    \end{subfigure}%
    \begin{subfigure}[b]{0.5\textwidth}
        \centering
        \begin{tikzpicture}
            \fill[blue_cblind!5!white] (0,0) rectangle (2,2.6);
            \draw[->,thick,  color=blue_cblind] (0,0) -- (2,0) node[anchor= west] {$\frac{\partial \fb}{\partial s_1}$};
            \draw[->,thick,  color=red_cblind] (0,0) -- (0,2.6) node[anchor= east] {$\frac{\partial \fb}{\partial s_2}$} ;
            \draw[-,thick,  dashed, color=blue_cblind] (0,2.6) -- (2,2.6);
            \draw[-,thick,  dashed, color=red_cblind] (2,0) -- (2,2.6);
            \draw[] (1,1.3) node[] {\mbox{\small$\norm{\frac{\partial \fb}{\partial s_1}}\norm{\frac{\partial \fb}{\partial s_2}}$}};
        \end{tikzpicture}%
        \begin{tikzpicture}
            \fill[blue_cblind!5!white] (0,0) rectangle (2,2.6);
            \draw[->, thick, color=blue_cblind] (0,0) -- (2,0) 
            node[anchor= west] {$\frac{\partial \fb}{\partial s_1}$}
            ;
            \draw[->, thick, color=red_cblind] (0,0) -- (1,2.4) 
            node[anchor= south east] {$\frac{\partial \fb}{\partial s_2}$}
            ;
            \draw[-, thick, dashed, color=blue_cblind] (1,2.4) -- (3,2.4);
            \draw[-, thick, dashed, color=red_cblind] (2,0) -- (3,2.4);
            \draw[] (1.5,1.2) node[] {\mbox{\small $|\Jb_\fb|$}};
        \end{tikzpicture}
        \vspace{-.5em}
    \end{subfigure}%
    \caption{\small  
    \textit{(Left)} 
    For the cocktail party problem, the ICM principle \textit{as traditionally understood} would say that the content of speech $p_\sb$ is independent of the mixing or recording process $\fb$ (microphone placement, room acoustics).
    IMA refines, or extends, this idea \textit{at the level of the mixing function} by postulating that the contributions~$\nicefrac{\partial \fb}{\partial s_i}$ of each source to $\fb$, as captured
    by the speakers' positions relative to the recording process, should not be fine-tuned to each other.
    \textit{(Right)} We formalise this independence between the~$\nicefrac{\partial \fb}{\partial s_i}$, which are the columns of the Jacobian~$\Jb_\fb$, as an \textit{orthogonality condition}: the absolute value of the determinant~$|\Jb_\fb|$, i.e., the volume of the parallelepiped spanned by $\nicefrac{\partial \fb}{\partial s_i}$, should decompose as the product of the norms of the~$\nicefrac{\partial \fb}{\partial s_i}$. 
    }
    \label{fig:intuition}
    \vspace{-.75em}
\end{figure}

Traditionally, ICM criteria have been developed for 
causal discovery, where \textit{both  cause and effect are observed}~\cite{daniuvsis2010inferring,janzing2012information,janzing2010telling,zscheischler2011testing}.
They enforce an independence between (i) the cause (source) distribution and (ii) the conditional or mechanism (mixing function) generating the effect (observations), and thus rely on the fact that the \textit{observed} cause distribution is informative. 
As we will show, this renders them insufficient for nonlinear ICA, since the constraints they impose 
are satisfied by common counterexamples to identifiability.
With this in mind, we introduce a new way to characterise or \textit{refine} the ICM principle for
unsupervised representation learning tasks such as nonlinear ICA. 

\textbf{Motivating example.}
To build intuition, we turn to a famous example of ICA and BSS: the cocktail party problem, illustrated in~\cref{fig:intuition}~\textit{(Left)}.
Here, a number of conversations are happening in parallel, and the task is to recover the individual voices $s_i$ from the recorded mixtures $x_i$.
The mixing or recording process~$\fb$ is primarily determined by the
room acoustics and the locations at which microphones are placed.
Moreover, each speaker influences the recording through their positioning in the room, and we may think of this influence as $\nicefrac{\partial \fb}{\partial s_i}$.
Our independence postulate then amounts to stating that the speakers' positions are not fine-tuned to the
room acoustics and microphone placement, or to each other, i.e., \textit{the contributions $\nicefrac{\partial \fb}{\partial s_i}$ should be independent (in a non-statistical sense).%
}%
\footnote{For additional intuition and possible violations in the context of the cocktail party problem, see~\Cref{app:violations_icm_ima}.}

\textbf{Our approach.}
We formalise this notion of
independence between the contributions~$\nicefrac{\partial \fb}{\partial s_i}$ of each source 
to the mixing process
(i.e., the columns of the Jacobian matrix $\Jb_\fb$ of partial derivatives) as an orthogonality condition, see~\cref{fig:intuition}~\textit{(Right)}.  
Specifically, the absolute value of the determinant $|\Jb_\fb|$, which describes the local change in infinitesimal volume induced by mixing the sources, should factorise or decompose as the product of the norms of its columns. 
This can be seen as a decoupling of the local influence of each partial derivative in the pushforward operation (mixing function) mapping the source distribution to the observed one, and gives rise to a novel framework which we term independent mechanism analysis (IMA).
IMA can be understood as a refinement of the ICM principle that applies the idea of independence of mechanisms 
at the level of the mixing function.

\textbf{Contributions.} The structure and contributions of this paper can be summarised as follows:
\begin{itemize}[topsep=-3pt,itemsep=2pt]
    \item we review well-known obstacles to identifiability of nonlinear ICA~(\cref{sec:background_ICA}), as well as existing ICM criteria~(\cref{sec:background_causality}), and show that the latter do not sufficiently constrain nonlinear ICA~(\cref{sec:unsuitability_of_existing_ICM_measures});
    \item we propose a more suitable ICM criterion 
    for unsupervised representation learning which gives rise to a new framework that we term independent mechanism analysis (IMA) (\cref{sec:IMA}); 
    we provide geometric and information-theoretic interpretations of IMA~(\cref{sec:ima_intuition}),
    introduce an IMA contrast function which is invariant to the inherent ambiguities of nonlinear ICA%
    ~(\cref{sec:ima_definition_properties}), %
    and 
    show that it rules out a large class of counterexamples and is consistent with existing identifiability results~(\cref{sec:ima_theory});
    \item we experimentally validate our theoretical claims 
    and propose a regularised maximum-likelihood learning approach based on the IMA constrast which outperforms the unregularised baseline~(\cref{sec:experiments});
    additionally, we introduce a method to learn nonlinear ICA solutions with triangular Jacobian and a metric to assess BSS
    which can be of independent interest for the nonlinear ICA community.
\end{itemize}%

%% file: sections/2_1_background_ICA.tex
Our work builds on and connects related literature from the fields of independent component analysis~(\cref{sec:background_ICA}) and causal inference~(\cref{sec:background_causality}). 
We review the most important concepts below.

\subsection{Independent component analysis (ICA)}
\label{sec:background_ICA}
Assume the following data-generating process for independent component analysis (ICA)
\begin{equation}
\label{eq:gen}
\xb = \fb(\sbb)\,, 
\quad\quad\quad\quad\quad
\textstyle
    p_\sb(\sb) = \prod_{i=1}^n p_{s_i}(s_i)\,,
\end{equation}
where the \textit{observed mixtures} $\xb\in\RR^n$ result from applying a \textit{smooth and invertible mixing function} $\fb:\RR^n\rightarrow\RR^n$ to a set of \textit{unobserved, independent signals or sources} $\sbb \in \mathbb{R}^n$ with smooth, factorised density $p_\sb$ with connected support (see illustration~\cref{fig:ICA_graph}). 
The goal of ICA is to learn an \textit{unmixing function}
$\gb:\RR^n\rightarrow\RR^n$ such that $\yb=\gb(\xb)$ has independent components.
\textit{Blind source separation} (BSS), on the other hand, aims to recover the true unmixing $\fb^{-1}$ and thus the true sources $\sb$ (up to tolerable ambiguities, see below).
Whether performing ICA corresponds to solving BSS is related to the concept of \textit{identifiability} of the model class.
Intuitively, identifiability is the desirable property that 
\emph{all models which give rise to the same mixture
distribution should be ``equivalent'' up to certain ambiguities},
formally defined as follows.%
\begin{definition}[$\sim$-identifiability]
\label{def:identifiability}
Let $\Fcal$ be the set of all smooth, invertible functions $\fb:\RR^n\rightarrow \RR^n$,
and $\Pcal$ be the set of all
smooth, factorised densities $p_\sb$ with connected support on $\RR^n$.
Let
$\Mcal\subseteq \Fcal\times\Pcal$ be a \textit{subspace of models} and
let $\sim$ be an \textit{equivalence relation} on $\Mcal$. 
Denote by $\fb_*p_\sb$ the \textit{push-forward density} of $p_\sb$ via $\fb$.
Then the generative process~\eqref{eq:gen} is said to be $\sim$-\emph{identifiable
on $\Mcal$}
if
\begin{equation}
\label{eq:identifiability}
    \forall (\fb, p_\sb), (\fbt, p_{\sbt})\in \Mcal:
    \quad  \quad
    \fb_* p_\sb = \fbt_* p_\sbt
    \quad  \quad
    \implies 
    \quad \quad
    (\fb, p_\sb)\sim(\fbt,p_\sbt)\, .
\end{equation}%
\end{definition}%
\vspace{-0.5em}
If the true model belongs to the model class $\Mcal$, then $\sim$-identifiability ensures that any model in $\Mcal$ learnt from (infinite amounts of) data will be $\sim$-equivalent to the true one.
An example is \textit{linear} ICA which is identifiable up to permutation and rescaling of the sources on the subspace $\Mcal_\LIN$ of pairs of (i) invertible matrices (constraint on $\Fcal$) and (ii) factorizing densities for which at most one $s_i$ is Gaussian (constraint on $\Pcal$)~\cite{darmois1953analyse,skitovic1953property, comon1994independent}, see~\Cref{app:linear_ICA} for a more detailed account.
In the nonlinear case (i.e., without constraints on $\Fcal$), identifiability is much more challenging.
If $s_i$ and $s_j$ are independent, then so are
$h_i(s_i)$ and $h_j(s_j)$
for any functions $h_i$ and $h_j$.
In addition to permutation-ambiguity,
such \textit{element-wise} $\hb(\sb)=(h_1(s_1), ..., h_n(s_n))$ 
can therefore not be resolved either.
We thus define the desired form of identifiability for nonlinear BSS as follows.
\begin{definition}[$\sim_\BSS$]
\label{def:bss_identifiability}
The equivalence relation $\sim_\BSS$ on $\Fcal\times\Pcal$ defined as in~\Cref{def:identifiability} is given by
\begin{equation}
    (\fb, p_\sb)\sim_\BSS(\fbt,p_\sbt) \iff 
    \exists \Pb, \hb 
    \quad 
    \text{s.t.}
    \quad  
    (\fb, p_\sb)=(\fbt\circ \hb^{-1}\circ\Pb^{-1}, (\Pb\circ\hb)_*p_\sbt)
\end{equation}%
where $\Pb
$ is a permutation and $\hb(\sb)=(h_1(s_1), ..., h_n(s_n))$ is an invertible, element-wise function.
\end{definition}%
\paul{I found this quite confusing, isn't it the case that $[\fbt\circ \hb^{-1}\circ\Pb^{-1}]_*[(\Pb\circ\hb)_*p_\sbt] = \fbt_*p_\sbt$?}
\paul{Ah maybe on second thought that is the point :) Maybe worth stating that this is why this is an interesting definition.}
A 
fundamental obstacle---and a crucial difference to the linear problem---%
is that in the nonlinear case, different mixtures of $s_i$ and $s_j$ can be independent,
i.e.,
solving ICA
is \textit{not} equivalent to
solving BSS.
A prominent example of this is given by the \emph{Darmois construction}~\cite{darmois1951construction, hyvarinen1999nonlinear}.%
\begin{definition}[Darmois construction]
\label{def:darmois_solution}
The \emph{Darmois construction} $\gb^\mathrm{D}:\RR^n\rightarrow(0,1)^n$ 
is obtained
by recursively applying the conditional
cumulative distribution function (CDF) transform:%
\begin{equation}
\label{eq:Darmois_construction}
\textstyle
    g^\mathrm{D}_i(\xb_{1:i})
    :=\PP(X_i\leq x_i|\xb_{1:i-1})
    =\int_{-\infty}^{x_i}p(x_i'|\xb_{1:i-1}) dx_i'
    \quad \quad
    \quad \quad (i=1,...,n).
\end{equation}%
\end{definition}%
\vspace{-0.5em}
\vspace{-0.5em}
The resulting \textit{estimated} sources $\yb^\mathrm{D}=\gb^\mathrm{D}(\xb)$ are mutually-independent uniform r.v.s by construction, see~\cref{fig:IGCI_Darmois} for an illustration.
However, they need not be meaningfully
related to the \textit{true} sources $\sb$,
and will,
in general, still be a nonlinear mixing thereof~\cite{hyvarinen1999nonlinear}.\footnote{Consider, e.g., a mixing $\fb$ with full Jacobian which yields a contradiction to~\Cref{def:bss_identifiability}, due to~\Cref{remark:triangular_Jacobian}.\label{footnote:ctrxmpl}}
Denoting the mixing function corresponding to~\eqref{eq:Darmois_construction} by $\fb^\text{D}=(\gb^\text{D})^{-1}$ and the uniform density on $(0,1)^n$ by $p_\ub$, the \textit{Darmois solution} $(\fb^\text{D}, p_\ub)$
thus allows construction of counterexamples to 
$\sim_\BSS$-identifiability on $\Fcal\times\Pcal$.\footnote{
By applying a change of variables, we can see %
that the transformed variables in~\eqref{eq:Darmois_construction} are uniformly distributed in the open unit cube, thereby corresponding to independent components~\cite[][\S~2.2]{papamakarios2019normalizing}.}
\begin{remark}
\label{remark:triangular_Jacobian}
$\gb^\mathrm{D}$ has lower-triangular Jacobian, i.e., $\nicefrac{\partial g_i^\mathrm{D}}{\partial x_j}=0\,$ for~$i<j$.
Since the order of the $x_i$ is arbitrary, applying $\gb^\text{D}$ after a permutation yields a different Darmois solution.
Moreover, \eqref{eq:Darmois_construction} yields independent components $\yb^\text{D}$
even if the sources $s_i$
were not
independent
to begin with.%
\footnote{This has broad implications for unsupervised learning, as it shows that, for i.i.d. observations, not only factorised priors, but \textit{any} unconditional prior is insufficient for identifiability (see, e.g.,~\cite{khemakhem2020variational}, Appendix D.2).}%
\end{remark}%
\begin{figure}
    \begin{subfigure}[b]{0.35\textwidth}
      \centering
      \includegraphics[width=\textwidth]{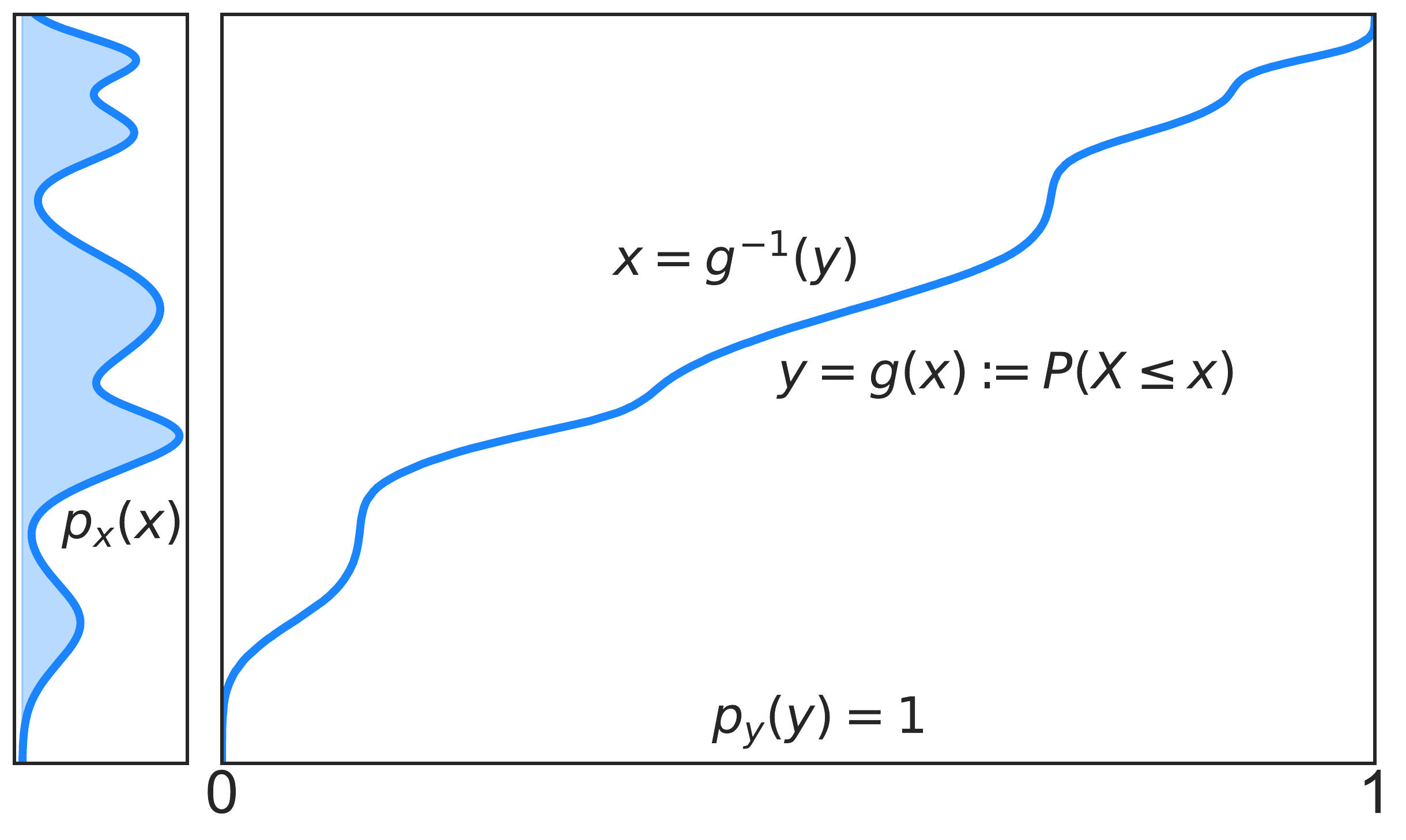}
      \vspace{-2.0em}
      \caption{CDF transform}
     \label{fig:IGCI_Darmois}
    \end{subfigure}%
    \begin{subfigure}[b]{0.225\textwidth}
        \centering
        \begin{tikzpicture}
            \centering
            \node (s1) [latent] {$s_1$};
            \node (s2) [latent, right=of s1, xshift=-1em] {$s_2$};
            \node (x1) [obs, below=of s1] {$x_1$};
            \node (x2) [obs, below=of s2] {$x_2$};
            \edge {s1,s2}{x1,x2};
        \end{tikzpicture}
        \caption{ICA}
        \label{fig:ICA_graph}
    \end{subfigure}%
    \begin{subfigure}[b]{0.2\textwidth}
        \centering
        \begin{tikzpicture}
            \centering
            \node (s1) [obs, draw=blue_cblind,line width=0.5mm] {$\cb$};
            \node (x1) [obs, below=of s1] {$\eb$};
            \edge[color=red_cblind] {s1}{x1};
        \end{tikzpicture}
        \caption{ICM}
        \label{fig:ICM_graph}
    \end{subfigure}%
    \begin{subfigure}[b]{0.225\textwidth}
        \centering
        \begin{tikzpicture}
            \centering
            \node (s1) [latent] {$s_1$};
            \node (s2) [latent, right=of s1, xshift=-1em] {$s_2$};
            \node (x) [obs, below=of s1, xshift=2em] {$\xb$};
            \edge [color=blue_cblind] {s1}{x};
            \edge [color=red_cblind] {s2}{x};
        \end{tikzpicture}
        \caption{IMA}
        \label{fig:IMA_graph}
    \end{subfigure}%
    \caption{\small (a) Any observed density $p_x$ can be mapped to a uniform $p_y$ via the CDF transform~$g(x)=\PP(X\leq x)$; Darmois solutions $(\fb^\text{D},p_\ub)$ constructed from~\eqref{eq:Darmois_construction} therefore automatically satisfy the independence postulated by IGCI~\eqref{eq:IGCI_condition}.
    (b) ICA setting with $n=2$ sources (shaded nodes are observed, white ones are unobserved).
    (c) Existing ICM criteria typically enforce independence between an observed input or cause distribution $p_\cb$ and a mechanism $p_{\eb|\cb}$ (independent objects are highlighted in blue and red). (d)  IMA enforces independence between the contributions of different sources $s_i$ to the mixing function $\fb$ as captured by $\nicefrac{\partial \fb}{\partial s_i}$.
    }
    \label{fig:graphical_models}
\end{figure}
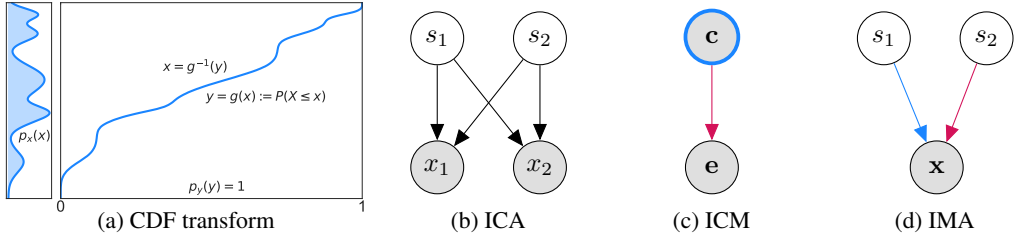
\vspace{-0.5em}
Another 
well-known 
obstacle to identifiability
are
\emph{measure-preserving automorphisms} (MPAs) of the source distribution $p_\sb$: these are 
functions 
$\ab
$ 
which map the source space to itself without affecting its distribution, i.e., $\ab_* p_\sb=p_\sb$~\cite{hyvarinen1999nonlinear}.
A particularly
instructive class of MPAs 
is the following~\cite{locatello2019challenging,khemakhem2020variational}.%
\begin{definition}[``Rotated-Gaussian'' MPA]
\label{def:measure_preserving_automorphism_Gaussian}
Let $\Rb\in O(n)$ be an orthogonal matrix, and denote by 
$\Fb_\sb(\sb)=(F_{s_1}(s_1), ..., F_{s_n}(s_n))$ and
$\bm\Phi(\zb)=(\Phi(z_1), ..., \Phi(z_n))$
the element-wise
CDFs
of a smooth, factorised density $p_\sb$
and of a Gaussian, respectively.
Then the ``rotated-Gaussian'' MPA $\ab^{\Rb}(p_\sb)
$
is
\begin{equation}
\label{eq:measure_preserving_automorphism_Gaussian}
\textstyle
    \ab^{\Rb}(p_\sb) =\Fb_\sb^{-1} \circ \bm\Phi \circ \Rb \circ \bm\Phi^{-1} \circ \Fb_\sb\,.
\end{equation}%
\end{definition}%
\vspace{-0.5em}
$\ab^{\Rb}(p_\sb)$ 
first
maps
to the
(rotationally invariant)
standard isotropic Gaussian (via $\bm\Phi^{-1} \circ \Fb_\sb$), then applies
a rotation, and finally maps back, 
without affecting the distribution of the estimated sources.
Hence, if $(\fbt, p_\sbt)$ is a valid solution,
then so is~$(\fbt\circ\ab^\Rb(p_\sbt),p_\sbt)$ for any $\Rb\in O(n)$.
Unless $\Rb$ is a permutation, this constitutes another common counterexample to $\sim_\BSS$-identifiability on $\Fcal\times\Pcal$.
Identifiability results for nonlinear ICA have recently been established for settings where an 
auxiliary variable $\ub
$ (e.g., environment index, time stamp, class label) renders the sources \emph{conditionally} independent~\cite{hyvarinen2016unsupervised, hyvarinen2017nonlinear, hyvarinen2019nonlinear,khemakhem2020variational}.
The assumption on $p_\sb$
in~\eqref{eq:gen} is replaced with
    $p_{\sbb|\ub}(\sbb|\ub)=\prod_{i=1}^n p_{s_i|\ub}(s_i|\ub)$, thus restricting $\Pcal$ in~\Cref{def:identifiability}.
    In most cases,
    $\ub$ is assumed to be observed, though~\cite{halva2020hidden} is a notable exception.
    Similar results exist given access to a second noisy view $\tilde{\xb}$~\cite{gresele2020incomplete}.

%% file: sections/2_2_background_causality.tex
\subsection{Causal inference and the principle of independent causal mechanisms (ICM)}
\label{sec:background_causality}
Rather than relying only on additional assumptions on $\Pcal$ (e.g., via auxiliary variables), we seek to further constrain~\eqref{eq:gen}
by also placing assumptions on the
set $\Fcal$ of 
mixing functions~$\fb$.
To this end, we draw inspiration from the field of causal inference~\cite{pearl2009causality, peters2017elements}.
Of central importance to our approach is the \emph{Principle of Independent Causal Mechanisms} (ICM)~\cite{lemeire2006causal,janzing2010causal,scholkopf2012causal}.%
\begin{principle}[ICM principle \cite{peters2017elements}]
\label{principle:ICM}
The causal generative process of a system's variables is composed of autonomous modules that do not inform or influence each other.
\end{principle}%
\vspace{-0.25em}
These ``modules'' are typically thought of as the conditional distributions of each variable given its direct causes.
Intuitively, the principle then states that these \textit{causal conditionals} correspond to \emph{independent mechanisms of nature} which do not share information.
Crucially, here ``independent'' does not refer to \textit{statistical} independence of random variables, but rather to independence of the underlying distributions as \textit{algorithmic} objects.
For a bivariate system comprising a cause $\cb$ and an effect $\eb$, this idea reduces to an independence of cause and mechanism, see~\Cref{fig:ICM_graph}. One way to formalise ICM uses Kolmogorov complexity $K(\cdot)$~\cite{kolmogorov1963tables} as a measure of algorithmic information%
~\cite{janzing2010causal}. %
However, since Kolmogorov complexity is is not computable, using ICM in practice requires
assessing~\cref{principle:ICM} with other suitable proxy criteria%
~\cite{hoyer2008nonlinear, janzing2010telling, zscheischler2011testing, peters2014causal, peters2014identifiability,
shajarisales2015telling, mooij2016distinguishing, peters2016causal, peters2017elements, blobaum2018cause, besserve2018group,janzing2021causal}.%
\footnote{``This can be seen as an algorithmic analog of replacing the empirically undecidable question of statistical independence with practical independence tests
that are based on
assumptions on the underlying distribution''~\cite{janzing2010causal}.}
Allowing for deterministic relations between cause (sources) and effect (observations),  the criterion which is most closely related to the ICA setting in~\eqref{eq:gen} is
\textit{information-geometric causal inference} (IGCI)~\cite{daniuvsis2010inferring,janzing2012information}.\footnote{\looseness-1 For a similar criterion which assumes linearity~\cite{janzing2010telling,zscheischler2011testing} and its relation to linear ICA, see~\Cref{app:trace_method}.}
IGCI assumes a nonlinear relation $\eb=\fb(\cb)$ and formulates 
a notion of independence between the cause distribution $p_\cb$ and the deterministic mechanism $\fb$ (which we think of as a degenerate conditional $p_{\eb|\cb}$) via the following condition (in practice, assumed to hold approximately),%
\begin{equation}
\label{eq:IGCI_condition}
\textstyle
    C_\IGCI(\fb,p_\cb):=
    \int \log  \left|\Jb_{\fb}(\cb)\right|p_\cb(\cb)d\cb 
    -
    \int \log \left|\Jb_{\fb}(\cb)\right|d\cb
    = 0 \, ,
\end{equation}%
where $(\Jb_\fb(\cb))_{ij}=\nicefrac{\partial f_i}{\partial c_j}(\cb)$ is the Jacobian matrix and $|\cdot |$ the absolute value of the determinant.
$C_\IGCI$
can
be understood as the covariance between $p_\cb$ and $\log \left|\Jb_{\fb}\right|$ (viewed as r.v.s on 
the unit cube w.r.t.\ the Lebesgue measure), so that 
$C_\IGCI=0$ rules out a form of fine-tuning between $p_\cb$ and $|\Jb_\fb|$. 
As its name suggests,  IGCI can, from an information-geometric perspective, also be seen as an orthogonality condition between cause and mechanism in the space of probability distributions~\cite{janzing2012information}, see~\Cref{sec:information_geometric_ICM}, particularly~\cref{eq:addIrregIGCI} for further details.

%% file: sections/3_unsuitability_of_existing_ICM_measures.tex
Our aim is to use the ICM~\cref{principle:ICM} to further constrain the space of models $\Mcal\subseteq\Fcal\times\Pcal$ and rule out common counterexamples to identifiability such as those presented in~\cref{sec:background_ICA}.
Intuitively, both the Darmois construction~\eqref{eq:Darmois_construction} and the rotated Gaussian MPA~\eqref{eq:measure_preserving_automorphism_Gaussian} give rise to ``\textit{non-generic}'' solutions which should violate ICM: the former, $(\fb^\text{D},p_\ub)$, due the triangular Jacobian of $\fb^\text{D}$ (see~\cref{remark:triangular_Jacobian}), meaning that each observation
$x_i=f^\text{D}_i(\yb_{1:i})$ only depends on a subset of the inferred independent components~$\yb_{1:i}$,
and the latter, $(\fb\circ\ab^\Rb(p_\sb),p_\sb)$, due to the dependence of $\fb\circ\ab^\Rb(p_\sb)$ on $p_\sb$~\eqref{eq:measure_preserving_automorphism_Gaussian}.

However, the ICM criteria described in~\cref{sec:background_causality} were developed for 
the task of cause-effect inference where \textit{both variables are observed}.
In contrast, in this work, we consider
an unsupervised representation learning task
where \textit{only the effects} (mixtures $\xb$) \textit{are observed}, but the causes (sources $\sb$) are not. 
It turns out that this
renders existing ICM criteria insufficient for BSS:
they can easily 
be satisfied by spurious solutions which are not equivalent to the true one.
We can show this for IGCI.
Denote by $ \Mcal_{\IGCI}=\{(\fb, p_\sb)\in\Fcal\times\Pcal: C_\IGCI(\fb, p_\sb)=0\}\subset \Fcal\times\Pcal$ the class of nonlinear ICA models satisfying IGCI~\eqref{eq:IGCI_condition}. Then the following negative result holds.%
\begin{proposition}[IGCI is insufficient for $\sim_\BSS$-identifiability]
\label{prop:IGCI_insufficient_for_BSS}
\eqref{eq:gen} is not $\sim_\BSS$-identifiable on $\Mcal_{\IGCI}$.%
\vspace{-0.75em}
\begin{proof}
IGCI~\eqref{eq:IGCI_condition} is
satisfied when $p_\sb$ is uniform.
However, the Darmois construction~\eqref{eq:Darmois_construction}
yields uniform sources, see~\cref{fig:IGCI_Darmois}.
This means that $(\fb^\text{D}\circ\ab^\Rb(p_\ub), p_\ub)\in\Mcal_\IGCI$%
, so
IGCI can be satisfied by solutions which do not 
separate the sources in the sense of~\cref{def:bss_identifiability}, see footnote~\ref{footnote:ctrxmpl} and~\cite{hyvarinen1999nonlinear}.%
\end{proof}%
\end{proposition}%
\vspace{-0.75em}
As illustrated in~\cref{fig:ICM_graph}, condition%
~\eqref{eq:IGCI_condition} and other similar criteria enforce a notion of ``genericity'' or ``decoupling'' of the mechanism w.r.t.\ the \emph{observed} input distribution.\footnote{In fact, many ICM criteria can be phrased as special cases of a unifying group-invariance framework~\cite{besserve2018group}.} 
They thus rely on the fact that the cause (source) distribution is informative, and are generally not invariant to reparametrisation of the cause variables.
In the (nonlinear) ICA setting, on the other hand, the \emph{learnt} source distribution may be fairly uninformative.
This poses a challenge for existing ICM criteria since any mechanism is generic w.r.t.\ an uninformative (uniform) input distribution. 

%% file: sections/4_IMA.tex
As argued in~\cref{sec:unsuitability_of_existing_ICM_measures}, enforcing independence between the input distribution and the mechanism (\cref{fig:ICM_graph}), as existing ICM criteria do, is insufficient for ruling out spurious solutions to nonlinear ICA.
We therefore propose a new ICM-inspired framework which
is more suitable for BSS and which
we term \textit{independent mechanism analysis} (IMA).\footnote{The title of the present work is  thus a reverence
to Pierre Comon's seminal 1994 paper~\cite{comon1994independent}.}
All proofs are provided in~\Cref{app:proofs}.

\subsection{Intuition behind IMA
}
\label{sec:ima_intuition}
As motivated using the cocktail party example in~\cref{sec:introduction} and~\cref{fig:intuition} \textit{(Left)},
our main idea is to enforce
a notion of 
\textit{independence between the contributions or influences of the different sources $s_i$ on the observations} $\xb=\fb(\sb)$
as illustrated in~\cref{fig:IMA_graph}---as opposed to between the source distribution and mixing function, cf.~\cref{fig:ICM_graph}.
These contributions or influences are captured by the vectors  of partial derivatives $\nicefrac{\partial \fb}{\partial s_i}$.
IMA can thus be understood as a \textit{refinement of ICM at the level of the mixing $\fb$}:
in addition to
\textit{statistically independent components}~$s_i$, we look for a mixing with \textit{contributions $\nicefrac{\partial \fb}{\partial s_i}$ which are independent}, in a non-statistical sense which we formalise %
as follows.
\begin{principle}[IMA]
\label{principle:IMA}
The mechanisms by which each source $s_i$ influences the observed distribution, as captured by the 
partial derivatives $\nicefrac{\partial\fb}{\partial s_i}$, are independent of each other in the sense that for all~$\sb$:%
\begin{equation}
\label{eq:IMA_principle}
  \log |\Jb_\fb(\sb)|
    = 
    \sum_{i=1}^n
    \log \norm{\frac{\partial \fb}{\partial s_i}(\sb)} 
\end{equation}
\end{principle}%

\textbf{Geometric interpretation.}
Geometrically, the IMA principle can be understood as an \textit{orthogonality condition}, as illustrated for $n=2$ in~\cref{fig:intuition}~\textit{(Right)}.
First, the vectors of partial derivatives~$\nicefrac{\partial \fb}{\partial s_i}$, for which the IMA principle postulates independence, are the \textit{columns} of%
~$\Jb_\fb$.
$|\Jb_\fb|$ 
thus measures the volume of the $n-$dimensional parallelepiped spanned by these columns, as shown on the right.
The product of their norms, 
on the other hand, corresponds to the volume of an $n$-dimensional box, or rectangular parallelepiped with side lengths $\smallnorm{\nicefrac{\partial \fb}{\partial s_i}}$, as shown on the left.
The two volumes are equal if and only if all columns $\nicefrac{\partial \fb}{\partial s_i}$ of $\Jb_\fb$ are orthogonal.
Note that%
~\eqref{eq:IMA_principle} is trivially satisfied for $n=1$, i.e., if there is no mixing, further highlighting its difference from ICM for causal discovery.

\textbf{Independent influences and orthogonality.} In a high dimensional setting (large $n$), this orthogonality can be intuitively interpreted from the ICM perspective as \emph{Nature choosing the direction of the influence of each source component in the observation space independently and from an isotropic prior}. Indeed, it can be shown that the scalar product of two independent isotropic random vectors in $\RR^n$ vanishes as the dimensionality $n$ increases (equivalently: two high-dimensional isotropic vectors are typically orthogonal). This property was previously exploited in other linear ICM-based criteria (see~\cite[][Lemma 5]{janzing2018detecting} and%
~\cite[][Lemma 1 \& Thm.\ 1]{janzing2010telling}%
).\footnote{This has also been used as a \textit{``leading intuition''} [sic] to interpret IGCI in~\cite{janzing2012information}.} %
The principle in~\eqref{eq:IMA_principle} can be seen as a constraint on the function space, enforcing such orthogonality between the columns of the Jacobian of $\fb$ at all points in the source domain, %
thus approximating the high-dimensional behavior described above.\footnote{To provide additional intuition on how IMA differs from existing principles of independence of cause and mechanism, we give examples, both technical and pictorial, of violations of both in~\Cref{app:violations_icm_ima}.}

\textbf{Information-geometric interpretation and comparison to IGCI.}
The additive contribution of the sources' influences $\nicefrac{\partial \fb}{\partial s_i}$ in~\eqref{eq:IMA_principle}
suggests their local \textit{decoupling at the level of the mechanism}~$\fb$.
\looseness-1 Note that IGCI~\eqref{eq:IGCI_condition}, on the other hand, postulates a different type of decoupling: one between $\log |\Jb_\fb|$ and~$p_\sb$. There, dependence between  cause and mechanism can be conceived as a fine tuning between the derivative of the mechanism and the input density. %
The IMA principle leads to a complementary, non-statistical measure of independence between the influences~$\nicefrac{\partial \fb}{\partial s_i}$ of the individual sources on the vector of observations.
Both the IGCI and IMA postulates have an information-geometric interpretation related to the influence of (``non-statistically'') independent modules on the observations: both lead to an \textit{additive decomposition of a KL-divergence between the effect distribution and a reference distribution.} 
For IGCI, independent modules correspond to the cause distribution and the mechanism mapping the cause to the effect (see~\eqref{eq:addIrregIGCI} in~\Cref{sec:information_geometric_ICM}). For IMA, on the other hand, these are the influences of each source component on the observations in an interventional setting (under soft interventions on individual sources), as measured by the KL-divergences between the original and intervened distributions. See~\Cref{app:ima_igci}, and  especially~\eqref{eq:KLdecompIMA}, for a more detailed account.  %

We finally remark that while recent work based on the ICM principle has mostly used the term ``mechanism'' to refer to causal Markov kernels $p(X_i|PA_i)$
or structural equations~\cite{peters2017elements}, we employ it %
in line with the broader use of this concept in the philosophical literature.\footnote{See Table 1 in~\cite{mahoney2001beyond} for a long list of definitions from the literature.%
} 
To highlight just two examples,~\cite{salmon2020scientific} states that {\it ``Causal processes, causal interactions, and causal laws provide the mechanisms by which the world works; to understand why certain things happen, we need to see how they are produced by these mechanisms''}; and~\cite{tilly2001historical} states that {\it ``Mechanisms are events that alter relations among some specified set of elements''}. Following this perspective, we argue that a causal mechanism can more generally denote any process that describes the way in which causes influence their effects: the partial derivative $\nicefrac{\partial\fb}{\partial s_i}$
thus reflects a causal mechanism in the sense that it describes the infinitesimal changes in the observations $\xb$, when an infinitesimal perturbation is applied to $s_i$.

\subsection{Definition and useful properties of the IMA contrast}
\label{sec:ima_definition_properties}
We now introduce a contrast function based on the IMA principle~\eqref{eq:IMA_principle} 
and show that it possesses several desirable properties in the context of nonlinear ICA.
First, we define a local contrast as the difference between the two integrands of~\eqref{eq:IMA_principle} for a particular value of the sources $\sb$.
\begin{definition}[Local IMA contrast]
\label{def:local_IMA_contrast}
The local IMA contrast $c_\IMA(\fb,\sb)$ of
$\fb$ at a point $\sb$ is
given by
\begin{equation}
\label{eq:adm_single_point} 
c_\IMA(\fb,\sb) = \sum_{i = 1}^n \log \norm{\frac{\partial \fb}{\partial s_i}(\sb)} - \log \left|\Jb_\fb(\sb)\right|\,.
\end{equation}%
\end{definition}%
\begin{remark}
\label{remark:left_KL}
This corresponds to the left KL measure of diagonality~\cite{alyani2017diagonality} for
$\sqrt{\Jb_\fb(\sb)^\top\Jb_\fb(\sb)}$.
\end{remark}%
The local IMA contrast $c_\IMA(\fb,\sb)$ 
quantifies the extent to which the IMA principle is violated at a given point $\sb$.
We summarise some of its properties in the following proposition.
\begin{restatable}[Properties of $c_\IMA(\fb,\sb)$]
{proposition}{impropties}
    \label{prop:local_IMA_contrast_properties}
    The local IMA contrast $c_\IMA(\fb,\sb)$ defined in~\eqref{eq:adm_single_point} satisfies:%
    \begin{enumerate}[(i), topsep=0pt,itemsep=0pt]
        \item $c_\IMA(\fb,\sb) \geq 0$, with equality if and only if all columns $\nicefrac{\partial \fb}{\partial s_i}(\sb)$ of $\Jb_\fb(\sb)$ are orthogonal. %
        \item $c_\IMA(\fb,\sb)$ is invariant to left multiplication of $\Jb_\fb(\sb)$ by an orthogonal matrix and to right multiplication by permutation and diagonal matrices.%
    \end{enumerate}%
\end{restatable}
Property \textit{(i)} formalises the geometric interpretation of IMA as an orthogonality condition on the columns of the Jacobian from~\cref{sec:ima_intuition}, and property \textit{(ii)} intuitively states that changes of orthonormal basis and permutations or rescalings of the columns of $\Jb_\fb$ do not affect their orthogonality. 
Next, we define a global IMA contrast w.r.t.\ a source distribution $p_\sb$ as the expected local IMA contrast.%
\begin{definition}[Global IMA contrast]
\label{def:global_IMA_contrast}
The 
global IMA contrast $C_\IMA(\fb,p_\sb)$ of
$\fb$ 
w.r.t.\
$p_\sbb$ is given by 
\begin{equation}
    C_\IMA(\fb, p_\sbb) = \EE_{\sb\sim p_\sb}[c_\IMA(\fb,\sb)] = \medint\int c_\IMA(\fb,\sb) p_\sb(\sb)  d\sb\,.
    \label{eq:adm_metric}
\end{equation}%
\end{definition}%
The global IMA contrast $C_\IMA(\fb,p_\sb)$ thus quantifies the extent to which the IMA principle is violated for a particular solution $(\fb, p_\sb)$ to the nonlinear ICA problem.
We summarise its properties as follows.%
\begin{restatable}[Properties of $C_\IMA(\fb,p_\sb)$]
{proposition}{admproperties}
\label{prop:global_IMA_contrast_properties}
The global IMA contrast $C_\IMA(\fb, p_\sbb)$ from~\eqref{eq:adm_metric}
satisfies:%
    \begin{enumerate}[(i), topsep=0pt,itemsep=0pt]
        \item $C_\IMA(\fb, p_\sbb) \geq 0$, with equality
        iff.\ $\Jb_{\fb}(\sb) = \Ob(\sb) \Db(\sb)$ almost surely w.r.t.\ $p_\sb$, where $\Ob(\sb), \Db(\sb)\in\RR^{n\times n}$ are orthogonal and  diagonal matrices, respectively;
        \item $C_\IMA(\fb, p_\sbb) = C_\IMA(\fbt, p_{\sbt})$ for any $\fbt=\fb\circ\hb^{-1}\circ \Pb^{-1}$ and $\sbt = \Pb\hb(\sb)$,  where $\Pb\in\RR^{n\times n}$ is a permutation and $\hb(\sb)=(h_1(s_1), ..., h_n(s_n))$ an invertible element-wise function.
    \end{enumerate}%
\end{restatable}%
\begin{wrapfigure}{r}{0.45\textwidth}
\vspace{-1.3em}
\centering
\begin{subfigure}{0.2\textwidth}
\includegraphics[height=\textwidth]{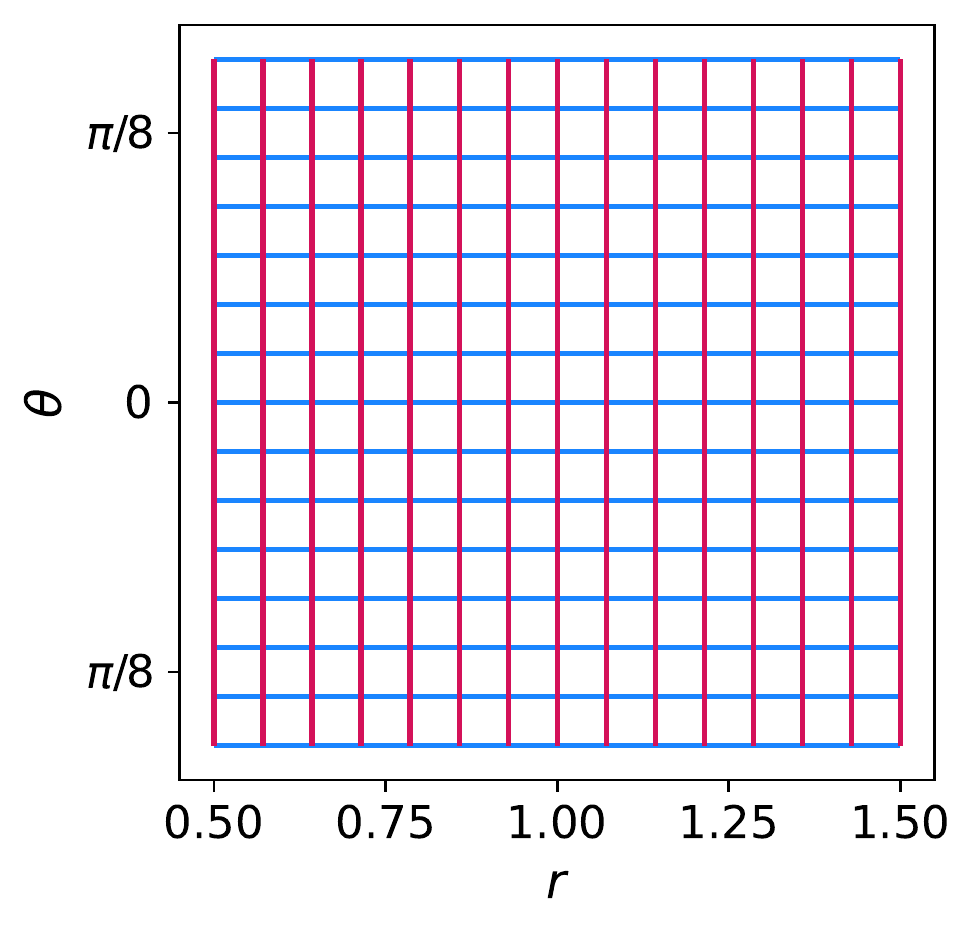}
\end{subfigure}%
\hspace{0.1em}
\begin{subfigure}{0.2\textwidth}
\includegraphics[height=\textwidth]{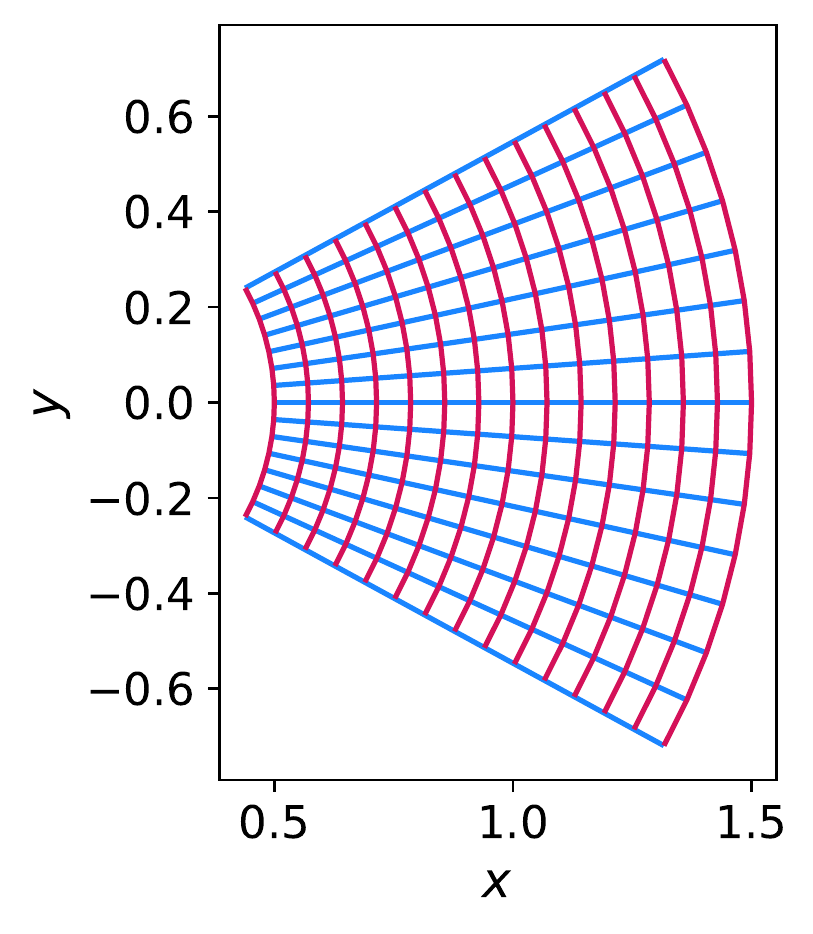}
\end{subfigure}%
\vspace{-0.75em}
\caption{\small An example of a (non-conformal) orthogonal coordinate transformation from polar (left) to Cartesian (right) coordinates.}
\label{fig:orthogonal_coordinate_transform}
\vspace{-1.5em}
\end{wrapfigure}
Property \textit{(i)} is the distribution-level analogue to \textit{(i)} of~\Cref{prop:local_IMA_contrast_properties} and only allows for orthogonality violations on sets of measure zero w.r.t.\ $p_\sb$.
This means that $C_\IMA$ can only be zero if $\fb$ is an \textit{orthogonal coordinate transformation} almost everywhere~\cite{lame1859leccons, darboux1910leccons,moon1971field}, see~\cref{fig:orthogonal_coordinate_transform} for an example.
We 
particularly
stress property \textit{(ii)}, as it precisely matches the inherent
indeterminacy
of nonlinear ICA: %
\textit{$C_\IMA$ is blind to reparametrisation of the sources by permutation and element wise transformation}.

\subsection{Theoretical analysis and justification of \texorpdfstring{$C_\IMA$}{theory}}
\label{sec:ima_theory}
\input{sections/4_3_implications}

%% file: sections/4_3_implications.tex
%
%
%
%
%
%
%
%
%
%
%

%
\begin{comment}
We call such solutions ``spurious'', and formally define them as follows:
\begin{definition}[Spurious solutions]
Given data generated according to~\ref{eq:gen}, and a smooth, invertible function $\gb$ s.t. $\yb = \gb(\xb)=\gb(\fb^{-1}(\sb)), \textstyle
p_\yb(\yb) = \prod_{i=1}^n p_{y_i}(y_i)$, we call $(\gb^{-1}, p_\yb)$ a \textit{spurious solution} if
$\partial^2 y_i / \partial s_j \partial s_{j'} \neq 0$ for some $i, j\neq j'$.     
\end{definition}
\end{comment}
%
%
We now show that, under suitable assumptions on the generative model~\eqref{eq:gen}, a large class of
spurious solutions---such as those based on the Darmois construction~\eqref{eq:Darmois_construction} or 
measure preserving automorphisms such as $\ab^\Rb$ from~\eqref{eq:measure_preserving_automorphism_Gaussian} as described in~\cref{sec:background_ICA}---%
exhibit nonzero IMA contrast.
Denote the class of nonlinear ICA models satisfying~\eqref{eq:IMA_principle} (IMA) by  ${\Mcal_{\IMA}=\{(\fb, p_\sb)\in \Fcal\times\Pcal: C_\IMA(\fb, p_\sb)=0\}\subset \Fcal\times\Pcal}$.
Our first main theoretical result is that, under mild assumptions on the observations, Darmois solutions will have strictly positive $C_\IMA$, making them distinguishable from those in~$\Mcal_\IMA$.%
\begin{restatable}{theorem}{admdarmois}
\label{thm:adm_darmois}
Assume the data generating process in~\eqref{eq:gen}
and assume that
$x_i \nindep x_j$ for some $i \neq j$. Then any  Darmois solution $(\fb^\text{D}, p_\ub)$ based on $\gb^\text{D}$ as defined in~\eqref{eq:Darmois_construction} satisfies $C_\IMA(\fb^\text{D}, p_\ub)>0$.
Thus a solution satisfying $C_\IMA(\fb, p_\sb)=0$ can be distinguished from $(\fb^\text{D}, p_\ub)$ based on the contrast $C_\IMA$.
\end{restatable}%
\vspace{-0.5em}
The proof is based on the fact that the Jacobian of $\gb^\text{D}$ is triangular (see~\Cref{remark:triangular_Jacobian}) and on the specific form of~\eqref{eq:Darmois_construction}.
A specific example %
of a mixing process satisfying the IMA assumption
is the case where $\fb$ is a conformal (angle-preserving) map.%
\begin{definition}[Conformal map]
\label{def:conformal_map}
A smooth map $\fb:\RR^n\rightarrow\RR^n$ is conformal if $\Jb_\fb(\sb)=\Ob(\sb)\lambda(\sb)$ $\forall\sb$, where $\lambda:\RR^n\rightarrow\RR$ is a scalar field,
and $\Ob\in O(n)$ is an orthogonal matrix.%
\end{definition}%
\begin{restatable}
{corollary}{confmapsadm}
\label{cor:IMA_identifiability_of_conformal_maps}
Under assumptions of~\Cref{thm:adm_darmois}, if additionally $\fb$ is a conformal map, then $(\fb,p_\sb)\in\Mcal_\IMA$ for any $p_\sb\in\Pcal$ due to~\cref{prop:global_IMA_contrast_properties} \textit{(i)}, see~\cref{def:conformal_map}. 
Based on~\cref{thm:adm_darmois}, $(\fb,p_\sb)$ is thus distinguishable from  Darmois solutions~$(\fb^\text{D}, p_\ub)$.%
\end{restatable}%
\vspace{-0.5em}
This 
is consistent with
a result
that proves identifiability of conformal maps
for $n=2$ and conjectures it in general~\cite{hyvarinen1999nonlinear}.\footnote{Note that~\Cref{cor:IMA_identifiability_of_conformal_maps} holds for any dimensionality $n$.}
However, conformal maps are only a small subset of all maps for which $C_\IMA=0$, as is apparent from the more flexible condition of~\Cref{prop:global_IMA_contrast_properties}~\textit{(i)}, compared to the stricter~\Cref{def:conformal_map}.
\begin{example}[Polar to Cartesian coordinate transform]
Consider the \textit{non-conformal}
transformation from polar to Cartesian coordinates~(see~\cref{fig:orthogonal_coordinate_transform}), defined as $(x,y)=\fb(r,\theta):=(r\cos(\theta),r\sin(\theta))$ with independent sources $\sb=(r,\theta)$, with $r\sim U(0,R)$ and $\theta\sim U(0, 2\pi)$.\footnote{For different $p_\sb$, $(x,y)$ can be made to have independent Gaussian components~(\cite{taleb1999source}, II.B), and
$C_\IMA$-identifiability is lost; this shows that the assumption of~\Cref{thm:adm_darmois} that $x_i \nindep x_j$ for some $i \neq j$ is crucial.}
Then, $C_\IMA(\fb,p_\sb)=0$ and $C_\IMA(\fb^\text{D}, p_\ub)>0$ for any
Darmois solution $(\fb^\text{D}, p_\ub)$
---see~\Cref{app:examples} for details.
\end{example}%
Finally, for the case in which the true mixing is linear, we obtain the following result.%
\begin{restatable}{corollary}{admidentlinear}
\label{cor:IMA_identifiability_of_linear_ICA}
Consider a linear ICA model, $\xb=\Ab\sb$, with $\EE[\sb^\top\sb]=\Ib$, and $\Ab\in O(n)$ an orthogonal, non-trivial mixing matrix, i.e., not the product of a diagonal and a permutation matrix~$\Db \Pb$.
If at most one of the $s_i$ is Gaussian, then $C_\IMA(\Ab, p_\sb)=0$ and $C_\IMA(\fb^\text{D}, p_\ub)>0$.%
\end{restatable}%
In a ``blind'' setting, we may not know a priori whether the true mixing is linear or not, and thus choose to learn a nonlinear unmixing.
\Cref{cor:IMA_identifiability_of_linear_ICA} shows that, in this case, Darmois solutions are still distinguishable from the true mixing via $C_\IMA$.
Note that unlike in~\cref{cor:IMA_identifiability_of_conformal_maps}, the assumption that $x_i \nindep x_j$ for some $i \neq j$ is not required for~\cref{cor:IMA_identifiability_of_linear_ICA}. In fact, due to Theorem 11 of~\cite{comon1994independent}, it follows from the assumed linear ICA model with non-Gaussian sources, and the fact that the mixing matrix  is not the product of a diagonal and a permutation matrix
(see also%
~\Cref{app:identifiability_and_linear_ICA}). 

Having shown that the IMA principle allows to distinguish a class of models (including, but not limited to conformal maps) from Darmois solutions, we next turn to a second well-known counterexample to identifiability: the ``rotated-Gaussian'' MPA $\ab^\Rb(p_\sb)$~\eqref{eq:measure_preserving_automorphism_Gaussian} from~\Cref{def:measure_preserving_automorphism_Gaussian}.
Our second main theoretical result is that, under suitable assumptions, this class of MPAs can also be ruled out for ``non-trivial'' $\Rb$.%
\begin{restatable}{theorem}{thmMPA}
\label{thm:IMA_identifiability_measure_preserving_automorphism}
Let $(\fb,p_\sb)\in\Mcal_\IMA$ and assume that $\fb$ is a conformal map.
Given $\Rb\in O(n)$, assume
additionally
that 
    $\exists$ at least one non-Gaussian 
    $s_i$ whose associated canonical 
    basis vector $\eb_i$ is not transformed by $\Rb^{-1}=\Rb^\top$ into another canonical basis vector $\eb_j$. 
Then $C_\IMA(\fb\circ \ab^\Rb(p_\sb),p_\sb)>0$.%
\end{restatable}%
\Cref{thm:IMA_identifiability_measure_preserving_automorphism} states that for conformal maps, applying the $\ab^\Rb(p_\sb)$ transformation at the level of the sources leads to an increase in $C_\IMA$ except for very specific rotations $\Rb$ that are ``fine-tuned'' to $p_\sb$ in the sense that they permute all non-Gaussian sources $s_i$ with another $s_j$.
Interestingly, as for the linear case, non-Gaussianity again plays an important role in the proof of~\Cref{thm:IMA_identifiability_measure_preserving_automorphism}.

%% file: sections/5_experiments.tex
\input{complex_figures/2d_sources_observatios}

Our theoretical results from~\cref{sec:IMA} suggest that $C_\IMA$ is a promising contrast function for nonlinear blind source separation. 
We test this empirically by evaluating the $C_\IMA$ of spurious nonlinear ICA solutions~(\cref{sec:experiment1_evaluation}), and using it as a learning objective to recover the true solution~(\cref{sec:experiment2_learning}).

We sample the ground truth sources from a uniform distribution in $[0,1]^n$; the reconstructed sources
are also mapped to the uniform hypercube as a reference measure via the CDF transform.
Unless otherwise specified, the ground truth mixing~$\fb$ is a M\"obius transformation~\cite{phillips1969liouville} (i.e., a conformal map) with randomly sampled parameters, thereby satisfying~\cref{principle:IMA}. In all of our experiments, we use JAX~\cite{jax2018github} and Distrax~\cite{distrax2021github}. For additional technical details, equations and plots see~\Cref{app:experiments}. The code to reproduce our experiments is available at \href{https://github.com/lgresele/independent-mechanism-analysis}{this link}.

\subsection{Numerical evaluation of the \texorpdfstring{$C_\IMA$}{numeval} contrast for spurious nonlinear ICA solutions
}
\label{sec:experiment1_evaluation}

\textbf{Learning the Darmois construction.} To
learn the Darmois construction from data, we use
 normalising flows%
, %
see~\cite{huang2018neural, papamakarios2019normalizing}. %
Since Darmois solutions have triangular Jacobian~(\cref{remark:triangular_Jacobian}), we use an
architecture based on 
residual flows~\cite{chen2019residualflows} which we constrain such that the Jacobian of the full model is 
triangular. This yields an expressive model which we %
train effectively via maximum likelihood.

\textbf{$C_\IMA$ of Darmois solutions.}
To check whether Darmois solutions (learnt from finite data) can be distinguished from the true one, as predicted by~\cref{thm:adm_darmois}, we generate $1000$ random mixing functions for $n=2$, compute the
$C_{\IMA}$ values of learnt solutions, 
and find that all values are indeed significantly larger than zero, see~\cref{fig:results1} \textbf{(a)}.
The same holds for higher dimensions,
see~\cref{fig:results1} \textbf{(b)} for results with $50$ random mixings for $n\in \{2, 3, 5 ,10\}$: with higher dimensionality, both the mean and variance of the $C_\IMA$ distribution for the learnt Darmois solutions generally attain higher values.\footnote{the latter possibly due to the increased difficulty of the learning task for larger $n$} We confirmed these findings for mappings which are not conformal, while still satisfying~\eqref{eq:IMA_principle}, in~\Cref{sec:app_eval}.

\textbf{$C_\IMA$ of MPAs.}
We also investigate the effect on $C_\IMA$ of applying an MPA $\ab^{\Rb}(\cdot)$ from~\eqref{eq:measure_preserving_automorphism_Gaussian} to the true solution or a learnt Darmois solution.
Results for $n=2$ dim.\ for different rotation matrices $\Rb$ (parametrised by the angle $\theta$) are shown in~\cref{fig:results1} \textbf{(c)}.
As expected, the behavior is periodic in $\theta$, and vanishes for the true solution (blue) at multiples of $\nicefrac{\pi}{2}$, i.e., when $\Rb$ is a permutation matrix, as predicted by~\cref{thm:IMA_identifiability_measure_preserving_automorphism}. For the learnt Darmois solution (red, dashed) $C_\IMA$ remains larger than zero.

\textbf{$C_\IMA$ values for random MLPs.} 
Lastly, we study
the behavior 
of 
spurious solutions based on the Darmois construction
under deviations from our assumption of $C_\IMA=0$ for the true mixing function.
To this end, we use invertible MLPs with
orthogonal weight initalisation and \texttt{leaky\_tanh} activations~\cite{gresele2020relative} as mixing functions; 
the more layers $L$ are added to the mixing MLP, the larger a deviation from our assumptions is expected. We compare the true mixing and learnt Darmois solutions over \tochange{$20$} realisations for each $L \in \{2, 3, 4\}$, $n=5$. 
Results are shown in figure~\cref{fig:results1} \textbf{(d)}: the $C_\IMA$ of the mixing MLPs grows with $L$; still, the one of the Darmois solution is typically higher.

\textbf{Summary.} We verify that spurious solutions can be distinguished from the true one based on~$C_\IMA$.

\subsection{Learning nonlinear ICA solutions with \texorpdfstring{$C_\IMA$}{learn}-regularised maximum likelihood
}
\label{sec:experiment2_learning}

\textbf{Experimental setup.} To use $C_\IMA$ as a learning signal, we consider a regularised maximum-likelihood approach, with the following objective:
${\Lcal(\gb) = \EE_{\xb}[\log p_\gb(\xb)] - \lambda \, C_\IMA(\gb^{-1}, p_\yb)}$, where $\gb$ denotes the learnt unmixing, $\yb = \gb(\xb)$ the reconstructed sources, and $\lambda\geq0$ a Lagrange multiplier. %
For $\lambda=0$, this corresponds to standard maximum likelihood estimation, whereas for $\lambda>0$, $\Lcal$ lower-bounds the likelihood, and recovers it exactly iff.\ $(\gb^{-1},p_\yb)\in\Mcal_\IMA$.
We train a residual flow $\gb$ (with full Jacobian) to maximise $\Lcal$.
For evaluation, we compute (i) the KL divergence to the true data likelihood, as a measure of goodness of fit for the learnt flow model; and (ii)
the mean correlation coefficient (MCC) between ground truth and reconstructed sources~\cite{hyvarinen2016unsupervised, khemakhem2020variational}.
We also introduce
(iii) a nonlinear extension of the Amari distance~\cite{amari1996new}
between the true mixing and the learnt unmixing, which
is larger than or equal to zero, with equality iff.\ %
the learnt model belongs to the BSS equivalence class~(\cref{def:bss_identifiability}) of the true solution, see~\Cref{sec:app_eval} for details. %

\textbf{Results.}
In~\cref{fig:results1} \textit{(Top)}, we show an 
example of the distortion induced by different \textit{spurious}
solutions for $n=2$, and contrast it with a solution learnt using our proposed objective \textit{(rightmost plot)}.
Visually, we find that the $C_\IMA$-regularised solution (with $\lambda=1$) recovers the true sources most faithfully.
Quantitative results for 50 learnt models for each $\lambda\in \{0.0, 0.5, 1.0 \}$ and \tochange{$n\in\{5, 7\}$} are summarised in~\cref{fig:results2} \tochange{ (see~\Cref{app:experiments} for additional plots%
)}.
As indicated by the KL divergence values \textit{(left)}, most trained models achieve a good fit to the data across all values of $\lambda$.\footnote{models with $n=7$ have high outlier KL values, seemingly less pronounced for nonzero values of $\lambda$%
}
We observe that using $C_\IMA$ (i.e.,
$\lambda>0$) is beneficial for BSS, both in terms of our nonlinear Amari distance \textit{(center, lower is better)} and  MCC \textit{(right, higher is better)}, though we do not observe a substantial difference between $\lambda=0.5$ and $\lambda=1$.\footnote{In~\Cref{sec:app_eval}, we also show that our method is superior to a linear ICA baseline, FastICA~\cite{hyvarinen1999fast}.}

\textbf{Summary:} $C_\IMA$ can be a useful learning signal to recover the true solution.

%

%
%
%
%
%
%
%

%

%
%
%
%
%
%
%
%
%
%
%
%
%
%
%
%
%
%

%
%
%
%
%
%
%
%
%
%

%
%
%
%
%

%

%
%
%
%
%
%

%

\begin{comment}
\paragraph{A special case: max likelihood for conformal maps}
Given that, this implies that the likelihood of (when ... is a conformal map) can be expressed as ...
\begin{equation}
\log p_{\bm{\theta}}(\xb) = \sum_i \log p_{s_i}(\gb^i_{\bm{\theta}}(\xb)) + \sum_i \log || \frac{\partial \fb}{\partial s_i} ||
\end{equation}
For conformal maps
\end{comment}

\newcommand\heightrow{2.9}
\newcommand\heightrowtwo{2.9}
\newcommand\heightamari{2.853}
\newcommand\heightmcc{2.853}
\newcommand\gapfive{0.002}

\begin{figure}%
    \hspace{-0.8 em}
    \begin{subfigure}[b]{0.15\textwidth}
        \centering
        \includegraphics[height=\heightrow cm]{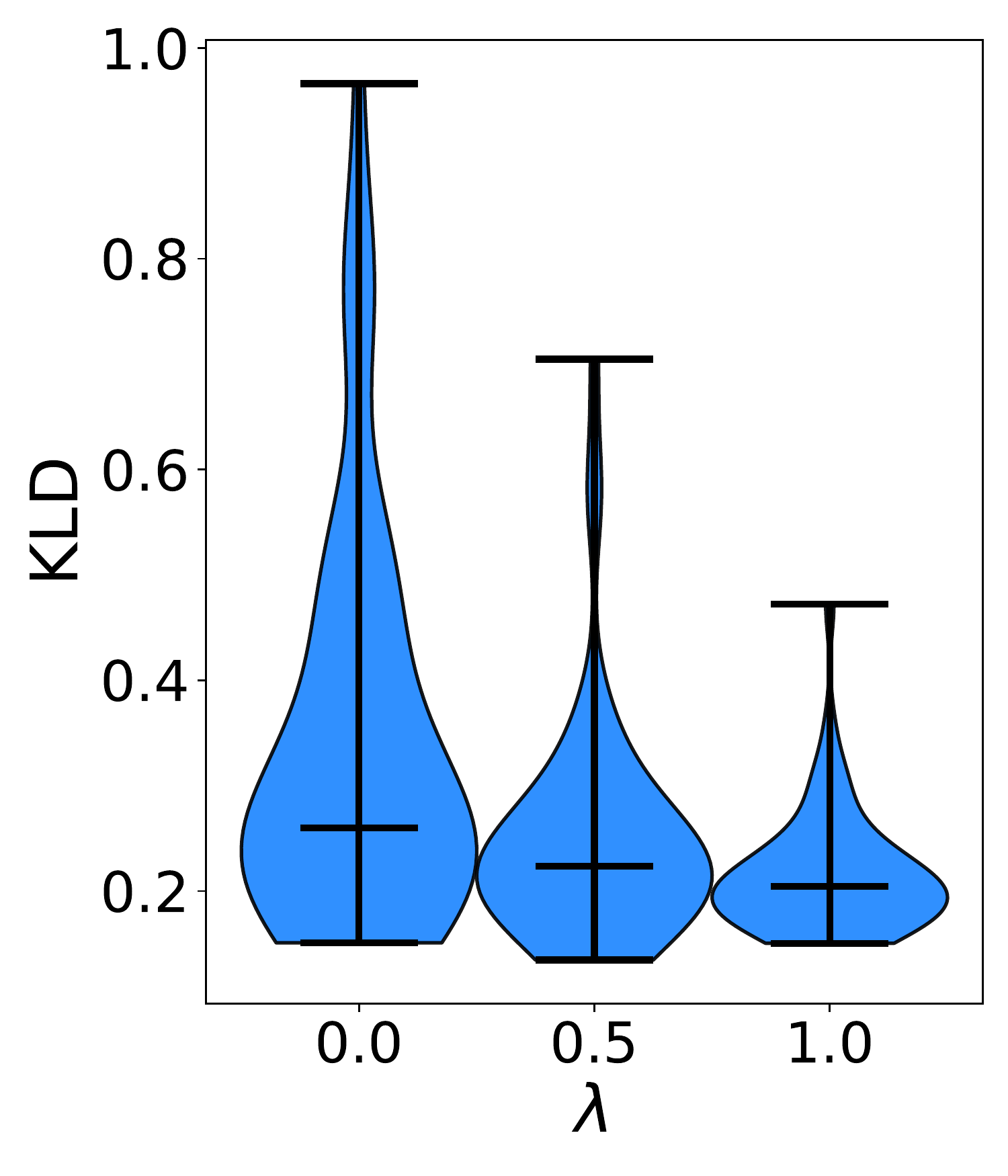}
    \end{subfigure}%
    \hspace{0.8 em}
    \begin{subfigure}[b]{0.15\textwidth}
        \centering
        \includegraphics[height=\heightrow cm]{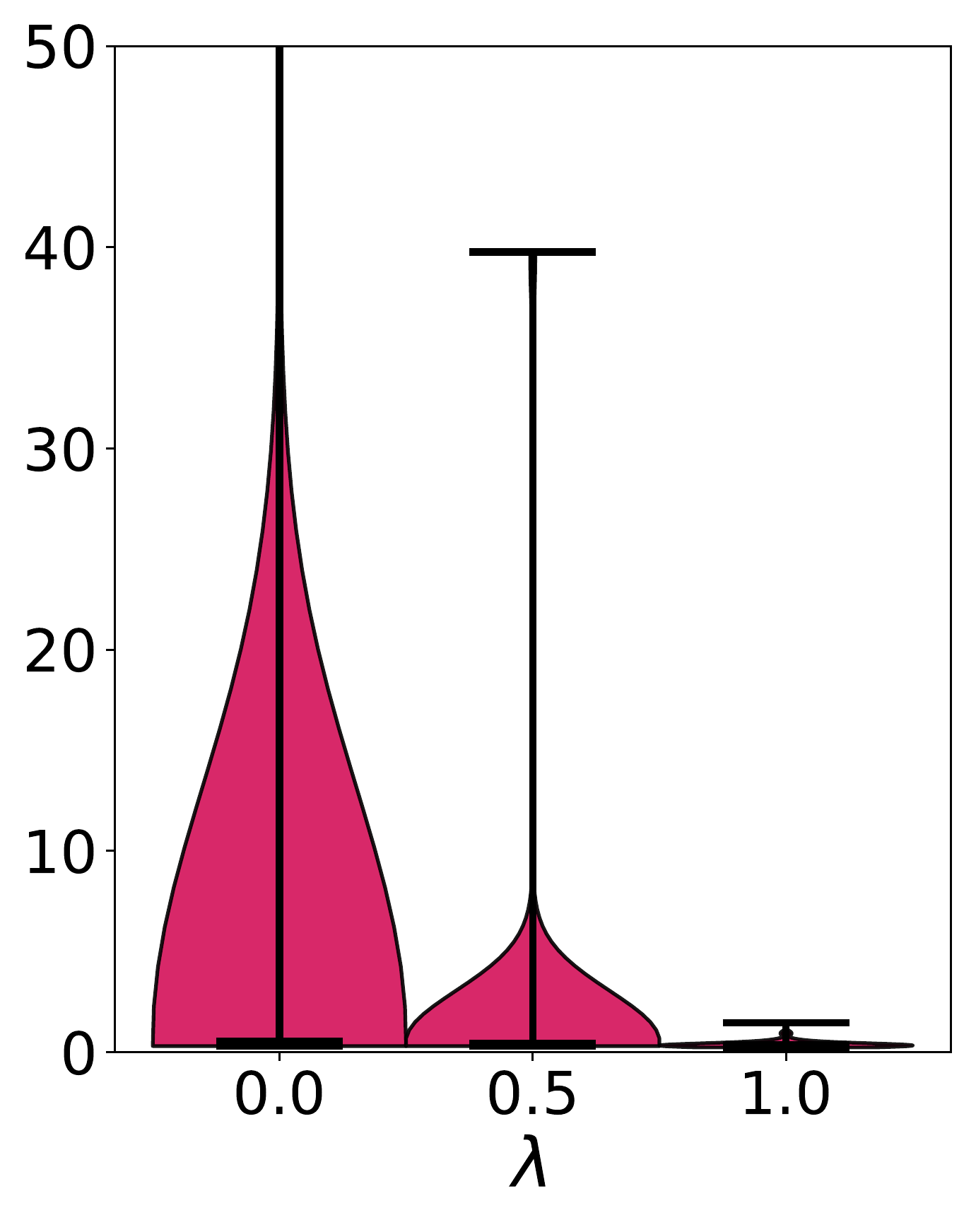}
    \end{subfigure}
    \hspace{0.005 em}
    \begin{subfigure}[b]{0.15\textwidth}
        \centering
        \includegraphics[height=\heightrow cm]{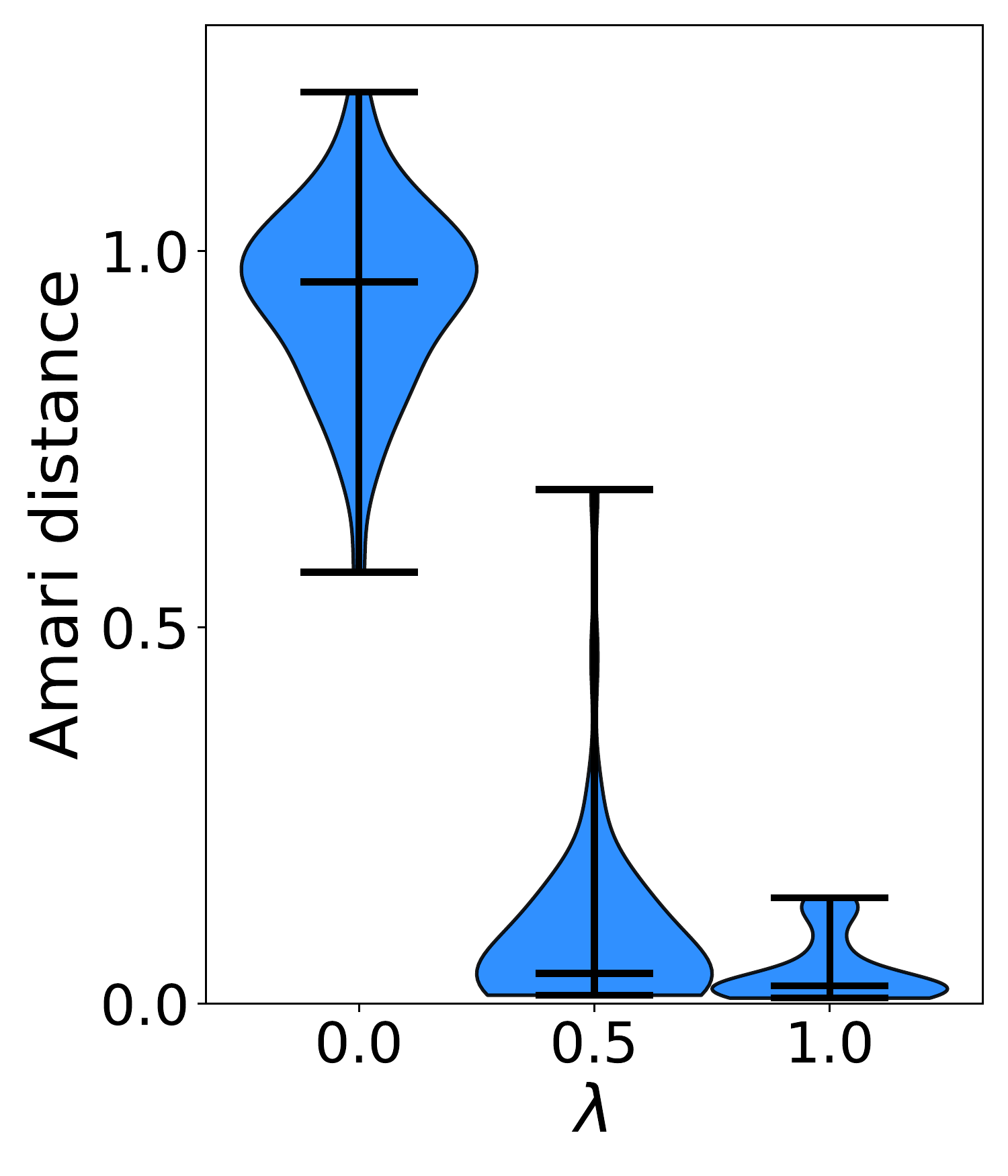}
    \end{subfigure}
    \hspace{0.5 em}
    \begin{subfigure}[b]{0.15\textwidth}
        \centering
        \includegraphics[height=\heightrow cm]{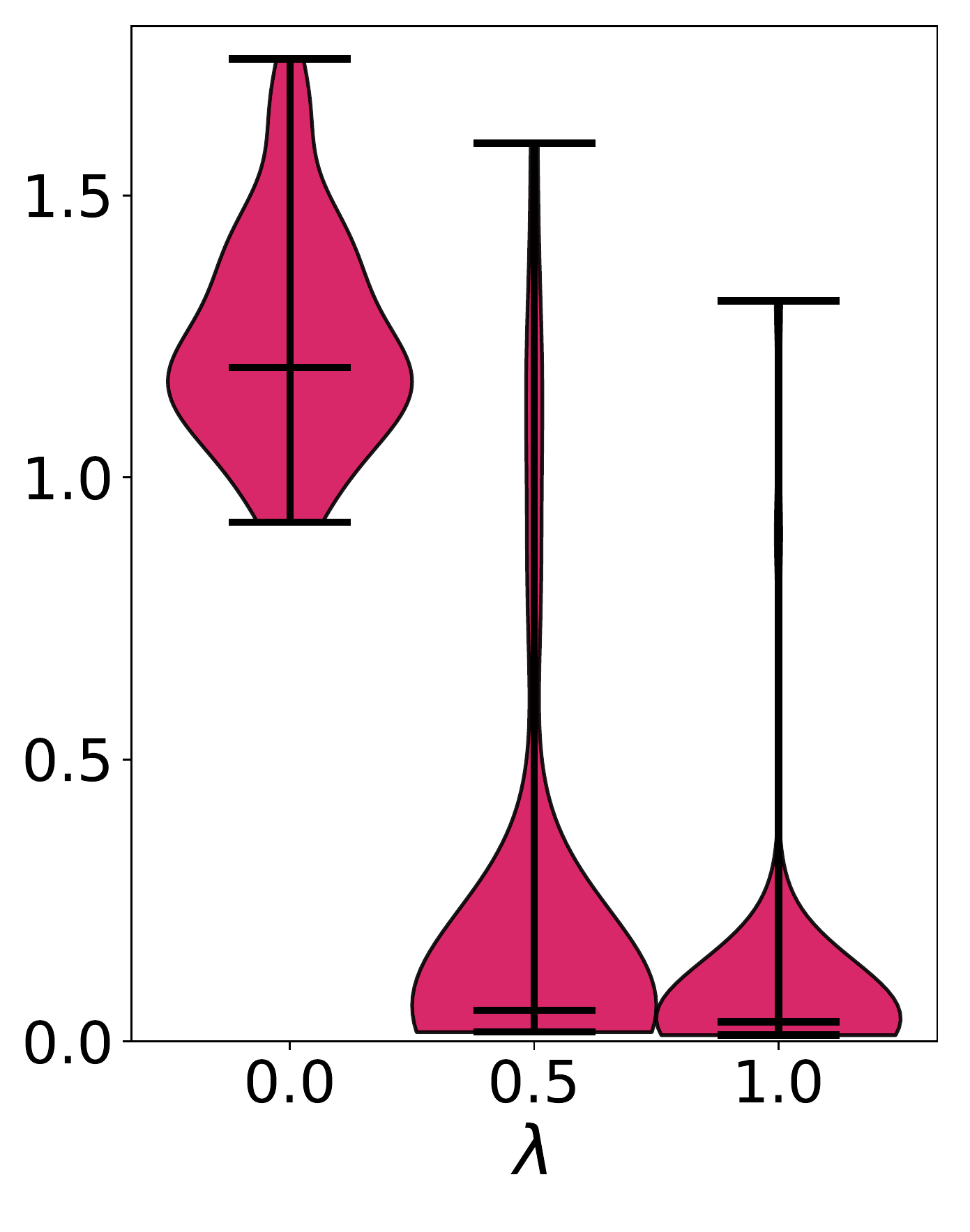}
    \end{subfigure}
    \hspace{0.005 em}
    \begin{subfigure}[b]{0.15\textwidth}
        \centering
        \includegraphics[height=\heightrow cm]{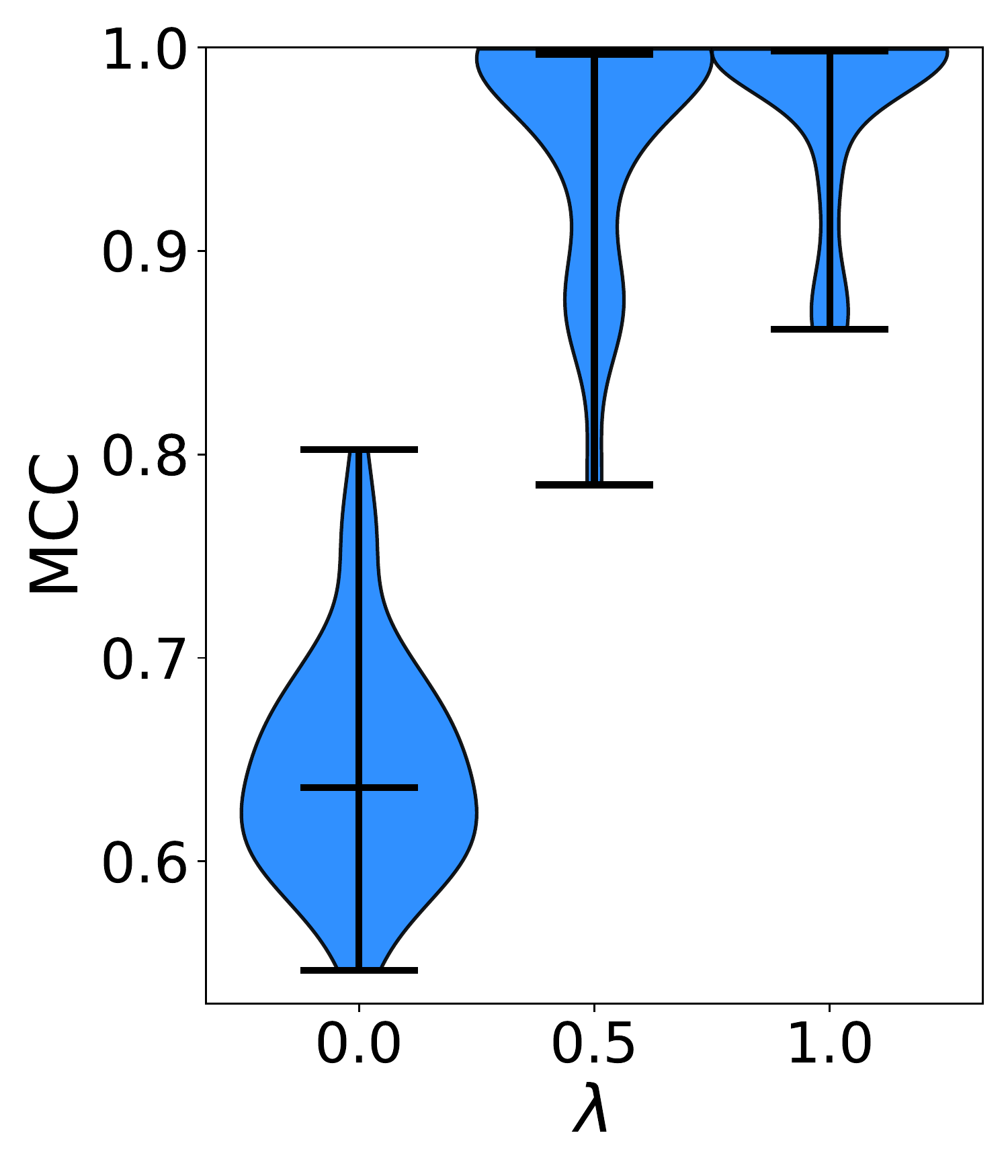}
    \end{subfigure}
    \hspace{0.55 em}
    \begin{subfigure}[b]{0.15\textwidth}
        \centering
        \includegraphics[height=\heightrow cm]{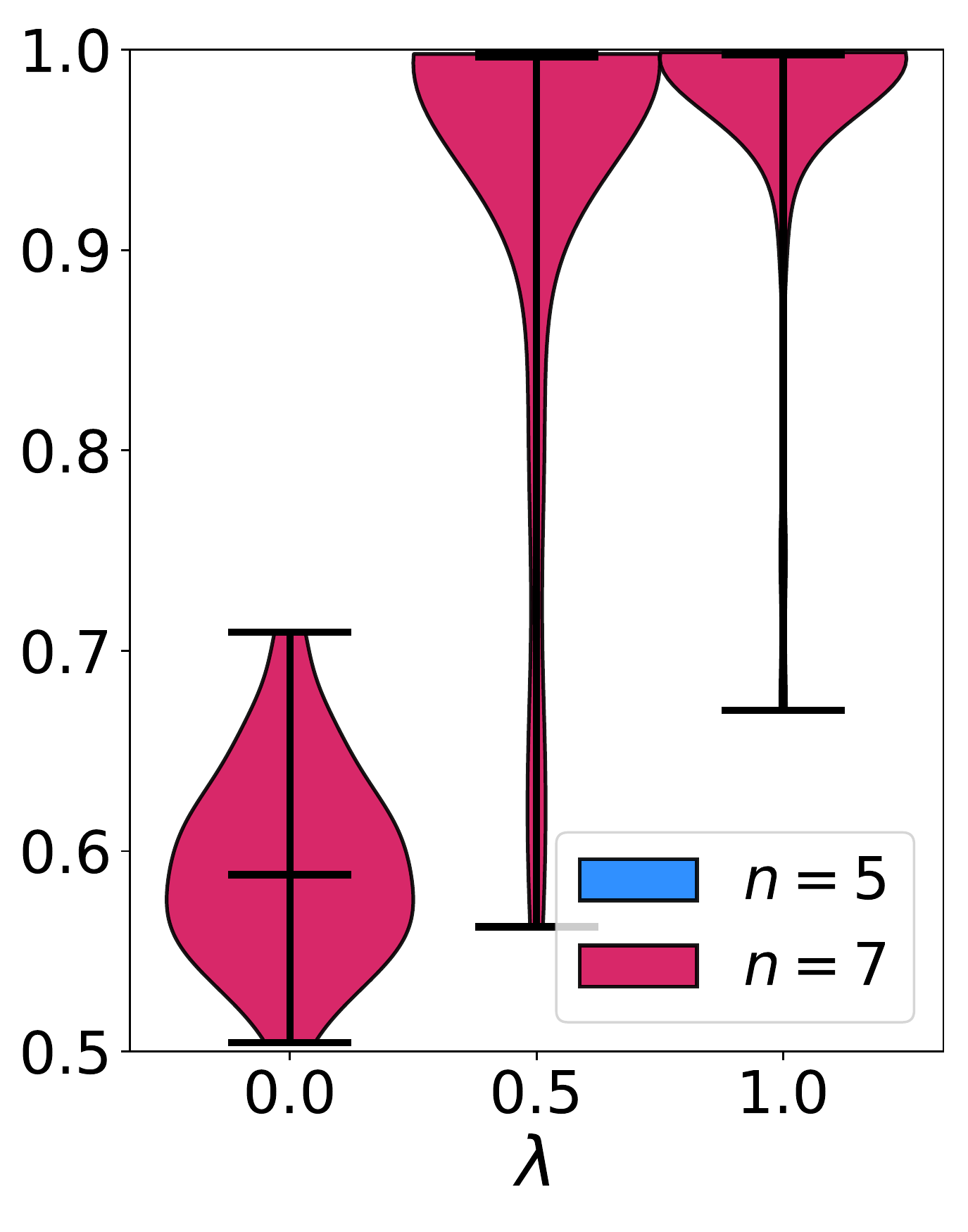}
    \end{subfigure}
    \vspace{-0.5em}
    \caption{\small BSS via $C_\IMA$-regularised 
    MLE
    for, side by side, \tochange{$n=5$ (blue) and $n=7$ (red)} dim. with $\lambda\in \{0.0,0.5,1.0\}$.
    \textit{(Left)} KL-divergence between ground truth likelihood and learnt model;
    \textit{(center)} nonlinear Amari distance given true mixing and learnt unmixing; \textit{(right)} MCC between true and reconstructed sources.
    }
    \label{fig:results2}
\end{figure}

\begin{comment}
\begin{figure}%
    \vspace{-1em}
    \begin{subfigure}[b]{0.33\textwidth}
        \centering
        \includegraphics[height=\heightrow cm]{plots/kld_violinplot.pdf}
        %
        %
    \end{subfigure}%
    \begin{subfigure}[b]{0.33\textwidth}
        \centering
        \includegraphics[height=\heightrow cm]{plots/mcc_violinplot.pdf}
        %
        %
    \end{subfigure}
    \begin{subfigure}[b]{0.33\textwidth}
        \centering
        \includegraphics[height=\heightrow cm]{plots/amari_violinplot.pdf}
        %
        %
    \end{subfigure}
    \vspace{-0.5em}
    \caption{\small BSS via $C_\IMA$-regularised maximum likelihood estimation
    %
    for \tochange{$n \in\{5, 7\}$} dim. and $\lambda\in \{0.0,0.5,1.0\}$.
    \textit{(Left)} KL-divergence between ground truth likelihood and learnt model;
    \textit{(center)} MCC between true and reconstructed sources; \textit{(right)} nonlinear Amari distance between true mixing and learnt unmixing.
    }
    \label{fig:results2}
    \vspace{-1em}
\end{figure}
\end{comment}

%% file: complex_figures/2d_sources_observatios.tex
\newcommand\width{2.25}
\newcommand\height{1.95}
\newcommand\gap{.005}
\newcommand\minipagewidth{.13}
\newcommand\folder{plots_hsv}
\newcommand\widthbottom{2.4}
\newcommand\heightbottom{2.85}

\newcommand\leftplace{-0.9}
\newcommand\lowplace{3.0}

\begin{figure}[t]
    \centering
        \begin{minipage}{\minipagewidth \textwidth}
            \begin{subfigure}{1.0\textwidth}
            \centering
            \includegraphics[height=\height cm, keepaspectratio]{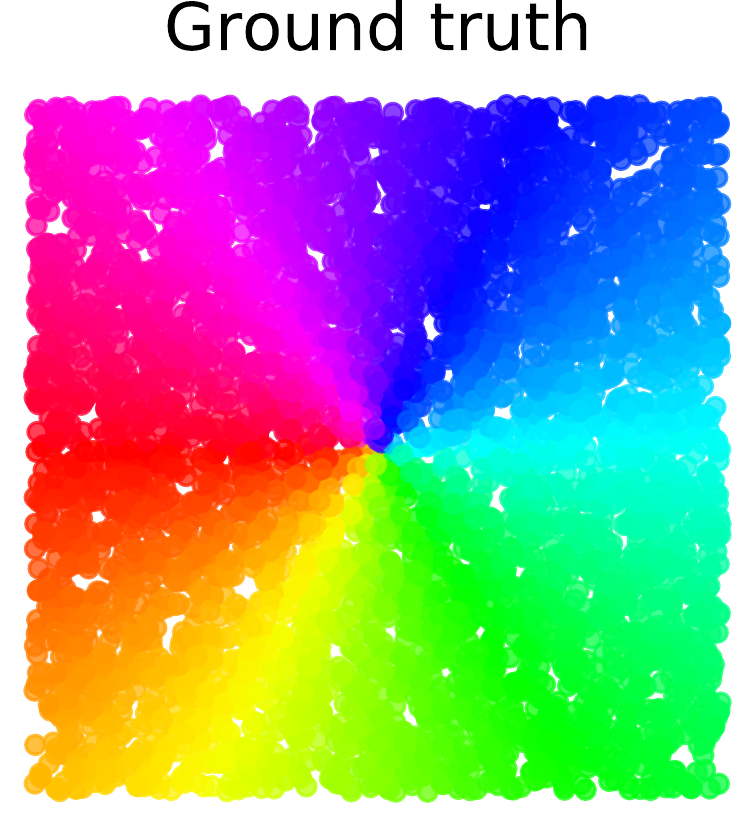}
            \end{subfigure}
        \end{minipage}%
        \hspace{\gap em}
        \begin{minipage}{\minipagewidth \textwidth}
            \begin{subfigure}{1.0\textwidth}
            \centering
            \includegraphics[height=\height cm, keepaspectratio]{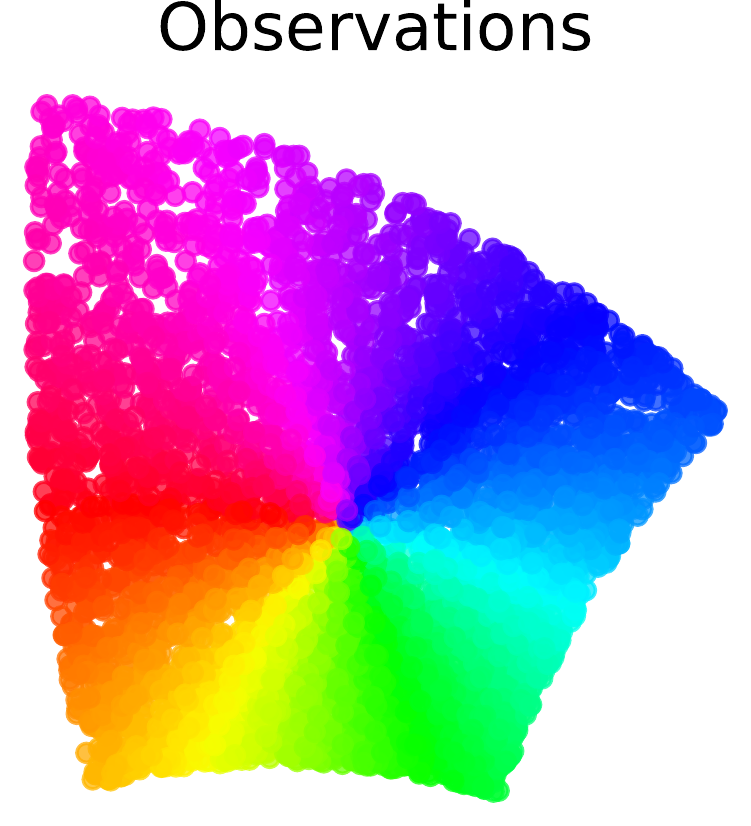}
            \end{subfigure}
        \end{minipage}%
        \hspace{\gap em}
        \begin{minipage}{\minipagewidth \textwidth}
            \begin{subfigure}{1.0\textwidth}
            \centering
            \includegraphics[height=\height cm, keepaspectratio]{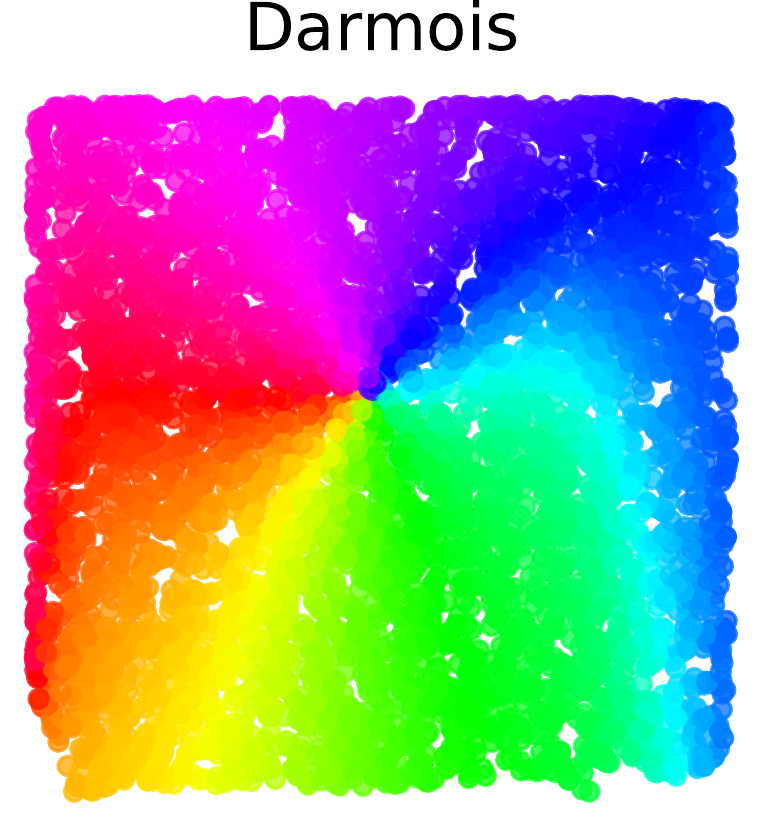}
            \end{subfigure}
        \end{minipage}%
        \hspace{\gap em}
        \begin{minipage}{\minipagewidth \textwidth}
            \begin{subfigure}{1.0\textwidth}
            \centering
            \includegraphics[height=\height cm, keepaspectratio]{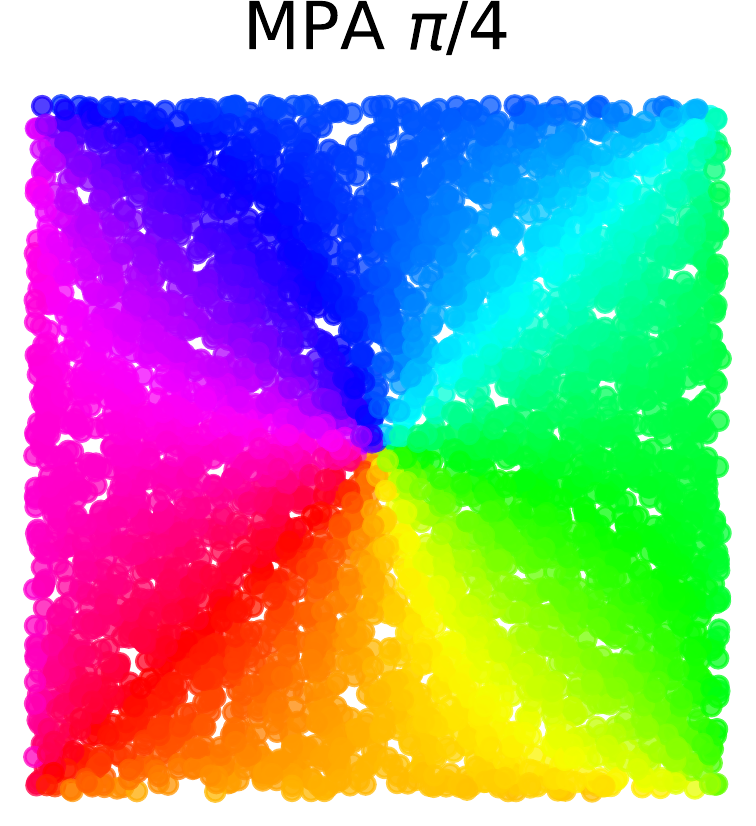}
            \end{subfigure}
        \end{minipage}%
        \hspace{\gap em}
        \begin{minipage}{\minipagewidth \textwidth}
            \begin{subfigure}{1.0\textwidth}
            \centering
            \includegraphics[height=\height cm, keepaspectratio]{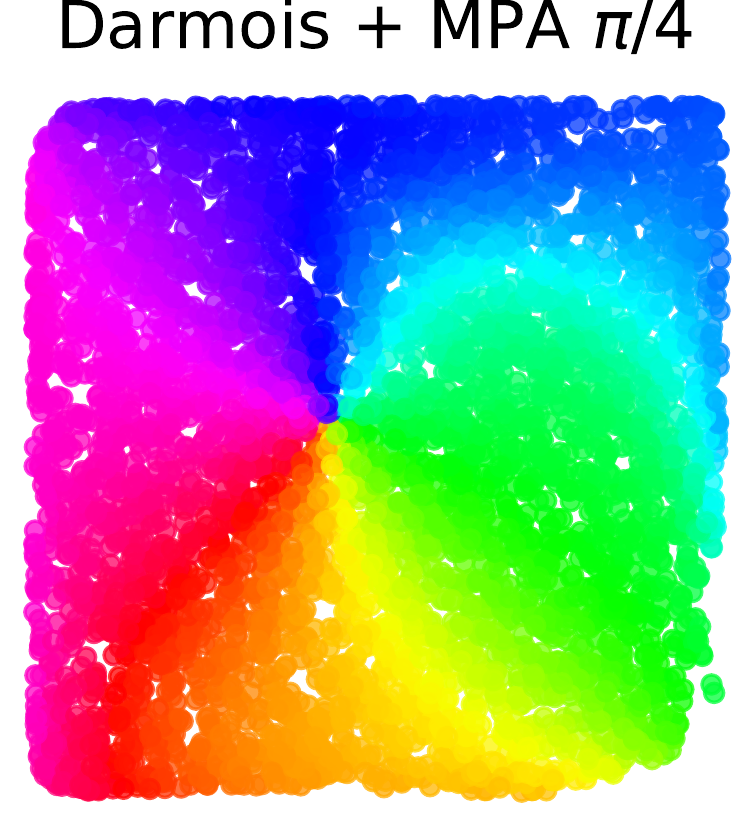}
            \end{subfigure}
        \end{minipage}%
        \hspace{\gap em}
        \begin{minipage}{\minipagewidth \textwidth}
            \begin{subfigure}{1.0\textwidth}
            \centering
            \includegraphics[height=\height cm, keepaspectratio]{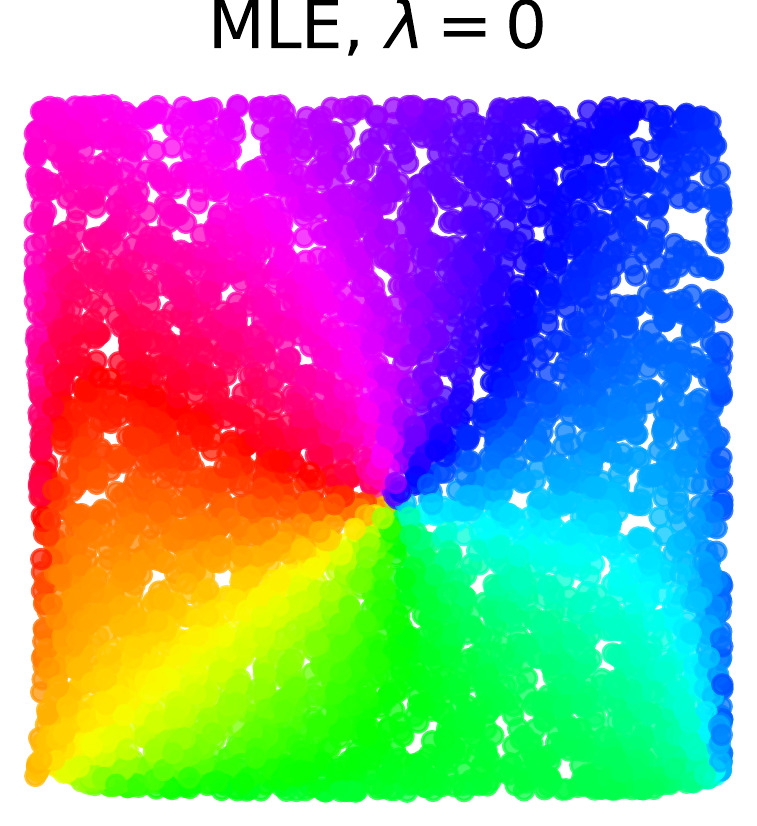}
            \end{subfigure}
        \end{minipage}
        \hspace{\gap em}
        \begin{minipage}{\minipagewidth \textwidth}
            \begin{subfigure}{1.0\textwidth}
            \centering
            \includegraphics[height=\height cm, keepaspectratio]{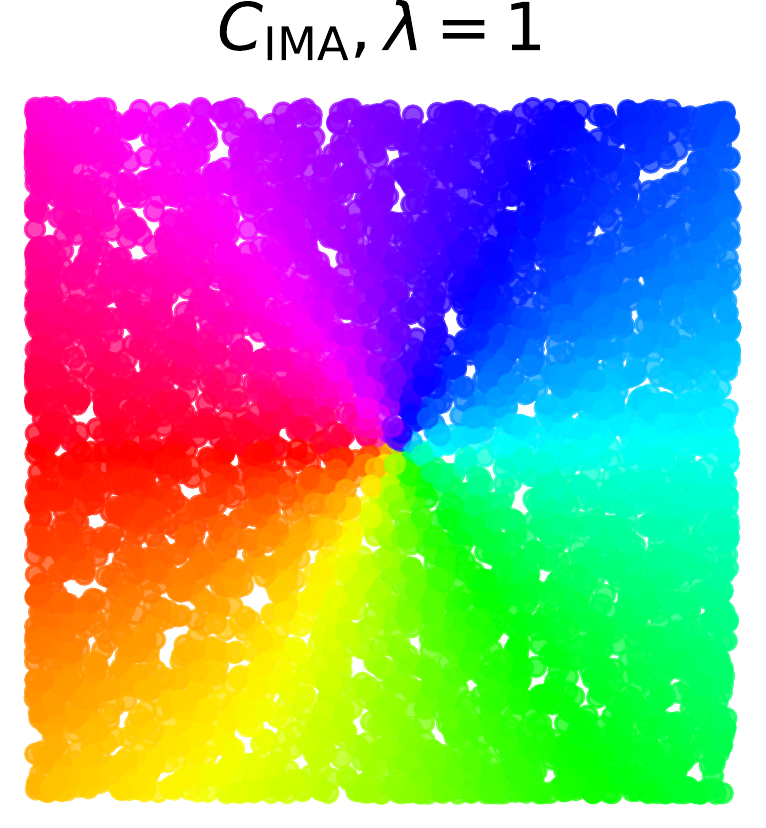}
            \end{subfigure}
        \end{minipage}
    \vspace{0.2em}
    \begin{subfigure}[b]{0.24\textwidth}
        \centering

        \begin{overpic}[height=\heightbottom cm]{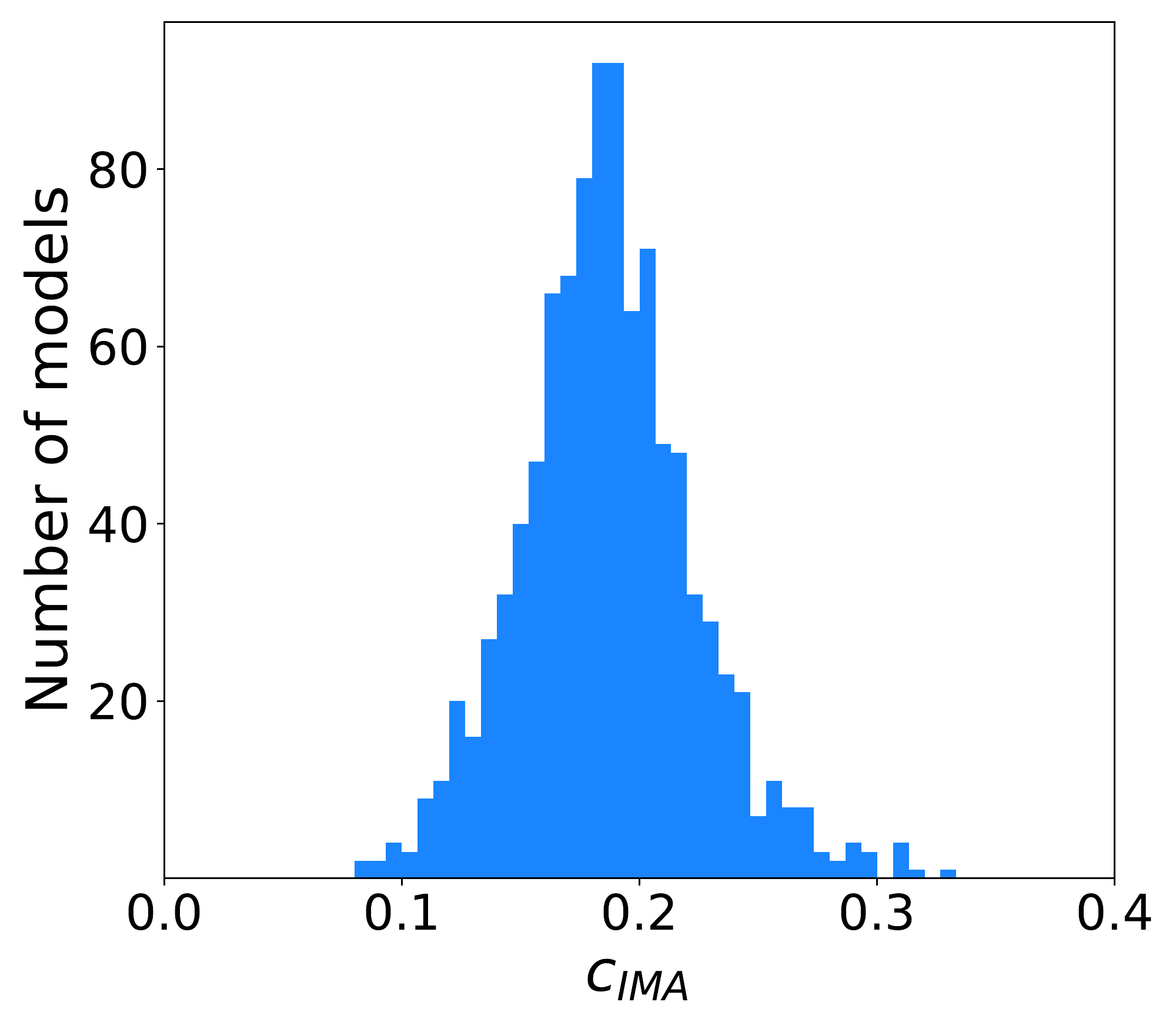}
 \put (\leftplace, \lowplace) {\textbf{\small(a)}}
\end{overpic}
    \end{subfigure}%
    \begin{subfigure}[b]{0.24\textwidth}
        \centering
        \begin{overpic}[height=\heightbottom cm]{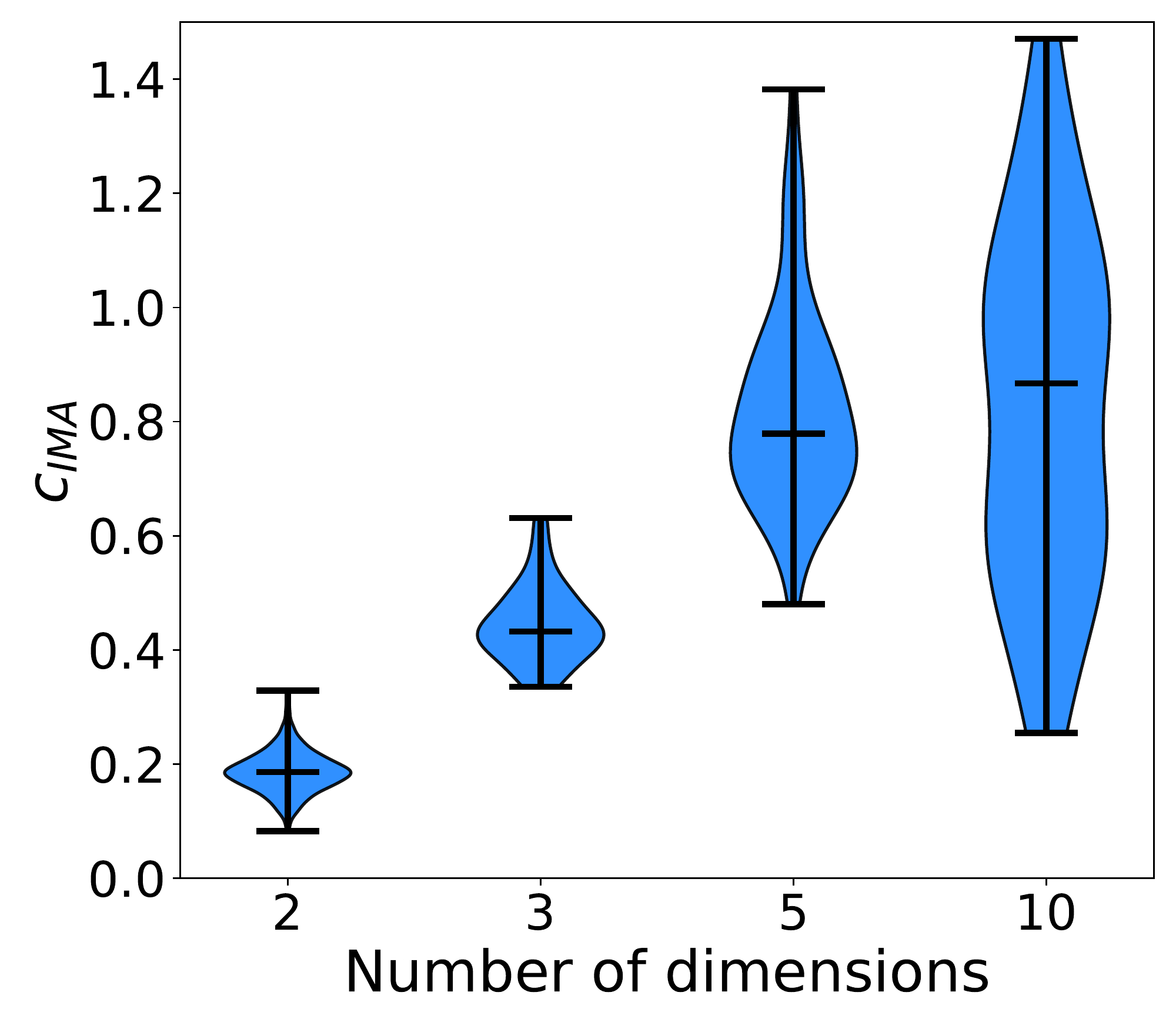}
 \put (\leftplace, \lowplace) {\textbf{\small(b)}}
\end{overpic}
    \end{subfigure}%
    \begin{subfigure}[b]{0.24\textwidth}
        \centering
                \begin{overpic}[height=\heightbottom cm]{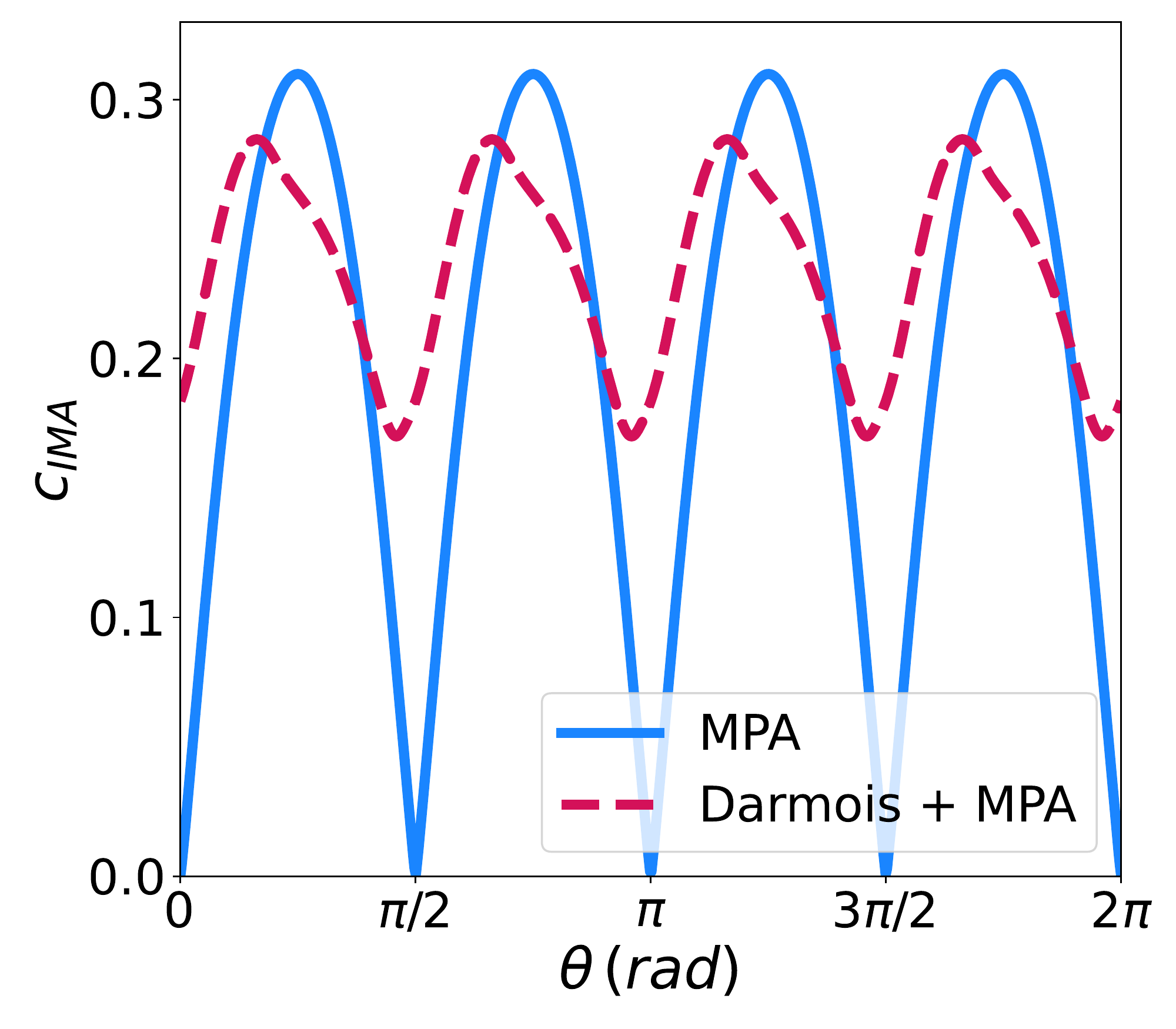}
 \put (\leftplace, \lowplace) {\textbf{\small(c)}}
\end{overpic}
    \end{subfigure}
    \begin{subfigure}[b]{0.24\textwidth}
        \centering
        \begin{overpic}[height=\heightbottom cm]{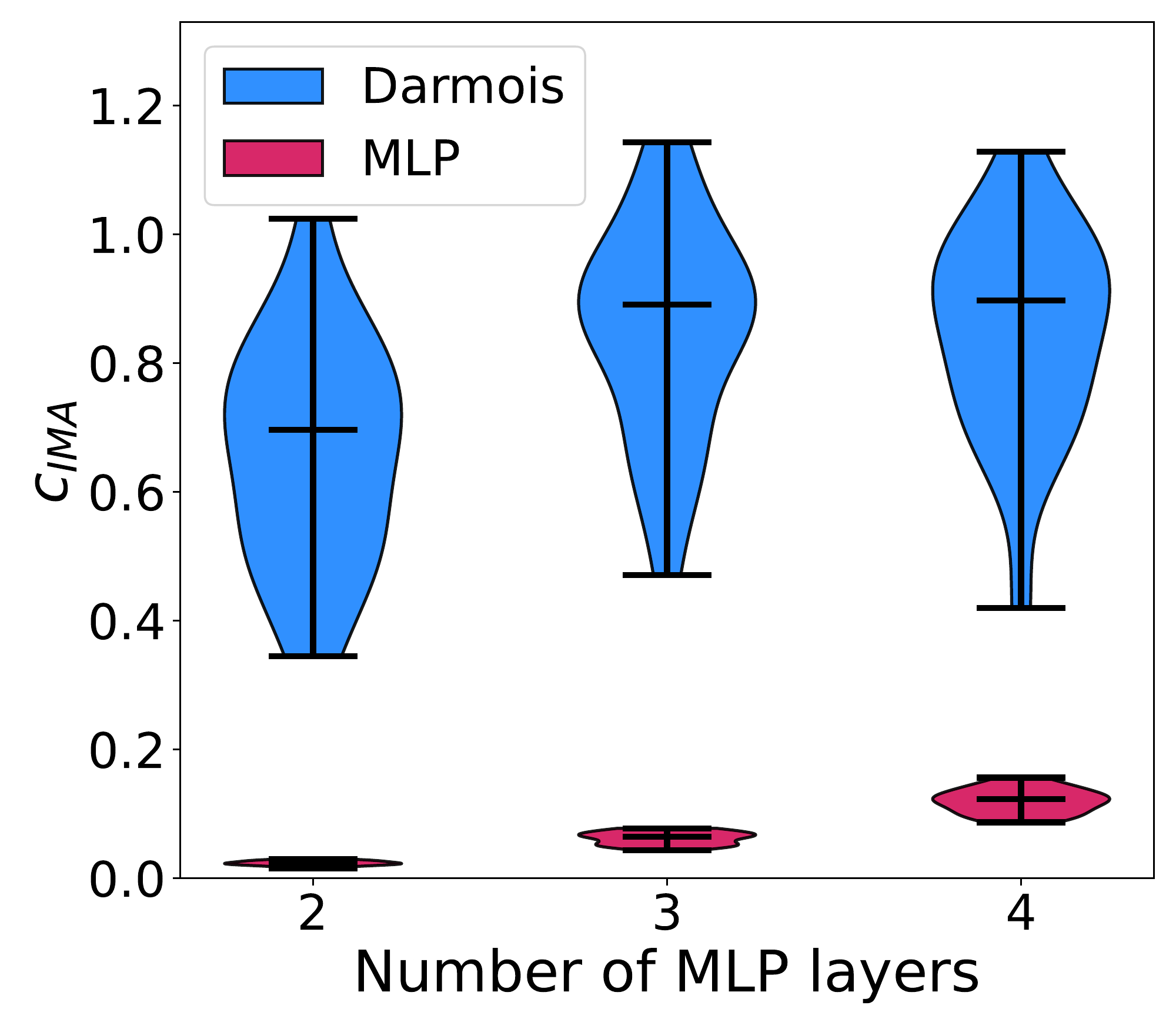}
 \put (\leftplace, \lowplace) {\textbf{\small(d)}}
\end{overpic}
    \end{subfigure}%
    \vspace{-0.5em}
    \caption{\small \textbf{Top.} Visual comparison of different nonlinear ICA solutions for $n=2$: \textit{(left to right)} true sources; observed mixtures;  Darmois solution; true unmixing, composed with the measure preserving automorphism (MPA) from~\eqref{eq:measure_preserving_automorphism_Gaussian} (with rotation by~$\nicefrac{\pi}{4}$); Darmois solution composed with the same MPA; maximum likelihood~($\lambda=0$); %
    and $C_\IMA$-regularised approach~($\lambda=1$).
    \textbf{Bottom.} Quantitative comparison of $C_\IMA$ for different spurious solutions:
    learnt Darmois solutions for \textbf{(a)} $n=2$, and \textbf{(b)} $n\in\{2, 3, 5, 10\}$ dimensions;
    \textbf{(c)} composition of the MPA~\eqref{eq:measure_preserving_automorphism_Gaussian} in $n=2$ dim.\ with the true solution (blue) and a  Darmois solution (red) for different angles. \textbf{(d)} 
    $C_\IMA$ distribution for true MLP mixing (red)  vs. Darmois solution (blue) for $n=5$ dim., $L\in\{2,3,4\}$ layers.
    }
    \label{fig:results1}
\end{figure}

%% file: sections/6_discussion.tex
\textbf{Assumptions on the mixing function.} Instead of relying on weak supervision in the form of auxiliary variables~\cite{hyvarinen2016unsupervised, hyvarinen2017nonlinear, hyvarinen2019nonlinear, gresele2020incomplete, halva2020hidden, khemakhem2020variational},
our IMA approach
places additional constraints on the functional form of the mixing process.
In a similar vein,
the \textit{minimal nonlinear distortion principle}~\cite{zhang2008minimal}
proposes to
favor solutions 
that are
as close to linear as possible. Another  %
example is the \textit{post-nonlinear model}~\cite{taleb1999source, zhang2009identifiability}, which assumes an element-wise nonlinearity applied after a linear mixing.
IMA is different in that it still allows for strongly nonlinear mixings~(see, e.g.,~\cref{fig:orthogonal_coordinate_transform}) provided that the columns of their Jacobians are (close to) orthogonal.
\looseness-1 In the related
field of disentanglement~\cite{bengio2013representation,locatello2019challenging}, a 
line of work that focuses on image generation with
adversarial networks~\cite{goodfellow2014generative} similarly proposes to constrain the ``generator'' function
via 
regularisation of its Jacobian~\cite{ramesh2018spectral} or Hessian~\cite{peebles2020hessian}, though mostly from an empirically-driven, rather than from an identifiability perspective as in the present work.

\textbf{Towards identifiability with $C_\IMA$.} The IMA principle rules out a large class of spurious solutions to nonlinear ICA. While we do not present a full identifiability result, our experiments show that $C_\IMA$ can be used to recover the BSS equivalence class, suggesting that identifiability might indeed hold, possibly under additional assumptions---e.g., for conformal maps~\cite{hyvarinen1999nonlinear}.

\textbf{IMA and independence of cause and mechanism.} While inspired by measures of independence of cause and mechanism as traditionally used for cause-effect inference~\cite{daniuvsis2010inferring, janzing2012information, janzing2010telling, zscheischler2011testing}, we view the IMA principle as addressing a different question,
in the sense that they evaluate independence between different elements of the causal model. 
Any nonlinear ICA solution that satisfies the IMA~\cref{principle:IMA} can be turned into one with uniform reconstructed sources---thus satisfying IGCI as argued in~\cref{sec:unsuitability_of_existing_ICM_measures}---through composition with an element-wise transformation which, according to~\cref{prop:global_IMA_contrast_properties} \textit{(ii)}, leaves the $C_\IMA$ value unchanged.
Both IGCI~\eqref{eq:IGCI_condition} and IMA~\eqref{eq:IMA_principle} can therefore be fulfilled simultaneosly, while the former on its own is inconsequential for BSS as shown in~\cref{prop:IGCI_insufficient_for_BSS}.

\textbf{BSS through algorithmic information.} 
Algorithmic information theory has previously been proposed as a unifying framework for identifiable approaches to \textit{linear} BSS~\citep{pajunen1998blind, pajunen1999blind}, in the sense that commonly-used contrast functions could, under suitable assumptions, be interpreted as proxies for the total complexity of the mixing and the reconstructed sources.
However, to the best of our knowledge, the problem of specifying suitable proxies for the complexity of \textit{nonlinear} mixing functions has not yet been
addressed.
We conjecture that our framework could be linked to this view, based on the additional assumption of algorithmic independence of causal mechanisms~\cite{janzing2010causal}, thus potentially representing an approach to \textit{nonlinear} BSS~by minimisation of algorithmic complexity.

\textbf{ICA for causal inference \& causality for ICA.}
Past advances in ICA have inspired novel causal discovery methods~\cite{shimizu2006linear,monti2020causal,khemakhem2021causal}.
The present work constitutes, to the best of our knowledge, the first effort to use ideas from causality (specifically ICM) for BSS.
An application of the IMA principle to causal discovery or causal representation learning~\cite{scholkopf2021toward} is an interesting direction for future work.

\textbf{Conclusion.} 
We introduce
IMA, a path to nonlinear BSS inspired by concepts from causality.
We postulate that the \textit{influences}
of different sources on the observed distribution should be approximately independent, and formalise this as an orthogonality condition on the columns of the Jacobian.
We prove that this constraint is generally violated by well-known spurious nonlinear ICA solutions, and propose a regularised maximum likelihood approach which we empirically demonstrate to be effective in recovering the true solution.
Our IMA principle holds exactly for orthogonal coordinate transformations, and is thus of potential interest for learning spatial representations~\cite{hinton1981frames}, robot dynamics~\cite{mistry2010inverse}, or physics problems where orthogonal reference frames are common~\cite{moon1971field}.

%% file: sections/checklist.tex
\clearpage

\section*{Checklist}

\begin{enumerate}

\item For all authors...
\begin{enumerate}
  \item Do the main claims made in the abstract and introduction accurately reflect the paper's contributions and scope?
    \answerYes{}
  \item Did you describe the limitations of your work?
    \answerYes{See~\cref{sec:discussion}, where we discuss limitations of our theory (e.g. lines 354-357) and open questions.}
  \item Did you discuss any potential negative societal impacts of your work?
    \answerNA{Our work is mainly theoretical, and we believe it does not bear immediate negative societal impacts.}
  \item Have you read the ethics review guidelines and ensured that your paper conforms to them?
    \answerYes{}
\end{enumerate}

\item If you are including theoretical results...
\begin{enumerate}
  \item Did you state the full set of assumptions of all theoretical results?
    \answerYes{We formally define the problem setting in~\cref{sec:background} and~\cref{sec:unsuitability_of_existing_ICM_measures}, and transparently state and discuss our assumptions in~\cref{sec:IMA}.
    }
	\item Did you include complete proofs of all theoretical results?
    \answerYes{Due to space constraints, full proofs and detailed explanations are mainly reported in~\cref{app:proofs} and~\cref{app:examples}; the proof of~\cref{prop:IGCI_insufficient_for_BSS} is given in the main text.}
\end{enumerate}

\item If you ran experiments...
\begin{enumerate}
  \item Did you include the code, data, and instructions needed to reproduce the main experimental results (either in the supplemental material or as a URL)?
    \answerYes{The code, data, and the configuration files are included in the supplemental material.}
  \item Did you specify all the training details (e.g., data splits, hyperparameters, how they were chosen)?
    \answerYes{All details, including hyperparameters, seed for random number generators, etc. are specified in the configuration files which will be included in the supplemental.}
	\item Did you report error bars (e.g., with respect to the random seed after running experiments multiple times)?
    \answerYes{We visualized the distribution of the considered quantities via histograms and violin plots, see e.g. \autoref{fig:results1} and \autoref{fig:results2}.}
	\item Did you include the total amount of compute and the type of resources used (e.g., type of GPUs, internal cluster, or cloud provider)?
    \answerYes{They are specified~\cref{app:experiments}.}
\end{enumerate}

\item If you are using existing assets (e.g., code, data, models) or curating/releasing new assets...
\begin{enumerate}
  \item If your work uses existing assets, did you cite the creators?
    \answerYes{We use the Python libraries JAX and Distrax and cited the creators in the article.}
  \item Did you mention the license of the assets?
    \answerYes{Both packages have Apache License 2.0; we report this in the appendices.}
  \item Did you include any new assets either in the supplemental material or as a URL?
    \answerYes{Implementations of our proposed methods and metrics will be provided in the supplemental.}
  \item Did you discuss whether and how consent was obtained from people whose data you're using/curating?
    \answerNA{}
  \item Did you discuss whether the data you are using/curating contains personally identifiable information or offensive content?
    \answerNA{}
\end{enumerate}

\item If you used crowdsourcing or conducted research with human subjects...
\begin{enumerate}
  \item Did you include the full text of instructions given to participants and screenshots, if applicable?
    \answerNA{}
  \item Did you describe any potential participant risks, with links to Institutional Review Board (IRB) approvals, if applicable?
    \answerNA{}
  \item Did you include the estimated hourly wage paid to participants and the total amount spent on participant compensation?
    \answerNA{}
\end{enumerate}

\end{enumerate}

%% file: appendix/A_identifiability_and_linear_ICA.tex
\label{app:linear_ICA}
In this Appendix, 
we provide additional background on the notion of identifiability and illustrate it using the example of linear ICA. 

\subsection{Identifiability in terms of equivalence relations}
\label{app:details_on_identifiability}

Traditionally, identifiability for a class of models $p_\theta$ for observed data $\xb$ parametrised by $\theta\in\Theta$ is expressed as the condition that there needs to be a one-to-one mapping between the space of models and the space of parameters, i.e., the model class $p_\theta$
is said to be identifiable if
\begin{equation}
\label{eq:identifiability_definition}
    \forall \theta, \theta'\in\Theta:
    \quad \quad 
    p_\theta(\xb)=p_{\theta'}(\xb) \forall \xb
    \quad
    \implies 
    \quad 
    \theta= \theta'.
\end{equation}
However, the equality on the RHS of~\eqref{eq:identifiability_definition} is a very strong condition which makes this type of (strong or unique) identifiability impractical for many settings.
For example, in the case of (linear or nonlinear) ICA, the ordering of the sources cannot be determined, so strong identifiability in the sense of~\eqref{eq:identifiability_definition} is infeasible.

The equality in parameter space on the RHS of the implication in~\eqref{eq:identifiability_definition} is therefore sometimes replaced by an equivalence relation $\sim$~\cite{khemakhem2020variational}, as is also the case for our~\cref{def:identifiability}.
An equivalence relation $\sim$ on a set $A$ is a binary relation between pairs of elements of $A$ which satisfies the following three properties:
\begin{enumerate}
    \item Reflexivity: $a\sim a$, $\forall a\in A$. 
    \item Symmetry: $a\sim b \implies b\sim a$, $\forall a,b\in A$.
    \item Transitivity: $(a\sim b) \land (b\sim c) \implies a\sim c $.
\end{enumerate}
An equivalence relation on a set $A$ imposes a partition into disjoint subsets. 
Each such subset corresponds to an equivalence class, i.e., the collection of all elements which are $\sim$-related to each other; for example, $[a]=\{b\in A:a\sim b\}$ denotes the equivalence class containing the element~$a$.

A trivial example of an equivalence relation is equality ($=$).
More useful examples in the context of ICA are equivalence up to permutation, rescaling, or scalar transformation.

Defining an appropriate equivalence class for the problem at hand therefore allows us to specify exactly the type of indeterminancies which cannot be resolved and up to which the true generative process can be recovered.
As argued in~\cref{sec:background}, for nonlinear ICA, the desired notion of identifiability---in the sense of the strongest feasible type of identifiable that is possible without further (parametric) assumptions---is captured by $\sim_\BSS$ from~\cref{def:bss_identifiability}.
We give another example for linear ICA in~\Cref{sec:linear_ICA}.

Since the generative process of nonlinear ICA~\eqref{eq:gen} is determined by the choice of mixing function and source distribution, the space $\Theta$ from~\eqref{eq:identifiability_definition}, in this case, corresponds to the product space of the space of mixing functions $\Fcal$ and source distributions $\Pcal$.
Moreover, the pushforward density $\fb_*p_\sb$ in~\cref{def:identifiability} corresponds to the density of the observed mixtures $p_\xb$, or $p_\theta(\xb)$ in~\eqref{eq:identifiability_definition}.

We deliberately choose to define identifiability and to express the observed distribution in terms of the source distribution and the mixing function---as opposed to in terms of the observed distribution and the unmixing function as in some prior work~\cite{hyvarinen2016unsupervised,hyvarinen2017nonlinear,hyvarinen2019nonlinear}---because this is aligned with the causal direction of data generation, and thus more consistent with the causal perspective at nonlinear BSS taken in the present work.
We also believe that, in this framework, separate constraints on the space of mixing functions $\Fcal$ and source distributions $\Pcal$ are expressed more naturally.

Next, we illustrate the above ideas for the well-studied case of linear ICA.

\subsection{Identifiability of linear ICA}
\label{sec:linear_ICA}
Linear ICA corresponds to the setting in which a linear mixing is applied to independent sources, i.e.,%
\begin{equation}
\label{eq:linear_ICA_model}
\xb = \Ab \sb,
\end{equation}
where $\Ab \in \mathbb{R}^{n \times n}$ is an invertible mixing matrix.
The source variables $\sb$ can be assumed to have zero mean
without affecting estimation of the mixing matrix,
and the ordering and
variances 
of the independent components cannot be determined, so
it is customary to assume $\EE[s_i^2]=1$~\cite{ICAbook}. 

Additionally, we can assume w.l.o.g. that the mixing matrix is orthogonal ($\Ab \Ab^\top = \Ib$), 
because we can always \textit{whiten} $\xb$ first through an invertible linear transformation and obtain an orthogonal mixing~\cite{ICAbook}, as explained in more detail in~\Cref{app:whitening}.

Now suppose that the reconstructed sources
\begin{equation}
    \yb = \Bb \xb = \Bb \Ab \sb
\end{equation}
have independent components for some orthogonal unmixing matrix~$\Bb \in \mathbb{R}^{n \times n}$.
Then $\Cb = \Bb \Ab$ is also orthogonal and
the following type of identifiability holds~\cite{darmois1953analyse,skitovic1953property, comon1994independent}.%
\begin{theorem}[Identifiability of linear ICA; based on Thm. 11 of~\cite{comon1994independent}] 
\label{thm:identifiability_of_linear_ICA}
Let $\sb$ be a vector of $n$ independent components, of which at most one is Gaussian and whose densities are not reduced to a point mass.
Let $\Cb \in\RR^{n\times n}$
be an orthogonal matrix. Then $\yb=\Cb\sb$ has (mutually) independent components iff.\
$\Cb=\bm\Db \Pb$, with $\Db$
a diagonal matrix and $\Pb$
a permutation matrix.%
\end{theorem}%
\Cref{thm:identifiability_of_linear_ICA} shows that
the two ambiguities deemed unresolvable (scale and ordering of the sources) are, in fact, the only ambiguities,
as long as at most one of the $s_i$ is Gaussian.
That is, linear ICA is identifiable up to rescaling and permutation of the sources, i.e., linearly transforming the observations $\xb$
into independent components is equivalent to separating the sources.

More formally, in terms of an equivalence relation, if we take $\Fcal'$ from~\eqref{eq:identifiability} as the space of invertible $n\times n$ matrices and $\Pcal'$ as the space of source distributions with at most one Gaussian marginal, then linear ICA is $\sim_\textsc{lin}$-identifiable on $\Fcal'\times\Pcal'$
where the equivalence relation $\sim_\textsc{lin}$ on $\Fcal'$ is defined as
\[\Bb\sim_\textsc{lin} \Bb' \iff \exists \Db,\Pb \,\, \text{s.t.} \,\, \Bb = \Db\Pb\Bb'.\]

\paragraph{Beyond non-Gaussianity.} 
Two other deviations from a Gaussian i.i.d.\ setting lead to identifiability:  nonstationarity~\cite{pham2001blind} and time correlation~\cite{pham1997blind}. 
A general 
information-geometric framework links these three 
different
routes to identifiability~\cite{cardoso2001three}.

\subsection{Whitening in the context of linear ICA}
\label{app:whitening}
For completeness, we give a brief account of the role of \textit{whitening in linear ICA}, which was mentioned in~\ref{sec:linear_ICA} and which again plays a role in~\ref{app:trace_method}.
The following exposition is partly based on~\cite{ICAbook}, \S 7.4.2.
A zero-mean random vector, say $\yb$, is said to be \textit{white} if its components are uncorrelated and their variances equal unity. In other words, the covariance matrix %
of $\yb$ is equal to the identity matrix:
\[
\EE \left[\mathbf{y} \mathbf{y}^\top\right]=\mathbf{I}\,.
\]
It is always possible to whiten a  zero-mean random vector $\xb$ through a linear operation,
\begin{equation}
    \zb = \Vb \xb\,.
    \label{eq:whiten_linear_op}
\end{equation}
As an example, a popular method for whitening uses the eigenvalue decomposition (EVD) of the covariance matrix,
\[
\EE\left[\mathbf{x} \mathbf{x}^\top\right]=\mathbf{E} \mathbf{D} \mathbf{E}^\top
\]
where $\mathbf{E}$ is the orthogonal matrix of eigenvectors of $\EE\left[\mathbf{x} \mathbf{x}^\top\right]$ and $\mathbf{D}$ is the diagonal matrix of its eigenvalues, $\mathbf{D}=\operatorname{diag}\left(\lambda_{1}, \ldots, \lambda_{n}\right)$. 
Note that the covariance matrix is a symmetric matrix, therefore it is diagonalisable.
Whitening can then be performed by substituting in~\eqref{eq:whiten_linear_op} the matrix
\begin{equation}
\mathbf{V}=\mathbf{E} \mathbf{D}^{-1 / 2} \mathbf{E}^\top \,.
\label{eq:whiten_linear}
\end{equation}
so that 
\[
\EE[\zb\zb^\top]=\mathbf{E} \mathbf{D}^{-1 / 2} \mathbf{E}^\top \mathbf{E} \mathbf{D} \mathbf{E}^\top\mathbf{E} \mathbf{D}^{-1 / 2} \mathbf{E}^\top=\Ib
\]

\textbf{Whitening is only half ICA.}
Assume a linear ICA model,
\begin{equation}
\xb = \Ab \sb\,. \label{eq:linearrica}
\end{equation}
and suppose that the observed data is whitened, for example, by the matrix $\Vb$ given in~\eqref{eq:whiten_linear}. 
Whitening transforms the mixing matrix into a new one, $\tilde{\Ab}=\Vb\Ab$. 
We have from~\eqref{eq:linearrica} and~\eqref{eq:whiten_linear}
\[
\mathbf{z}=\mathbf{V A s}=\tilde{\mathbf{A}} \mathbf{s}
\]
Note that whitening does not solve linear ICA, since \textit{uncorrelatedness is weaker than independence}. 
To see this, consider any orthogonal transformation $\mathbf{U}$ of $\mathbf{z}$:
\[
\mathbf{y}=\mathbf{U} \mathbf{z}.
\]
Due to the orthogonality of $\mathbf{U},$ we have
\[
\EE\left[\mathbf{y} \mathbf{y}^\top\right]=\EE\left[\mathbf{U} \mathbf{z} \mathbf{z}^\top \mathbf{U}^\top\right]
=\Ub \EE\left[\mathbf{z} \mathbf{z}^\top\right] \Ub^T
=
\mathbf{U I U}^\top=\mathbf{I}\,,
\]
so, $\yb$ is white as well. Thus, we cannot tell if the independent components are given by $\mathbf{z}$ or $\mathbf{y}$ using the whiteness property alone. Since $\mathbf{y}$ could be any orthogonal transformation of $\mathbf{z},$ whitening gives the independent components only up to an orthogonal transformation. %

On the other hand, whitening is useful as a pre-processing step in ICA: its utility resides in the fact that the new mixing matrix $\tilde{\mathbf{A}}= \Vb \Ab$ is orthogonal. This can be seen from
\[
\EE\left[\mathbf{z z}^\top\right]=\tilde{\mathbf{A}} \EE\left[\mathbf{s s}^\top\right] \tilde{\mathbf{A}}^\top=\tilde{\mathbf{A}} \tilde{\mathbf{A}}^\top=\mathbf{I}.
\]
We can thus restrict the search for the (un)mixing matrix to the space of orthogonal matrices. Instead of having to estimate $n^{2}$ parameters (the elements of the original matrix $\mathbf{A}$), we only need to estimate an orthogonal mixing matrix $\tilde{\mathbf{A}}$ which contains $n(n-1) / 2$ degrees of freedom; e.g., in two dimensions, an orthogonal transformation is determined by a single angle parameter. For larger $n$, an orthogonal matrix contains only about half of the number of parameters of an arbitrary matrix.

Whitening thus ``solves half of the problem of ICA''. Because whitening is a very simple and standard procedure---much simpler than any ICA algorithm---it is a good idea to reduce the complexity of the problem this way. The remaining half of the parameters has to be estimated by some other method.

%% file: appendix/B_trace_method_igci.tex
We now provide additional discussion of the ICM principle and its connection to ICA and IMA.
First, we introduce a linear ICM criterion and discuss its relation with linear ICA in~\Cref{app:trace_method}.

\subsection{Trace method}
\label{app:trace_method}
As mentioned in~\cref{sec:background_causality}, besides IGCI, another existing ICM criterion that is closely related to ICA due to also assuming a deterministic relation between cause $\cb$ and effect $\eb$ 
is 
the \textit{trace method}~\cite{janzing2010telling,zscheischler2011testing}.
The trace method
assumes a linear relationship, 
\begin{equation}
\label{eq:trace_method_model}
    \eb=\Ab\cb,
\end{equation}
and formulates ICM as an ``independence'' between the covariance matrix $\bm\Sigma$ of $\cb$ and the mechanism~$\Ab$ (which, as for IGCI, we can again think of as a degenerate conditional $p_{\eb|\cb}$) via 
the condition%
\begin{equation}
\label{eq:trace_condition}
\textstyle
\tau(\Ab\bm\Sigma\Ab^\top)=\tau(\bm\Sigma)\tau(\Ab\Ab^\top)   
\end{equation}
where $\tau(\cdot)$ denotes the renormalized trace. 
Intuitively, this condition~\eqref{eq:trace_condition} rules out a fine-tuning of $\Ab$ to the eigenvectors of $\bm\Sigma$ which would violate the assumption of no shared information between the cause distribution (specifically, its covariance structure) and the mechanism.

As with IGCI and nonlinear ICA, it can be seen by comparing~\eqref{eq:trace_method_model} and~\eqref{eq:linear_ICA_model} that \textit{the trace method assumes the same generative model as linear ICA} (where the cause $\cb$ corresponds to the independent sources $\sb$ and the effect to the observed mixtures $\xb$).
While the focus of the present work is on nonlinear ICA, we briefly discuss the usefulness of the trace method as a constraint for achieving identifiability in a linear ICA setting.

As is clear from~\eqref{eq:trace_condition}, the trace condition is trivially satisfied if the covariance matrix of the sources (causes) is the identity, $\Sigmab=\Ib$.
However, as explained in~\Cref{app:whitening}, in the context of linear ICA this can easily be achieved by whitening the data.
As with IGCI, the trace method was developed for cause-effect inference where both variables are observed, and thus relies on the observed cause distribution being informative.
This renders is unsuitable (on its own) to constrain the unsupervised representation learning problem of linear ICA problem where the sources are unobserved. 

Note, however, that this is qualitatively different from the IGCI argument presented in~\cref{sec:unsuitability_of_existing_ICM_measures}, as whitening on its own does not necessarily lead to independent variables, but only uncorrelated ones, and thus does not solve linear ICA---unlike the Darmois construction in the case of nonlinear ICA which also yields independent components.

\subsection{Information geometric interpretation of
the ICM principle}
\label{sec:information_geometric_ICM}
There is a well-established connection between IGCI 
and the trace method~\cite{janzing2012information}. At the heart of this derivation lies an information-geometric interpretation of the ICM principle 
for probability distributions, which we sketch in this section. 
First, we need to review some basic concepts.

\paragraph{Background on information geometry.}
Information geometry~\cite{amari2007methods, amari2021information} is a discipline in which ideas from differential geometry are applied to probability theory. Probability distributions correspond to points on a Riemannian manifold, known as \textit{statistical manifold}. Equipped with the Kullback-Leibler (KL) divergence, also called the relative entropy distance, as a premetric,\footnote{A premetric on a set $\mathcal{X}$ is a function $d:\mathcal{X}\times \mathcal{X}\to \mathbb{R}^{+}\cup\{0\}$ such that (i) $d(x,y) \geq 0$ for all $x$ and $y$ in $\mathcal{X}$ and (ii) $d(x,x)=0$  for all $x\in\Xcal$. 
} one can study the geometrical properties of the statistical manifold. %
For two probability distributions $P$ and $Q$,
 we denote their KL divergence by $\KL(P\|Q)$, which is defined for $P$ absolutely continuous with respect to $Q$ as:
\[
\KL (P\|Q)=\int dP \log \frac{dP}{dQ}\,.
\]
An interesting property of the KL divergence is its invariance to reparametrisation. Consider an invertible transformation $h$, mapping random variables $X$ and $Y$ to $h(X)$ and $h(Y)$, respectively (the domains and codomains being arbitrary spaces, e.g., discrete or Euclidean of arbitrary dimension).
Then the KL divergence between $P_X$ and $P_Y$ is preserved by the pushforward operation implemented by $h$, such that
\begin{equation}
\label{eq:KL_invariance}
    \KL (P_{h(X)}\|P_{h(Y)}) = \KL (P_X\|P_Y)\,.
\end{equation}

\begin{figure}[t]
    \centering
    \includegraphics[width=0.5\textwidth]{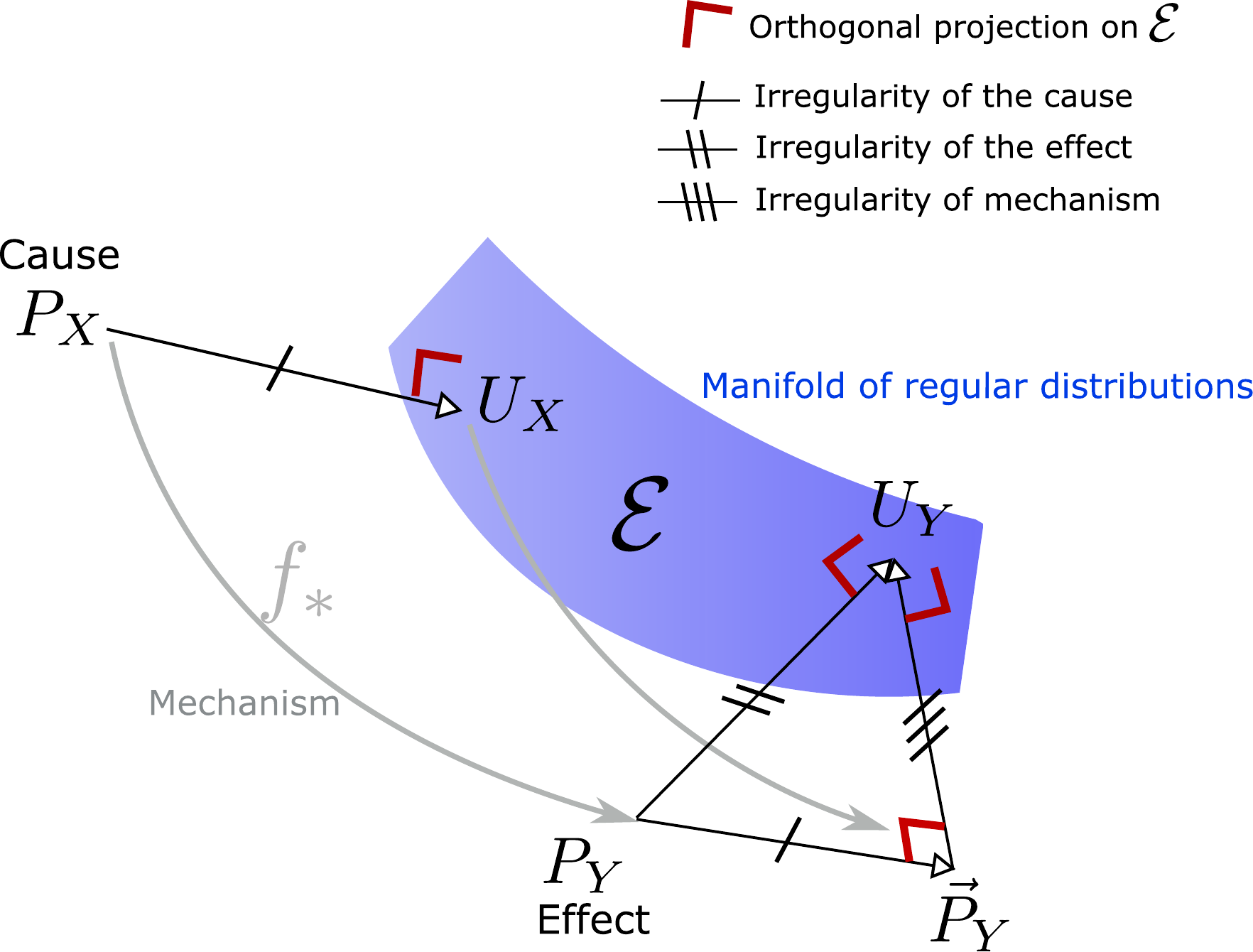}
    \caption{\small \textbf{Interpretation of the ICM principle as an orthogonality principle in information space.} The irregularity of the effect distribution, as measured by $\KL(P_Y\|U_Y)$, can be decomposed into the irregularities of the cause, as measured by $\KL(P_X\|U_X)$, and the irregularity of the mechanism $f$, as measured by $\KL(\vec{P}_Y\|U_Y)$. Here, $U_X$ and $U_Y$ denote the orthogonal projections of $P_X$ and $P_Y$ onto the manifold $\Ecal$ of regular distributions, and $\vec{P}_Y$ denotes the pushforward of the regular distribution $U_X$ via $f$. Note that the KL divergence is invariant to reparametrisation by invertible functions. }
    \label{fig:information_geometry}
\end{figure}

\paragraph{Interpretation of ICM as orthogonality condition in information space.
}
Consider a deterministic causal relationship of the form $Y:=f(X)$, and denote by $P_X$ and $P_Y$ the marginal distributions of the cause $X$ and the effect $Y$, respectively.
The ``irregularity'' of each distribution can be quantified by evaluating their divergence to a reference set $\mathcal{E}$ of ``regular'' distributions,\footnote{Here ``regular'' is only meant in an intuitive sense, not implying any further mathematical notion. If $\mathcal{E}$ is the set of Gaussians, for instance, the distance from  $\mathcal{E}$ measures non-Gaussianity.}  
\[
\KL(P_X\|\mathcal{E})=\inf_{U\in\mathcal{E}} \KL(P_X\|U),\; \quad \KL(P_Y\|\mathcal{E})=\inf_{U\in\mathcal{E}} \KL(P_Y\|U).
\]
Let us assume that these infima are reached at a unique point, their projections onto $\mathcal{E}$:
\[
U_X=\arg\min_{U\in\mathcal{E}} \KL(P_X\|U),\; \quad U_Y=\arg\min_{U\in\mathcal{E}} \KL(P_Y\|U).
\]
As elaborated in \cite[\S 4]{janzing2012information}, the choice of $\mathcal{E}$ is context-dependent. 
For example, in the context of the trace method \cite{janzing2010telling}, $X$ and $Y$ are assumed to be $n$-dimensional multivariate Gaussian random vectors, and $\mathcal{E}$ is taken as the set of multivariate \textit{isotropic} Gaussian distributions. 
In contrast, when IGCI is applied in contexts where the considered mechanism is a deterministic non-linear diffeomorphism, the reference distributions are typically uniform distributions~\cite{daniuvsis2010inferring, janzing2015justifying}.

Overall, it can be shown that the independence postulate underlying these approaches leads to the following decomposition of the irregularity of $P_Y$ (see~\cite[Thm.\ 2]{janzing2012information}):
\[
\KL(P_Y\|U_Y)=\KL(P_Y\|\vec{P}_Y)+\KL(\vec{P}_Y\|U_Y)
\]
where $\vec{P}_Y$ denotes the distribution of $f(U_X)$, i.e., the hypothetical distribution of the effect that would be obtained if the cause $X$ were replaced by  the random variable  $U_X$ (which corresponds to the closest regularly distributed random variable to $X$). 

Since applying the bijection $f^{-1}$ preserves the KL divergences, see~\eqref{eq:KL_invariance}, we can obtain the equivalent relation
\begin{equation}
\label{eq:addIrregIGCI}
\KL(P_Y\|U_Y)=\KL(P_X\|U_X)+\KL(\vec{P}_Y\|U_Y)\,.
\end{equation}

This relation can be interpreted as an \textit{orthogonality principle} in information space by considering the KL divergences as a generalization of the squared Euclidean norm for the difference vectors $\overrightarrow{P_YU_Y}$, $\overrightarrow{P_Y\vec{P_Y}}$ and $\overrightarrow{\vec{P_Y}U_Y}$.
It can thus be viewed as a Pythagorean theorem in the space of distributions, see~\cref{fig:information_geometry} for an illustration.

The orthogonality principle~\eqref{eq:addIrregIGCI} thus
captures a  decomposition of the irregularity $\KL(P_Y\|U_Y)
$ of $P_Y$ on the LHS into the sum of two irregularities on the RHS: the irregularity $\KL(P_X\|U_X)
$ of $P_X$, and the term $\KL(\vec{P}_Y\|U_Y)
$ which
measures the irregularity of the mechanism $f$ indirectly, via the ``irregularity'' of the distribution resulting from applying $f$ to a regular distribution $U_Y$.

Overall, the decomposition~\eqref{eq:addIrregIGCI} links the 
postulate of independence between the cause distribution, on the one hand, and the mechanism, on the other hand, to an \textit{orthogonality of their irregularities in information space} (namely the statistical manifold of information geometry). 
As proposed in \cite{janzing2012information}, this can be intuitively interpreted as a geometric form of independence if we assume that Nature chooses such irregularities independently of each other, and ``isotropically'' in a high-dimensional subspace of irregularities.

While, to date, we are not aware of similar results in the context of information geometry (i.e., on the statistical manifold), this intuition is supported by concentration of measure results in Euclidean spaces. Indeed, in high-dimensions, it is likely that two vectors are close to orthogonal if they are chosen independently according to a uniform prior~\cite{gorban2018blessing}. 

We will take inspiration of the decomposition~\eqref{eq:addIrregIGCI} to justify IMA in the following section.

\subsection{Decoupling of the influences in IMA and comparison with IGCI}
\label{app:ima_igci}

In contrast to~\Cref{sec:information_geometric_ICM}, in this section we will, for notational consistency with the main paper, assume that all distributions have a density with respect to the Lebesgue measure, and thus consider, with a slight abuse of notation, that the KL divergence is a distance between two densities on the relevant support, such that
\[
\KL(p\|q) = \int p(x)\log \frac{p(x)}{q(x)} dx\,.
\]

\paragraph{Overview.} In line with the information-geometric interpretation of IGCI presented in~\Cref{sec:information_geometric_ICM}, we also consider an interpretation of IMA in information space. 
We consider the KL-divergence between the observed density $p_{\xb}$ of $\xb=\fb(\sb)$ and an \textit{interventional} distribution~$p_{\widehat{\xb}}$ of $\widehat{\xb}=\widehat{\fb}(\sb)$, resulting from a soft intervention that replaces the mixing function $\fb$ with another mixing $\widehat{\fb}$. 
We take $\KL(p_{\xb}\|p_{\widehat{\xb}})$ as a measure of the causal effect of the soft intervention 
(or perturbation)
that turns $\fb$ into $\widehat{\fb}$---similarly to how $\KL(P_Y\|U_Y)$ is used as a measure of the irregularity of the effect distribution in the context of IGCI~(\Cref{sec:information_geometric_ICM}).

As we will show, under suitable assumptions, the functional form imposed on $\fb$ by the IMA~\Cref{principle:IMA} can lead to a decomposition of the \textit{causal effect} of an intervention on the mechanism into a sum of terms, corresponding to the causal effects of separate soft interventions on the mechanisms associated to each source. In contrast, IGCI decomposes \textit{irregularities} of the effect distribution into \textit{two terms, one irregularity of the cause and one irregularity of the mechanism}.

\paragraph{Soft-interventions on the individual mechanisms.}
Assume $\fb$ satisfies the IMA principle. We consider interventions performed through the element-wise transformation $\sib$ such that 
\[
\sib :\sb\mapsto 
\left[ 
\begin{array}{cc}
    \si_1 (s_1) \\
     \vdots \\
     \si_j (s_j)\\
     \vdots \\
     \si_n (s_n)
\end{array}
\right]\,.
\]
This can be seen as a composition  of  $n$ soft interventions $\{\sib_j\}$ on each individual source component~$j$, implemented through univariate smooth diffeomorphisms $\si_j$, such that
\[
\sib_j :\sb\mapsto 
\left[ 
\begin{array}{cc}
    s_1 \\
     \vdots \\
     \si_j (s_j)\\
     \vdots \\
     s_n
\end{array}
\right]\,,
\]
and $\sib = \sib_n \circ \cdots \circ \sib_1$ (in arbitrary order, since the individual $\sib_j$ commute). This soft intervention can be seen as turning the random variable $\sb$ into $\widehat{\sb}$, yielding the intervened observations $\widehat{\xb}=\fb (\widehat{\sb})$. Alternatively, the intervention on $\xb$ can be implemented by replacing $\fb$ by $\widehat{\fb}=\fb \circ \sib$---i.e., $\widehat{\xb}=\widehat{\fb} (\sb)$. Notably, since $\fb$ satisfies the IMA principle, so does $\widehat{\fb}$ (due to~\Cref{prop:global_IMA_contrast_properties}, \textit{(ii)}, since $\sib$ is an element-wise nonlinearity). Moreover, the partial derivatives of the intervened function are given by 
\[
\frac{\partial \widehat{\fb}}{\partial s_i}(\sb)= \frac{\partial \fb}{\partial s_i}(\sib(\sb))\left|\frac{d \si_i}{d s_i}\right|(s_i)\,.
\]
The classical change of variable formula for bijection $\fb$ yields the expression of the pushforward density of $\xb$ as
\[
p_{\xb}(\xb) = |J_{\fb}(\fb^{-1}(\xb))|^{-1} p_{\sb}(\fb^{-1}(\xb))\,,
\]
and for $\widehat{\xb}$ we get
\[
p_{\widehat{\xb}}(\widehat{\xb}) = |J_{\widehat{\fb}}(\widehat{\fb}^{-1}(\widehat{\xb}))|^{-1} p_{\sb}(\widehat{\fb}^{-1}(\widehat{\xb}))\,,
\]

\paragraph{Information geometric interpretation of IMA.}

Let us now compute the KL divergence between the intervened and observed distribution,
\begin{equation}
  \KL(p_{\xb}\|p_{\widehat{\xb}}) = \int p_{\xb}(\xb) \log \frac{p_{\xb}(\xb)}{p_{\widehat{\xb}}(\xb)} d\xb\,.  
  \label{eq:klintervnotinterv}
\end{equation}
Expressing the density of the observed variables as a pushforward of the density of the sources, and without additional assumptions on $\fb$ and $\widehat{\fb}$ besides smoothness and invertibility, we get,  
\begin{eqnarray*}
\KL(p_{\xb}\|p_{\widehat{\xb}}) & 
= &\mathlarger{\int} \left| \Jb_{\fb} (\fb^{-1}(\xb))\right|^{-1} p_{\sb}(\fb^{-1}(\xb)) \log \frac{ \left| \Jb_{\fb} (\fb^{-1}(\xb))\right|^{-1} p_{\sb}(\fb^{-1}(\xb))}{\left| \Jb_{\widehat{\fb}} (\widehat{\fb}^{-1}(\xb))\right|^{-1} p_{{\sb}}(\widehat{\fb}^{-1}(\xb))} d\xb\,.%
\end{eqnarray*}
We now consider a factorization of $\sb$ over a directed acyclic graph (DAG), such that
\[
p_{\sb}(\sb)= \prod_j p_{j}(s_j|\mbox{pa}(s_j))\,,
\]
where $\mbox{pa}(s_j)$ denotes the components associated to the parents of node $j$ in the DAG. Because $\sib$ is an element-wise transformation the factorization will be the same for $p_{\widehat{\sb}}$.

If we now additionally assume that $\fb$ and $\widehat{\fb}$ satisfy the IMA postulate, we get
\begin{eqnarray*}
\KL(p_{\xb}\|p_{\widehat{\xb}}) 
&= &\mathlarger{\int} \!\! \left| \Jb_{\fb} (\fb^{-1}(\xb))\right|^{-1}\!\! p_{\sb}(\fb^{-1}(\xb)) \sum_{i=1}^n \log \frac{ \left\| \frac{\partial \fb}{\partial s_i}( \fb^{-1}(\xb))\right\|^{-1} \!\!p_i(\fb^{-1}(\xb)_i|\mbox{pa}(\fb^{-1}(\xb)_i))}{\left\|  \frac{\partial \widehat{\fb}}{\partial s_i} (\widehat{\fb}^{-1}(\xb))\right\|^{-1} 
\!\!p_i(\widehat{\fb}^{-1}(\xb)_i|\mbox{pa}(\widehat{\fb}^{-1}(\xb)_i))}  d\xb\, . %
\end{eqnarray*}

By reparameterizing the integral in terms of the source coordinates, we get (using $\widehat{\fb}^{-1}=\sib^{-1}\circ {\fb}^{-1}$)
\begin{equation}\label{eq:KLdecomp1}
\KL(p_{\xb}\|p_{\widehat{\xb}}) 
=
\sum_{i=1}^n \mathlarger{\int}  p_{\sb}(\sb) \log \frac{ \left\| \frac{\partial \fb}{\partial s_i}( \sb)\right\|^{-1} p_i(\sb_i|\mbox{pa}(\sb_i))}{\left\|  \frac{\partial \widehat{\fb}}{\partial s_i} (\sib^{-1}(\sb))\right\|^{-1} p_i\left({\sib}^{-1}(\sb)_i|\mbox{pa}({\sib}^{-1}(\sb)_i)\right)}
d\sb\,.
\end{equation}
such that the $KL$ divergence can be written as a sum of $n$ terms, each associated to the intervention on a mechanism $\frac{\partial \fb}{\partial s_i}$.
Positivity of %
these terms would suggest that we can interpret each of them as quantifying the individual contribution of a soft intervention $\sib_j$ applied to the original sources.%

In the following, we propose a justification for the positivity of these terms \emph{in a restricted setting where only the $m$ leaf nodes of the graph are intervened on (with $1\leq m\leq n$).}\footnote{A leaf node in a DAG is one that does not have any descendants.}  %
In the special case of independent sources, all nodes are leaves and $m=n$. 

Under this assumption, we consider (without loss of generality) an ordering of the nodes such that the $m$ first nodes are the leaf nodes in the DAG. Then we argue that the terms of the right-hand side of \eqref{eq:KLdecomp1} associated to leaf nodes ($i \leq m$) are positive, as they correspond to the expectations of KL-divergences.
Indeed, taking one of the first $m$ terms, denoted $i$, we have the factorization 
\[
p_{\sb}(\sb)= p_i(s_i|\mbox{pa}(s_i))\prod_{j\neq i} p_{j}(s_j|\mbox{pa}(s_j))\,,
\]
where $\prod_{j\neq i} p_{j}(s_j|\mbox{pa}(s_j))$ does not depend on $s_i$ because node $i$ is a leaf node. Moreover, as non-leaf nodes are not intervened on, the transformation $\sib$ does not modify the value of any parent variables in these factorizations. As a consequence, the integral can be computed as an iterated integral with respect to $s_i$ and $\sb_{-i}\,$, where $\sb_{-i}$ denotes the vector including all source variables but $s_{i}$, such that
\begin{multline*}
    \mathlarger{\int}  p_{\sb}(\sb) \log \frac{ \left\| \frac{\partial \fb}{\partial s_i}( \sb)\right\|^{-1} \!\!p_i(\sb_i|\mbox{pa}(\sb_i))}
    {\left\|  \frac{\partial \widehat{\fb}}{\partial s_i} (\sib^{-1}(\sb))\right\|^{-1} \!\! p_i({\sib}^{-1}(\sb)_i|\mbox{pa}({\sib}^{-1}(\sb)_i))} d\sb \\
    = \mathbb{E}_{\sb_{-i}\sim \prod_{j\neq i} p_{j}(s_j|\mathrm{pa}(s_j))}
    \left[\mathlarger{\int}\!\! p(s_i|\mbox{pa}(s_{i})) 
    \log \frac{ \left\| \frac{\partial \fb}{\partial s_i}(s_i,\sb_{-i})\right\|^{-1} \!\!p_i(\sb_i|\mbox{pa}(\sb_i))}{\left\|  \frac{\partial \widehat{\fb}}{\partial s_i} 
    \!\!\left(\sigma_i^{-1}(s_i),\sib^{-1}(\sb)_{-i}\right)\right\|^{-1} 
    \!\!p_i({\si}^{-1}_i(s_i)|\mbox{pa}(s_i))} ds_i\right]\,.
\end{multline*}
As illustrated in Fig.~\ref{fig:KLmapping}, for a fixed $\sb_{-i}$, consider the straight line  $\mathcal{L}_{\sb_{-i}}=\{(s_i,\sb_{-i}):s_i\in \RR \}$ in source space (parallel to the $s_i$ coordinate axis). This line is mapped in observation space to the smooth curve $\fb[\mathcal{L}_{\sb_{-i}}]$, by $\fb$ in a smooth invertible way. Similarly, $\widehat{\fb}=\fb\circ \sib$ maps $\mathcal{L}_{\sib^{-1}(\sb_{-i})}$ to the same image curve, since $\widehat{\fb}[\mathcal{L}_{\sib^{-1}(\sb_{-i})}] = \fb\circ\sib[\mathcal{L}_{\sib^{-1}(\sb_{-i})}] = \fb[\mathcal{L}_{\sb_{-i}}]$. 

\begin{figure}
    \centering
    \includegraphics[width=.99\textwidth]{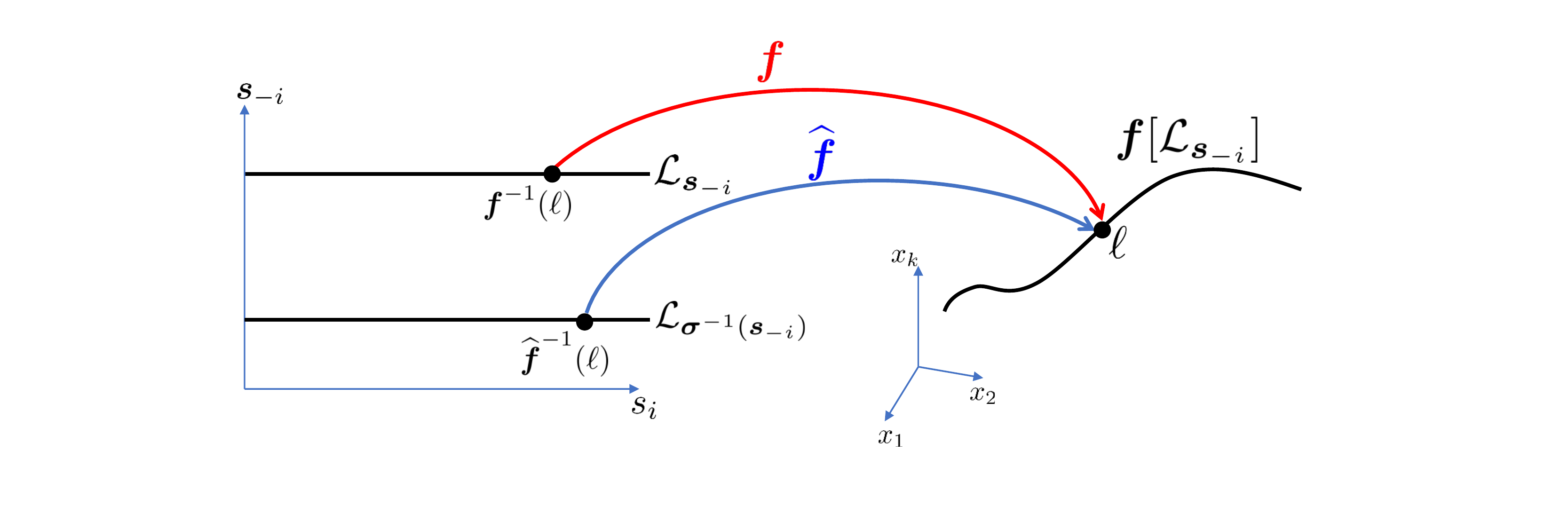}  
    \caption{Illustration of the mapping between lines in source space to a curve in observation space.
    $\Lcal_{\sb_{-i}}$ is the line obtained by varying $s_i$ while keeping the value of all other sources fixed to $\sb_{-i}$.
    $\mathcal{L}_{\sib^{-1}(\sb_{-i})}$ is then defined by applying the transformations in $[\sib^{-1}]_{-i}$ to $\Lcal_{\sb_{-i}}$. Both lines are mapped to the same image line $\fb[\mathcal{L}_{\sb_{-i}}]$.
    \label{fig:KLmapping}}
    \label{fig:my_label}
\end{figure}

By using the change of variable formula to represent the integral on $\fb[\mathcal{L}_{\sb_{-i}}]$ indexed by the curvilinear coordinate $\ell$, we get the expression of the pushfoward distribution $\fb_* p_i(\,.\,|\mbox{pa}(s_{i}))$ on the curve $\fb[\mathcal{L}_{\sb_{-i}}]$
\[
\Big[\fb_* p_i\left(\,.\,|\mbox{pa}(s_{i})\right)\Big](\ell)=\norm{\frac{\partial \fb}{\partial s_i}\left(\fb^{-1}(\ell),\sb_{-i}\right)}^{-1}
p_i\left(\fb^{-1}(\ell)|\mbox{pa}(s_i)\right) \,.
\]
where, to simplify notation, $\fb^{-1}(\ell)$ denotes in this context the coordinate $s_i$ on $\mathcal{L}_{\sb_{-i}}$ in bijection with the curvilinear coordinate $\ell$ on $\fb[\mathcal{L}_{\sb_{-i}}]$.

Similarly, we get the expression of the pushfoward distribution $\widehat{\fb}_* p_i(\,.\,|\sib^{-1}(\mbox{pa}(s_{i})))$ from $\mathcal{L}_{\sib^{-1}(\sb_{-i})}$ to the curve $\fb[\mathcal{L}_{\sb_{-i}}]$ (using again the fact that parent variables are not intervened on, and thus left unchanged by $\sib$)
\[
\left[\widehat{\fb}_* p_i(\,.\,|\sib^{-1}(\mbox{pa}(s_{i})))\right](\ell)=\norm{\frac{\partial \widehat{\fb}}{\partial s_i}\left(\widehat{\fb}^{-1}(\ell),\sib^{-1}(\sb)_{-i}\right)}^{-1} p_i\left(\widehat{\fb}^{-1}(\ell)|\mbox{pa}(s_i)\right) \,.
\]

These terms appear when rewriting the $i$-th term (for a leaf variable) in~\eqref{eq:KLdecomp1} as a curvilinear integral:
\begin{multline*}
\mathlarger{\int}  p_{\sb}(\sb) \log \frac{ \left\| \frac{\partial \fb}{\partial s_i}( \sb)\right\|^{-1} \!\!p_i(\sb_i|\mbox{pa}(\sb_i))}{\left\|  \frac{\partial \widehat{\fb}}{\partial s_i} (\sib^{-1}(\sb))\right\|^{-1} 
\!\!p_i\left({\sib}^{-1}(\sb)_i|\mbox{pa}({\sib}^{-1}(\sb)_i)\right)} d\sb \\
= \mathbb{E}_{\sb_{-i}\sim \prod_{j\neq i} p_{j}(s_j|\mathrm{pa}(s_j))}
\left[
\mathlarger{\int} \left\|\frac{\partial \fb}{\partial s_i}(\fb^{-1}(\ell),\sb_{-i})\right\|^{-1} \!\!p_i(\fb^{-1}(\ell)\mid\mbox{pa}(s_{i})) \right.\\
\left. \log \frac{ \left\| \frac{\partial \fb}{\partial s_i}(\fb^{-1}(\ell),\sb_{-i})\right\|^{-1} \!\!p_i\left(\fb^{-1}(\ell)\mid\mbox{pa}(s_{i})\right)}{\left\|  \frac{\partial \widehat{\fb}}{\partial s_i} (\widehat{\fb}^{-1}(\ell),\sib^{-1}(\sb)_{-i})\right\|^{-1} \!\!p_i\left(\widehat{\fb}^{-1}(\ell)\mid \mbox{pa}(s_{i})\right)} d\ell 
\right]\,.
\end{multline*}
The inner integral term can thus be interpreted as the KL divergence between two pushforward measures defined on $\fb_*[\mathcal{L}_{\sb_{-i}}]$ by $\fb$ and $\widehat{\fb}$, that we can denote by
\[
\KL \left(\fb_* p_i\left(\,.\,|\mbox{pa}(s_{i})\right)\,\|\,\widehat{\fb}_* p\left(\,.\,|\sib^{-1}(\mbox{pa}(\sb_{i}))\right) \right).
\]
To conclude, this implies that the causal effect of the soft intervention $\fb\rightarrow \widehat{\fb}$ can be decomposed as the following sum of $m$ positive terms associated to interventions on each leaf variable, plus an additional term for the remaining non-leaf variables, which further simplifies (in comparison to \eqref{eq:KLdecomp1}) due to the assumption that those variables are unintervened. 
\begin{multline}\label{eq:KLdecompIMA}
\KL(p_{\xb}\,\|\,p_{\widehat{\xb}}) 
= %
\sum_{i=1}^m \mathbb{E}_{\sb_{-i}\sim \prod_{j\neq i} p_{j}(s_j|\mathrm{pa}(s_j))}
\left[\KL\left(\fb_* p\left(\,.\,|\mbox{pa}(s_{i}\right)\,\|\,\widehat{\fb}_* p\left(\,.\,|\sib^{-1}(\mbox{pa}(s_{i}))\right)\right)\right]\\
+\sum_{i>m} \mathlarger{\int}  p_{\sb}(\sb) \log \frac{ \left\| \frac{\partial \fb}{\partial s_i}( \sb)\right\|^{-1} }{\left\|  \frac{\partial \widehat{\fb}}{\partial s_i} (\sib^{-1}(\sb))\right\|^{-1} }
d\sb \,.
\end{multline}
This expression suggests that the KL-divergences appearing in the first $m$ terms each reflect the causal effect of an intervention on the mechanism at the level of one single source coordinate $i$, turning $\frac{\partial {\fb}}{\partial s_i}$ into~$\frac{\partial \widehat{\fb}}{\partial s_i}$. When the sources are jointly independent, we have $m=n$ and the right hand side of \eqref{eq:KLdecompIMA} contains only positive terms. An interesting direction for future work would be to analyse the remaining term in the case of non unconditionally independent sources. 

In contrast to the decomposition~\eqref{eq:addIrregIGCI} in the context of IGCI, the IMA decomposition~\eqref{eq:KLdecompIMA} involves $m$ (expectations of) KL-divergence terms instead of two, each related to the intervention on the part of the mechanism $\frac{\partial \fb}{\partial s_i}$ that reflects the influence of a single source.

\subsection{Independence of cause and mechanism and IMA}
\label{app:violations_icm_ima}

We now discuss an example in which a formalisation of the principle of independence of cause and mechanism~\cite{janzing2010telling} is violated, and one in which the IMA principle is violated.

\subsubsection{Violations of independence of cause and mechanism}

In the context of the Trace method~\cite{janzing2010telling}, used in causal discovery, a technical example of fine-tuning can be constructed by taking a vector of i.i.d. random variables with arbitrary (not diagonal) covariance matrix $\Sigma$ as the cause, and by constructing the mechanism as a whitening matrix, turning the cause variables into uncorrelated (effect) variables. By doing so, the singular values and singular vectors of the matrix (the mechanism) are fine-tuned to the input covariance matrix (a property of the cause distribution)%
, and such fine-tuning can be quantified via the Trace method (see~\cite{janzing2010telling}, Section 1). %

\subsubsection{Violations of the IMA principle}

\paragraph{Technical example.}
As mentioned in~\cref{sec:unsuitability_of_existing_ICM_measures}, %
an example of a mixing function $\fb$ which is non-generic according to the IMA principle is an autoregressive function, for example an autoregressive normalising flow~\cite{papamakarios2019normalizing}, where the $k$-th component of the observations only depends on the $k$-th sources: intuitively, this would correspond to the unlikely cocktail-party setting where the $k$-th microphone only picks up the voices of the first  speakers. More precisely, as we show in~\cref{lemma:trijaccima}, this leads to positive $C_\IMA$ value for such mixing.

\paragraph{Pictorial example: Violations of the IMA principle in a cocktail party.} %
A cocktail party (\cref{fig:intuition}, left) may violate our IMA principle when the locations of several speakers and the room acoustics have been fine tuned to one another. This is for example the case in concert halls where the acoustics of the room have been fine-tuned to the position and configuration of multiple locations on the stage, where the sources (i.e., the voices of the actors or singers) are emitted---in order to make the listening experience as homogeneous as possible across the spectators (that is, the influence of each of the sources on the different listeners should not differ too much). This would lead to an increase in collinearity between the columns of the mixing's Jacobian, thus violating the IMA principle. 

Additionally, we recall that the ICM principle is often informally introduced by referencing the fine-tuning and non-generic viewpoints giving rise to certain visual illusions, such as the Beuchet chair (see~\cite{peters2017elements}, Section 2); in a similar vein, we can imagine that violations of the IMA principle in the cocktail party setting may be related to illusions in binaural hearing such as for example the Franssen effect, where the listener is tricked into incorrectly localizing a sound~\cite{schroeder1993listening}.%

%% file: appendix/C_proofs.tex
We now provide the proofs of all our theoretical results from the main paper.

\subsection{Proof of\texorpdfstring{~\Cref{prop:local_IMA_contrast_properties}}{}}

Before giving the proof, 
it is useful to rewrite the local IMA constrast~\eqref{eq:adm_single_point} as follows:
\begin{align}
    c_\IMA(\fb,\sb) =& \sum_{i = 1}^n \log \norm{\frac{\partial \fb}{\partial s_i}(\sb)} - \log \left|\Jb_\fb(\sb)\right| \nonumber\\
    =& \frac{1}{2} \left( \log \left|\diag \left(\Jb^\top_\fb(\sb) \Jb_\fb(\sb)\right) \right| - \log \left|\Jb^\top_\fb(\sb) \Jb_\fb(\sb) \right| \right) \nonumber \\
     =& \frac{1}{2} D_{KL}^{\text{left}}\left(\Jb^\top_\fb(\sb) \Jb_\fb(\sb)\right)\,,
    \label{eq:dkl_diag}
\end{align}
where the quantity in~\eqref{eq:dkl_diag} is called the left KL measure of diagonality of the matrix $\Jb^\top_\fb(\sb) \Jb_\fb(\sb)$~\cite{alyani2017diagonality} (see~\cref{remark:left_KL}):
\begin{align*}
    D_{KL}^{\text{left}}(\Ab) &= - \log | (\diag(\Ab))^{-\frac{1}{2}} \Ab (\diag(\Ab))^{-\frac{1}{2}} |\\
                              &= \log |\diag(\Ab)| - \log |\Ab|\,.
\end{align*}
From~\eqref{eq:dkl_diag}, it can be seen that $c_\IMA(\fb,\sb)$ is a function of $\Jb_\fb(\sb)$ only through $\Jb^\top_\fb(\sb) \Jb_\fb(\sb)$. 

\impropties*
\begin{proof} 
For ease of exposition, we denote the value of the Jacobian of $\fb$ evaluated at the point $\sb$ by $\Jb_\fb(\sb) = \Wb$.
The two properties can then be proved as follows:
\begin{enumerate}[(i)] 
        \item This is a consequence of Hadamard's inequality, applied to the expression on the RHS of~\eqref{eq:adm_single_point}, which states that, for a matrix $\Wb$ with columns $\wb_i$, $\sum_{i = 1}^n \log \norm{\wb_i} \geq \log |\Wb|$; equality in Hadamard's inequality is achieved iff. the vectors $\wb_i$ are orthogonal.
        \item We split the proof in three parts.
        \begin{enumerate}[a.]
            \item \textit{Invariance to left multiplication by an orthogonal matrix:}\\
            Let $\Wbt = \Ob \Wb$, with $\Ob$ an orthogonal matrix, i.e., $\Ob\Ob^\top=\Ib$.
            Then the property follows from  writing $c_\IMA(\fb,\sb)$ as in~\eqref{eq:dkl_diag}:
            \begin{align*}
                \frac{1}{2}\leftkl(\Wbt^\top \Wbt) =\frac{1}{2} \leftkl(\Wb^\top \Ob^\top  \Ob \Wb) 
                =\frac{1}{2} \leftkl(\Wb^\top \Ib \Wb) 
                = \frac{1}{2}\leftkl(\Wb^\top \Wb) %
            \end{align*}
            \item
            \textit{Invariance to right multiplication by a permutation matrix:}\\
            Let $\Wbt = \Wb \Pb$, with $\Pb$ a permutation matrix. Then $\Wbt$ is just $\Wb$ with permuted columns. Clearly, the sum of the log-column-norms does not change by changing the order of the summands.
            Further, $\log|\tilde{\Wb}| = \log|\Wb| + \log|\Pb| =\log|\Wb|$, because the absolute value of the determinant of a permutation matrix is one. 
            \item \textit{Invariance to right multiplication by a diagonal matrix}:\\ Let $\Wbt = \Wb \Db$, with $\Db$ a diagonal matrix. Consider the two terms on the RHS of~\eqref{eq:adm_single_point}.
            For the first term, we know that the columns of~$\tilde{\Wb}$ are scaled versions of the columns of $\Wb$, that is $\tilde{\wb}_i = d_i \wb_i$, where $d_i$ denotes the $i^\text{th}$ diagonal element of $\Db$.
            Then $\norm{\tilde{\wb}_i} = |d_i|\norm{\wb_i}$. 
            For the second term, we use the decomposition of the determinant: \[\log|\tilde{\Wb}| = \log|\Wb| + \log|\Db| = \log|\Wb| + \sum_{i=1}^n\log|d_i|.\]
            Taken together, we obtain
            \begin{align*}
            \sum_{i = 1}^n \log \norm{\wbt_i} - \log |\Wbt| 
            &= \sum_{i = 1}^n \log \left(|d_i| \norm{\wb_i}\right) -\left( \log |\Wb| + \sum_{i=1}^n\log|d_i|\right)\\
            &= \sum_{i = 1}^n \log \norm{\wb_i} + \sum_{i = 1}^n \log |d_i|  - \log |\Wb| - \sum_{i=1}^n\log|d_i|\\
            &= \sum_{i = 1}^n \log \norm{\wb_i}   - \log |\Wb|%
            \end{align*}%
        \end{enumerate}%
    \end{enumerate}%
\end{proof}

\subsection{Proof of\texorpdfstring{~\Cref{prop:global_IMA_contrast_properties}}{}}
\admproperties*
\begin{proof}
 The properties can be proved as follows:
    \begin{enumerate}[(i)]
        \item From property \textit{(i)} of~\cref{prop:local_IMA_contrast_properties}, we know that $c_\IMA(\fb, \sbb) \geq 0$. Hence, $C_\IMA(\fb, p(\sbb)) \geq 0$ follows as a direct consequence of integrating the non-negative quantity $c_\IMA(\fb, \sbb)$.
        
        Equality is attained iff.\ $c_\IMA(\fb, \sbb)=0$ almost surely w.r.t.\ $p_\sb$, which according to property \textit{(i)} of~\cref{prop:local_IMA_contrast_properties} occurs 
        iff.\  the columns of $\Jb_{\fb}(\sb)$ are orthogonal almost surely w.r.t.\ $p_\sb$.
        
        It remains to show that this is the case iff.\ $\Jb_{\fb}(\sb)$ can be written as $\Ob(\sb) \Db(\sb)$, with $\Ob(\sb)$ and $\Db(\sb)$  orthogonal and diagonal matrices, respectively.
        (To avoid confusion, note that \textit{orthogonal columns} need not have unit norm, whereas an \textit{orthogonal matrix} $\Ob$ satisfies $\Ob\Ob^\top=\Ib$.)
        
        The \textit{if} is clear since right multiplication by a diagonal matrix merely re-scales the columns, and hence does not affect their orthogonality. 
        
        For the \textit{only if}, let $\Jb_\fb(\sb)$ be any matrix with orthogonal columns $\jb_i(\sb)$, $\jb_i(\sb)^\top \jb_j(\sb)=0, \forall i\neq j$, and denote the column norms by $d_i(\sb)=||\jb_i(\sb)||$. Further denote the normalised columns of $\Jb_\fb(\sb)$ by $\ob_i(\sb)=\jb_i(\sb) / d_i(\sb)$ and let $\Ob(\sb)$ and $\Db(\sb)$ be the orthogonal and diagonal matrices with columns $\ob_i(\sb)$ and diagonal elements $d_i(\sb)$, respectively. Then $\Jb_\fb(\sb)=\Ob(\sb)\Db(\sb)$.
        \item Let $\fbt=\fb\circ\hb^{-1}\circ \Pb^{-1}$ and $\sbt = \Pb\hb(\sb)$,  where $\Pb\in\RR^{n\times n}$ is a permutation matrix and $\hb(\sb)=(h_1(s_1), ..., h_n(s_n))$ is an invertible element-wise function.
        Then
        \begin{equation}
        \label{eq:equality_of_global_contrasts}
            C_\IMA(\fbt,p_\sbt)
            =\int c_\IMA(\fbt,\sbt)p_\sbt(\sbt) d\sbt
            =\int c_\IMA(\fbt,\sbt)p_\sb(\sb) d\sb
        \end{equation}
        where, for the second equality, we have used the fact that 
        \[p_\sbt(\sbt)d\sbt = p_\sb(\sb)d\sb\,.\]
        since $\Pb \circ \hb$ is an invertible tranformation (see, e.g., ~\cite{rezendeshort}).
        It thus suffices to show that
        \begin{equation}
        \label{eq:equality_of_local_contrasts}            c_\IMA(\fbt,\sbt)=c_\IMA(\fb,\sb).
        \end{equation}
        at any point $\sbt=\Pb\hb(\sb)$.
        To show this, we write
        \begin{align}
            \Jb_\fbt(\sbt)
            &=\Jb_{\fb\circ\hb^{-1}\circ\Pb^{-1}}(\Pb\hb(\sb)) \nonumber \\
            &=\Jb_{\fb\circ\hb^{-1}}\left(\Pb^{-1}\Pb\hb(\sb)\right)\, \Jb_{\Pb^{-1}}\left(\Pb\hb(\sb)\right) \nonumber\\
            &=\Jb_{\fb\circ\hb^{-1}}(\hb(\sb))\, \Jb_{\Pb^{-1}}(\Pb\hb(\sb)) \nonumber\\
            &=\Jb_{\fb}(\hb^{-1}\circ\hb(\sb))\, \Jb_{\hb^{-1}}(\hb(\sb)) \, \Jb_{\Pb^{-1}}(\Pb\hb(\sb)) \nonumber\\
            &=\Jb_{\fb}(\sb) \, \Db(\sb) \Pb^{-1}
            \label{eq:Jacobian_relation}
        \end{align}
        where we have repeatedly used the chain rule for Jacobians, as well as that $\Pb^{-1}\Pb=\Ib$; that permutation is a linear operation, so $\Jb_\Pb(\sb)=\Pb$ for any $\sb$; and that $\hb$ (and thus $\hb^{-1}$) is an element-wise transformation, so the Jacobian $\Jb_{\hb^{-1}}$ is a diagonal matrix $\Db(\sb)$.
        
        The equality in~\eqref{eq:equality_of_local_contrasts} then follows from~\eqref{eq:Jacobian_relation} by applying property \textit{(ii)} of~\cref{prop:local_IMA_contrast_properties}, according to which $c_\IMA$ is invariant to right multiplication of the Jacobian $\Jb_\fb(\sb)$ by diagonal and permutation matrices.
        
        Substituting~\eqref{eq:equality_of_local_contrasts} into the RHS of~\eqref{eq:equality_of_global_contrasts}, we finally obtain
        \begin{equation*}
            C_\IMA(\fbt,p_\sbt)=C_\IMA(\fb,p_\sb).
        \end{equation*}%
        \end{enumerate}%
\end{proof}

\subsection{Remark on a similar condition to IMA, expressed in terms of the rows of the Jacobian}

We remark that the condition imposed by the IMA~\cref{principle:IMA} needs to be expressed in terms of the columns of the Jacobian, and would not lead to a criterion with desirable properties for BSS if it were instead expressed in terms of its rows (which correspond to gradients of the $f_i(\sb)$). One way to justify this is that, for the same condition expressed on the rows of the Jacobian, that is
\[
\sum_{i = 1}^n \log \norm{\nabla f_i(\sb)} - \log \left|\Jb_\fb(\sb)\right| =0\,, 
\]
property \textit{(ii)} of~\Cref{prop:local_IMA_contrast_properties} would not hold (because invariance would hold w.r.t.\ right, not left, multiplication with a diagonal matrix). As a consequence, the resulting global contrast would not be blind to reparametrisation of the source variables by permutation and element-wise invertible transformations, thereby not being a good contrast in the context of blind source separation. 

\subsection{Proof of\texorpdfstring{~\Cref{thm:adm_darmois}}{}}
Before proving the main theorem, we first introduce some additional details on the Jacobian of the Darmois construction~\cite{hyvarinen1999nonlinear} which will be important for the proof.

\input{appendix/jacobian_darmois}

It is additionally useful to introduce the following lemmas.

\begin{restatable}{lemma}{trijaccima}
\label{lemma:trijaccima}
A function $\fb$ with triangular Jacobian has $C_\IMA(\fb, p_\sb) = 0$ iff. its Jacobian is diagonal almost everywhere. Otherwise,  $C_\IMA(\fb, p_\sb) > 0$.%
\end{restatable}%
\begin{proof}
    Let $\fb$ have lower triangular Jacobian at $\sb$, and denote $\Jb_\fb(\sb) = \Wb$. Then we have 
    \[
    c_\IMA(\fb,\sb) = \sum_{i = 1}^n \log \left( \sqrt{\sum_{j=i}^n w^2_{ji} } \right) - \sum_{i = 1}^n \log \left| w_{ii} \right|\,,
    \]
    where $w_{ji} = [\Wb]_{ji} $. Since the logarithm is a strictly monotonically increasing function and since \[\sqrt{\sum_{j=1}^n w^2_{ji} } \geq | w_{ii} |\,,\] with equality iff.\ $w_{ji}=0, \forall j \neq i$ (i.e., iff.\ $\Wb$ is a diagonal matrix), we must have $c_\IMA(\fb,\sb)=0$ iff.\ $\Wb$ is diagonal. 
    
    $C_\IMA(\fb,p_\sb)$ is therefore equal to zero iff. $\fb$ has diagonal Jacobian almost everywhere, and it is strictly larger than zero otherwise. 
\end{proof}

\begin{restatable}{lemma}{elemwise}
\label{lemma:elemwise}
A smooth function $\fb:\RR^n \rightarrow \RR^n$ whose Jacobian is diagonal everywhere is an element-wise function, $\fb(\sb) = (f_1(s_1), ..., f_n(s_n))$.
\end{restatable}%
\begin{proof}
Let $\fb$ be a smooth function with diagonal Jacobian everywhere.

Consider the function $f_i(\sb)$ for any $i\in\{1, ..., n\}$. Suppose \textit{for a contradiction} that $f_i$ depends on $s_j$ for some $j\neq i$. Then there must be at least one point $\sb^*$ such that $\nicefrac{\partial f_i}{\partial s_j} (\sb^*)\neq 0$.
However, this contradicts the assumption that $\Jb_\fb$ is diagonal everywhere (since $\nicefrac{\partial f_i}{\partial s_j}$ is an off-diagonal element for $i\neq j$).
Hence, $f_i$ can only depend on $s_i$ for all $i$, i.e., $\fb$ is an element wise function.
\end{proof}%

We can now restate and prove~\Cref{thm:adm_darmois}.

\admdarmois*

\begin{proof}
First, the Jacobian $\Jb_{\gb^\text{D}}(\xb)$ of the Darmois construction $\gb^\text{D}$ is lower triangular $\forall\xb$, see~\eqref{eq:full_jacobian_darmois}.

Because CDFs are monotonic functions (strictly monotonically increasing given our assumptions on $\fb$ and $p_\sb$), $\gb^\text{D}$ is invertible.

We can thus apply the inverse function theorem (with $\fb^\text{D}=(\gb^\text{D})^{-1}$) to write
\[
\Jb_{\fb^\text{D}}(\yb)=\left(\Jb_{\gb^\text{D}}(\xb)\right)^{-1}
\]
Since the inverse of a lower triangular matrix is lower triangular, we conclude that $\Jb_{\fb^\text{D}}(\yb)$ is lower triangular for all $\yb=\gb^\text{D}(\xb)$. 

Now, according to~\Cref{lemma:trijaccima}, we have
$C_\IMA(\fb^\text{D}, p_{\ub}) > 0$, unless $\Jb_{\fb^\text{D}}$ is diagonal almost everywhere.

Suppose \textit{for a contradiction} that $\Jb_{\fb^\text{D}}$ is diagonal almost everywhere.

Since $\fb$ and $p_\sb$ are smooth by assumption, so is the push-forward $p_\xb=\fb_*p_\sb$, and thus also $\gb^\text{D}$ (CDF of a smooth density) and its inverse $\fb^\text{D}$. 
Hence, the partial derivatives $\nicefrac{\partial f^\text{D}_i}{\partial y_j}$, i.e., the elements of $\Jb_{\fb^\text{D}}$ are continuous.

Consider an off-diagonal element $\nicefrac{\partial f^\text{D}_i}{\partial y_j}$ for $i\neq j$. Since these are zero almost everywhere, and because continuous functions which are zero almost everywhere must be zero everywhere, we conclude that $\nicefrac{\partial f^\text{D}_i}{\partial y_j}=0$ everywhere for $i\neq j$, i.e., the Jacobian $\Jb_{\fb^\text{D}}$ is \textit{diagonal everywhere}. 

Hence, we conclude from~\cref{lemma:elemwise} that $\fb^\text{D}$ must be an element-wise function, $\fb^\text{D}(\yb)=(f^\text{D}_1(y_1), ..., f^\text{D}_1(y_n))$.

Since $\yb$ has independent components by construction, it follows that $x_i=f^\text{D}_i(y_i)$ and $x_j=f^\text{D}_j(y_j)$ are independent for any $i\neq j$. 

However, this constitutes a contradiction to the assumption that $x_i \not\independent x_j$ for some $x_j$. 

We conclude that $\Jb_{\fb^\text{D}}$ cannot be diagonal almost everywhere, and hence, by~\cref{lemma:trijaccima}, we must have $C_\IMA(\fb^\text{D},p_\ub)>0$.
\end{proof}
%
%
%

\begin{comment}
We know that 
\begin{align}
c_{21} &= \frac{\partial}{\partial x_1} \int_{-\infty}^{x_2} p(t|x_1) dt \\
&=  \int_{-\infty}^{x_2} \frac{\partial}{\partial x_1} p(t|x_1) dt
\end{align}
\julius{Note that interchanging the order of differentiation and integration requires certain assumptions (see Leibniz' rule); specifically, that $p(t|x_1)$ is once differentiable w.r.t. $x_1$ and that $p(t|x_1)$ and its partial derivative are integrable w.r.t. $t$.
%
I think this is where the assumption of $\Ab$ being full rank / orthogonal comes in.
}

And $c_{21}$ is zero for all $x_2$ \julius{maybe we should use the notation $c_{ji}(\xb)$ to indicate that it may depend on all $\xb$? also I think we need $\forall \xb$ and not just $\forall x_2$ here.}
iff
\julius{the ``only if'' is not obvious to me: while the density is a positive quantity, its derivative need not be, so maybe we should treat the two cases separately...}
\begin{equation}
    \frac{\partial}{\partial x_1} p(t|x_1) = 0
    \label{eq:partial}
\end{equation}
Which implies that $X_1 \independent X_2$. \luigi{Does this work? Should we say ``almost surely'' (it could be nonzero in a set of zero mass), do we need to assume smooth densities? Does eq.~\ref{eq:partial} imply independence?...}
\end{comment}

%
%
%
%
%

\subsection{Proof of\texorpdfstring{~\Cref{cor:IMA_identifiability_of_conformal_maps}}{}}
\confmapsadm*
\begin{proof}
The proof follows from property \textit{(i)} of~\Cref{prop:global_IMA_contrast_properties}: by definition, the Jacobian of conformal maps at any point $\sb$ can be written as $\Ob(\sb) \lambda(\sb)$, with $\lambda: \RR^n \rightarrow \RR$, which is a special case of $\Ob(\sb) \Db(\sb)$, with $\Db(\sb) = \lambda(\sb) \Ib $.
\end{proof}

\subsection{Proof of\texorpdfstring{~\Cref{cor:IMA_identifiability_of_linear_ICA}}{}}
\admidentlinear*
\begin{proof}
Since, by assumption, the mixing matrix is non-trivial (i.e., not the product of a diagonal and permutation matrix), and  at most one of the $s_i$ is Gaussian, according to~\Cref{thm:identifiability_of_linear_ICA} there must be at least one pair $x_i, x_j$, with $i \neq j$, such that $x_i \nindep x_j$.

We can then use the same argument as in the proof of~\cref{thm:adm_darmois}
to show that the Darmois construction has nonzero $C_\IMA$, whereas the linear orthogonal transformation $\Ab$ has orthogonal Jacobian, and thus $C_\IMA=0$ by property \textit{(i)} of~\cref{prop:global_IMA_contrast_properties}.
\end{proof}

\subsection{Proof of\texorpdfstring{~\Cref{thm:IMA_identifiability_measure_preserving_automorphism}}{}}
\thmMPA*
\begin{proof}
Recall the definition 
\[
 \ab^{\Rb}(p_\sb) =\Fb_\sb^{-1} \circ \bm\Phi \circ \Rb \circ \bm\Phi^{-1} \circ \Fb_\sb.
\]
For notational convenience, we denote $\sib = \bm\Phi^{-1} \circ \Fb_\sb$
and write
\[
 \ab^{\Rb}(p_\sb) =\sib^{-1} \circ \Rb \circ \sib.
\]
Note that, since both $\Fb_\sb$ and $\bm\Phi$ are element-wise transformations, so is $\sib$. 

First, by using property \textit{(ii)} of~\cref{prop:global_IMA_contrast_properties} (invariance of $C_\IMA$ to element-wise transformation), we obtain
\[
    C_\IMA(\fb\circ  \ab^{\Rb}(p_\sb), p_{\sb})
    =C_\IMA(\fb\circ \sib^{-1} \circ \Rb \circ \sib, p_{\sb})
    = C_\IMA(\fb\circ \sib^{-1} \circ \Rb,p_\zb)\,,
\]
with $\zb=\sib(\sb)$ such that $p_\zb$ is an isotropic Gaussian distribution.

Suppose \textit{for a contradiction} that $C_\IMA(\fb\circ \sib^{-1} \circ \Rb,p_\zb)=0$.

According to property \textit{(i)} of~\cref{prop:global_IMA_contrast_properties},
this entails that the matrix
\begin{equation}\label{eq:diagMP}
\Jb_{\fb\circ \sib^{-1} \circ \Rb}(\zb)^\top\Jb_{\fb\circ \sib^{-1} \circ \Rb}(\zb)=
\Rb^\top 
\, \Jb_{\sib^{-1}}(\zb)^\top 
\, \Jb_\fb(\sib^{-1}(\zb))^\top  
\, \Jb_\fb(\sib^{-1}(\zb)) 
\, \Jb_{\sib^{-1}}(\zb)
\,  \Rb\,
\end{equation}
is diagonal 
almost surely w.r.t.\ $p_\zb$.
Moreover, smoothness of $p_{\sb}$ and $\fb$ implies the matrix expression of~\eqref{eq:diagMP} is a continuous function of $\zb$. Thus \eqref{eq:diagMP} actually needs to be diagonal for all $\zb\in \RR^n$, i.e., \textit{everywhere} (c.f., the argument used in the proof of~\cref{thm:adm_darmois}, l.1008--1013).

Since $(\fb,p_\sb)\in\Mcal_\IMA$ by assumption, by property \textit{(i)} of~\cref{prop:global_IMA_contrast_properties}, the inner term on the RHS of~\eqref{eq:diagMP},
\[
\Jb_\fb( \sib^{-1}(\zb))^\top 
\,  
\Jb_\fb( \sib^{-1}(\zb)),
\]
is diagonal. 
Moreover, since $\sib$ is an element-wise transformation, 
$\Jb_{\sib^{-1}}(\zb)^\top$ and $\Jb_{\sib^{-1}}(\zb)$ 
are also diagonal. 
Taken together, this implies that
\begin{equation}
\label{eq:diagonal_inner_term}
\Jb_{\sib^{-1}}(\zb)  
\, 
\Jb_\fb( \sib^{-1}(\zb))^\top 
\,  
\Jb_\fb( \sib^{-1}(\zb))
\, 
\Jb_{\sib^{-1}}(\zb)    
\end{equation}
is diagonal (i.e., ~\eqref{eq:diagMP} is of the form $\Rb^\top \Db(\zb)\Rb$ for some diagonal matrix $\Db(\zb)$).

Without loss of generality, we assume the first component $s_1$ of $\sb$ is non-Gaussian and satisfies the assumptions stated relative to $\Rb$ (axis not invariant nor sent to another canonical axis).

Now, since both the Gaussian CDF $\bm \Phi$ and the CDF $\Fb_\sb$ are smooth (the latter by the assumption that of $p_\sb$ is a smooth density), $\sib$ is a smooth function, and thus has continuous partial derivatives.

By continuity of the partial derivative, the first diagonal element $\frac{\partial \sigma_1^{-1}}{\partial z_1}$ of $\Jb_{\sib^{-1}}$ must be strictly monotonic in a neighborhood of some $z_1^0$ (otherwise $\sigma_1$ would be an affine transformation, which would contradict non-Gaussianity of $s_1$).

On the other hand, our assumptions relative to $\Rb$ entail that there are at least two non-vanishing coefficients in the first row of $\Rb$ (i.e., first column of $\Rb^\top$).\footnote{In short, if this were not the case, this column would have a single non-vanishing coefficient, which would need to be one due to the unit norm of the rows of this orthogonal matrix. %
Such structure of the matrix $\Rb$ would entail that  the associated canonical basis vector $\eb_1$ is transformed by $\Rb^{-1}=\Rb^\top$ into a canonical basis vector $\eb_j$ which contradicts the assumptions.} 
Let us call $i\neq j$ such pair of coordinates, i.e., $r_{1j}\neq 0$ and $r_{1i}\neq 0$. 

Now consider the off-diagonal term $(i,j)$ of~\eqref{eq:diagMP}, which we assumed (for a contradiction) must be zero almost surely w.r.t.\ $p_\zb$.
Since the term in~\eqref{eq:diagonal_inner_term} is diagonal, this off-diagonal term is given by:
\[
\sum_{k=1}^n 
\left(\frac{d\sigma_k^{-1}}{dz_k}(z_k)\right)^2 \norm{\frac{\partial \fb}{ds_k}\circ \sib^{-1}(\zb)}^2 r_{ki} r_{kj}
=
\sum_{k=1}^n  \left(\frac{d\sigma_k^{-1}}{dz_k}(z_k)\right)^2 \lambda(\sib^{-1}(\zb))^2 r_{ki} r_{kj}=0\,.
\]
where for the first equality we have used the fact that $\fb$ is a conformal map with conformal factor~$\lambda(\sb)$ (by assumption), and 
where the second equality must hold
almost surely w.r.t.\ $p_\zb$.

Since $\fb$ is invertible, it has non vanishing Jacobian determinant. Hence, the conformal factor $\lambda$ must be a strictly positive function, so
\[
\lambda(\sib^{-1}(\zb))^2>0, \, \forall \zb.
\]
Thus, for almost all $\zb$, we must have:
\begin{equation}
\label{eq:contradiction_expression}
\sum_{k=1}^n  \left(\frac{d\sigma_k^{-1}}{dz_k}(z_k)\right)^2  r_{ki} r_{kj}=0\,.
\end{equation}
Now consider the first term $\left(\frac{d\sigma_1^{-1}}{dz_1}(z_1)\right)^2  r_{1i} r_{1j}$  in the sum.  

Recall that $r_{1i}r_{1j}\neq 0$,
and that $\frac{d\sigma_1^{-1}}{dz_1}(z_1)$
is strictly monotonic  on a neighborhood of $z_1^0$.

As  a consequence, $\left(\frac{d\sigma_1^{-1}}{dz_1}(z_1)\right)^2  r_{1i} r_{1j}$ is also strictly monotonic with respect to $z_1$  on a neighborhood of $z_1^0$ (where the  other variables $(z_{2},...,z_{n})$ are left  constant), while the other
terms in the sum in~\eqref{eq:contradiction_expression} are left constant because $\sib$ is an element-wise transformation.

This leads to a contradiction as~\eqref{eq:contradiction_expression} (which should be satisfied for all $\zb$) cannot stay constantly zero as $z_1$ varies within the neighbourhood of $z_1^0$. 

Hence our assumption that $C_\IMA(\fb\circ  \ab^{\Rb}(p_\sb), p_{\sb})=0$ cannot hold.

We conclude that $C_\IMA(\fb\circ  \ab^{\Rb}(p_\sb), p_{\sb})>0$.
\end{proof}

%% file: appendix/jacobian_darmois.tex
\paragraph{Jacobian of the Darmois construction for $n=2$.}
\label{app:Jacobian_Darmois}
Consider the Darmois construction for $n=2$, 
\begin{align*}
     y_1&=g^{\text{D}}_1(x_1):=F_{X_1}(x_1)=\PP_{X_1}(X_1\leq x_1)\\
     y_2&=g^{\text{D}}_2(y_1, x_2):=F_{X_2|Y_1=y_1}(x_2)=\PP_{X_2|Y_1=y_1}(X_2\leq x_2|Y_1=y_1)
\end{align*}

Its Jacobian takes the form
\begin{equation}
\label{eq:general_Jacobian_Darmois_2d}
    \Jb_{\gb^\text{D}} (\xb) =
    \begin{pmatrix}
    p(x_1) & 0\\
    c_{21}(\xb) & p(x_2|x_1)
    \end{pmatrix}\,,
\end{equation}
where 
\[
c_{21}(\xb)=\frac{\partial }{\partial x_1} \int_{-\infty}^{x_2} p(x'_2|x_1)dx'_2\,.
\]

\begin{comment}
\julius{TODO: simplify the above;}
\paragraph{The non-diagonal term}
We can expand the non-diagonal term as follows
\begin{align}
    c_{21} =& \frac{\partial }{\partial x_1} \int_{-\infty}^{x_2} p(X_2=t|X_1=x_1) dt \\
    =& \frac{\partial }{\partial x_1} \int_{-\infty}^{x_2} \frac{p(X_2=t,X_1=x_1)}{p(X_1=x_1)} dt\\ 
    =&\frac{\partial }{\partial x_1} 
    \frac{1}{p(X_1=x_1)}
    \int_{-\infty}^{x_2} p(X_2=t,X_1=x_1) dt\\
    =& \left(\int_{-\infty}^{x_2} p(X_2=t,X_1=x_1) dt\right) \frac{\partial }{\partial x_1} 
    \frac{1}{p(X_1=x_1)} \\ 
    &+ \frac{1}{p(X_1=x_1)}
    \frac{\partial }{\partial x_1}\int_{-\infty}^{x_2} p(X_2=t,X_1=x_1) dt\\
    =& -\left(\int_{-\infty}^{x_2} p(X_2=t,X_1=x_1) dt\right) \frac{p'(X_1=x_1)}{p^2(X_1=x_1)} 
     \nonumber\\ 
    &+ \frac{1}{p(X_1=x_1)}
    \frac{\partial }{\partial x_1}\int_{-\infty}^{x_2} p(X_2=t,X_1=x_1) dt \nonumber
\end{align}
where in the final equality we used $\frac{\partial}{\partial x}\frac{1}{f(x)} = -\frac{f'(x)}{f^2(x)}\,. $ 
\end{comment}

\paragraph{Jacobian of the Darmois construction: general case.}
In the general case, the Jacobian of the Darmois construction will be 
\begin{equation}
\Jb_{\gb^\text{D}}(\xb) = 
\begin{pmatrix} 
p(x_1) & \cdots & 0 \\
 & \ddots & \vdots \\ 
\mathbf{C}(\xb) &  & p(x_n |x_1, \ldots, x_{n-1} )
\end{pmatrix} 
\label{eq:full_jacobian_darmois}
\end{equation}
where the components $c_{ji}(\xb_{1:j})$ of $\mathbf{C}(\xb)$ for all $i<j$ %
are defined  by
\begin{align*}
c_{ji}(\xb_{1:j})&=\frac{\partial }{\partial x_i} \int_{-\infty}^{x_j} p(x'_j| \xb_{1:j-1} ) dx'_j \,.
\end{align*}

%% file: appendix/D_examples.tex
\begin{example}[Polar to Cartesian coordinates]
Consider the following example of a nonlinear ICA model which represents a change of basis from polar to Cartesian coordinates:
\begin{align*}
    \xb
    =
    \begin{pmatrix}
    x_1 \\ x_2
    \end{pmatrix}
    =
    \fb(\sb)
    =
    \begin{pmatrix}
    f_1(\sb) \\ f_2(\sb)
    \end{pmatrix}
    =
    \begin{pmatrix}
    r \cos(\theta) \\ r\sin(\theta)
    \end{pmatrix}
\end{align*}
with sources
\begin{align*}
    \sb
    =
    \begin{pmatrix}
    s_1 \\ s_2
    \end{pmatrix}
    =
    \begin{pmatrix}
    r \\ \theta
    \end{pmatrix}
    , 
    \quad \quad
    r\sim U[0, R],
    \quad \quad
    \theta~\sim U[0, 2\pi],
\end{align*}

First, we consider the Jacobian of the true mixing $\fb$ which is given by:
\begin{equation*}
    \Jb_{\fb}(\sb)=\Jb_{\fb}(r,\theta)=
    \begin{pmatrix}
    \cos(\theta) & -r\sin(\theta)\\
    \sin(\theta) & r\cos(\theta)
    \end{pmatrix},
\end{equation*}
and its determinant and column norms are given by 
\begin{align*}
    \left|
    \det \Jb_{\fb}(\sb)
    \right|
    &=r
    \left(
    \cos^2(\theta)+\sin^2(\theta)
    \right)
    =r
    \\
    \norm{\frac{\partial \fb}{\partial s_1}(\sb)}
    &=
    \norm{\frac{\partial \fb}{\partial r}(r,\theta)}=\cos^2(\theta)+\sin^2(\theta)=1\\
    \norm{\frac{\partial \fb}{\partial s_2}(\sb)}
    &=
    \norm{\frac{\partial \fb}{\partial \theta}(r,\theta)}
    =r
    \left(
    \cos^2(\theta)+\sin^2(\theta)
    \right)
    =r
\end{align*}
In other words, the columns of $\Jb_{\fb}(\sb)$ are orthogonal for all $\sb$, so that $C_\IMA=0$ for the true solution.

Next, we apply the Darmois construction.

First, we write the joint density of $(x_1,x_2)$ using the change of variable formula:
\begin{equation*}
    p(x_1, x_2) = |\det \Jb_{\fb}(r,\theta)|^{-1} p(r,\theta)
    =
    r^{-1} \frac{1}{2\pi R}
    =
    \frac{1}{\sqrt{x_1^2+x_2^2}} \frac{1}{2\pi R}.
\end{equation*}
Next, we compute the marginal density $p(x_1)$. Note that the observations $\xb$ live on the disk of radius $R$, $\norm{\xb}\leq R$, so $p(x_1,x_2)=0$ whenever $x_1^2+x_2^2>R^2$.
\begin{equation*}
    p(x_1)
    =
    \int_{-\sqrt{R^2-x_1^2}}^{\sqrt{R^2-x_1^2}}
    p(x_1,x_2) dx_2
    =
    \frac{1}{2\pi R}
    \int_{-\sqrt{R^2-x_1^2}}^{\sqrt{R^2-x_1^2}}
    \frac{dx_2}{\sqrt{x_1^2+x_2^2}}
    =
    \frac{1}{2\pi R}
    \int_{-\sqrt{R^2-x_1^2}}^{\sqrt{R^2-x_1^2}}
    \frac{dx_2}{x_1 \sqrt{1+(\frac{x_2}{x_1})^2}}
\end{equation*}
Applying the change of variable $t=\frac{x_2}{x_1}$ with $dt=\frac{dx_2}{x_1}$, and using the integral $\int(1+t^2)^{-\frac{1}{2}} dt = \arcsinh(t)+C$, as well as the fact that $\arcsinh$ is an odd function, we obtain
\begin{equation*}
    p(x_1)
    =
    \frac{1}{2\pi R}
    \int_{-\sqrt{\left(\frac{R}{x_1}\right)^2-1}}^{\sqrt{\left(\frac{R}{x_1}\right)^2-1}}
    \frac{dt}{\sqrt{1+t^2}}
    =
    \frac{1}{\pi R}
    \arcsinh
    \left(
    \sqrt{
    \left(\frac{R}{x_1}\right)^2-1
    }
    \right)
\end{equation*}
Next, we compute the conditional density $p(x_2|x_1)$:
\begin{equation*}
    p(x_2|x_1)
    =
    \frac{p(x_1,x_2)}{p(x_1)}
    =
    \frac{
    (2\pi R)^{-1}
    \left(
    x_1^2+x_2^2
    \right)
    ^{-1}
    }
    {(\pi R)^{-1}\arcsinh
    \left(
    \sqrt{
    \left(\frac{R}{x_1}\right)^2-1
    }
    \right)
    }
    =
    \left(
    2\sqrt{x_1^2+x_2^2}\arcsinh
    \left(
    \sqrt{
    \left(\frac{R}{x_1}\right)^2-1
    }
    \right)
    \right)^{-1}
\end{equation*}

Finally, we compute the off-diagonal term in the general form of the inverse Jacobian for Damois-style solutions in~\eqref{eq:general_Jacobian_Darmois_2d}:
\begin{align*}
    c_{21}(\xb)
    &=
    \frac{\partial }{\partial x_1} \int_{-\infty}^{x_2} p(x_2|x_1) dx_2
    =
     \frac{\partial }{\partial x_1}
     \int_{-\sqrt{R^2-x_1^2}}^{x_2}
     \frac{dx_2}
     {
    2\sqrt{x_1^2+x_2^2}\arcsinh
    \left(
    \sqrt{
    \left(\frac{R}{x_1}\right)^2-1
    }
    \right)
    }
    \\
    &=
    \frac{1}{2}
    \frac{\partial }{\partial x_1}
    \left(
    \arcsinh
    \left(
    \sqrt{
    \left(\frac{R}{x_1}\right)^2-1
    }
    \right)^{-1}
    \int_{-\sqrt{R^2-x_1^2}}^{x_2}
     \frac{dx_2}
     {
    \sqrt{x_1^2+x_2^2}
    }
    \right)
    \\
    &=
    \frac{1}{2}
    \frac{\partial }{\partial x_1}
    \left(
    \arcsinh
    \left(
    \sqrt{
    \left(\frac{R}{x_1}\right)^2-1
    }
    \right)^{-1}
    \left(
    \arcsinh(x_2)
    -
    \arcsinh
    \left(
    -
    \sqrt{
    \left(\frac{R}{x_1}\right)^2-1
    }
    \right)
    \right)
    \right)
    \\
    &=
    \frac{1}{2}
    \frac{\partial }{\partial x_1}
    \left(
    \arcsinh
    \left(
    \sqrt{
    \left(\frac{R}{x_1}\right)^2-1
    }
    \right)^{-1}
    \left(
    \arcsinh(x_2)
    +
    \arcsinh
    \left(
    \sqrt{
    \left(\frac{R}{x_1}\right)^2-1
    }
    \right)
    \right)
    \right)
    \\
    &=
    \frac{1}{2}
    \frac{\partial }{\partial x_1}
    \left(
    1+\frac{\arcsinh(x_2)}{\arcsinh
    \left(
    \sqrt{
    \left(\frac{R}{x_1}\right)^2-1
    }
    \right)}
    \right)
    \\
    &=
    \frac{1}{2}
    \arcsinh(x_2)
    \frac{\partial }{\partial x_1}
    \left(
    \arcsinh
    \left(
    \sqrt{
    \left(\frac{R}{x_1}\right)^2-1
    }
    \right)^{-1}
    \right)
    \\
    &=
    -
    \frac{1}{2}
    \arcsinh(x_2)
    \arcsinh
    \left(
    \sqrt{
    \left(\frac{R}{x_1}\right)^2-1
    }
    \right)^{-2}
    \frac{\partial }{\partial x_1}
    \arcsinh
    \left(
    \sqrt{
    \left(\frac{R}{x_1}\right)^2-1
    }
    \right)
\end{align*}
Using the derivative $\frac{\partial}{\partial t}\arcsinh(t)=(t^2+1)^{-\frac{1}{2}}$ and repeatedly applying the chain rule, we obtain:
\begin{align*}
    c_{21}(\xb)
    &=
    -
    \frac{1}{2}
    \arcsinh(x_2)
    \arcsinh
    \left(
    \sqrt{
    \left(\frac{R}{x_1}\right)^2-1
    }
    \right)^{-2}
    \frac{x_1}{R}
    \frac{\partial }{\partial x_1}
    \left(
    \sqrt{
    \left(\frac{R}{x_1}\right)^2-1
    }
    \right)
    \\
    &=
    -
    \frac{1}{2}
    \arcsinh(x_2)
    \arcsinh
    \left(
    \sqrt{
    \left(\frac{R}{x_1}\right)^2-1
    }
    \right)^{-2}
    \frac{x_1}{R}
    \frac{1}{2}
    \frac{1}{\sqrt{
    \left(\frac{R}{x_1}\right)^2-1
    }}
    (-2) 
    R^2
    x_1^{-3}
    \\
    &=
    \frac{R}{2x_1 \sqrt{
    R^2-x_1^2
    }}
    \arcsinh(x_2)
    \arcsinh
    \left(
    \sqrt{
    \left(\frac{R}{x_1}\right)^2-1
    }
    \right)^{-2}
\end{align*}

Again, recall that this only holds inside the disk of radius $R$, otherwise $c_{12}=0$ (as the CDF will be zero or one, irrespective of $x_1$).

The $C_\IMA$ for the Darmois solution thus takes the form:
\begin{align*}
    C_\IMA^\text{Darmois}
    &=
    \int_{}^{} 
    \frac{1}{2}
    \log 
    \left(
    p(x_1)^{-2}+c_{21}(\xb)^2p(x_1,x_2)^{-2}
    \right)
    +
    \log 
    \left(
    p(x_2|x_1)^{-1}
    \right)
    -
    \log
    \left(
    p(x_1,x_2)^{-1}
    \right)
    d\sb
    \\
    &=
    \int_{}^{} 
    \frac{1}{2}
    \log 
    \Bigg[
    \left(
    \frac{1}{\pi R}
    \arcsinh
    \left(
    \sqrt{
    \left(\frac{R}{x_1}\right)^2-1
    }
    \right)
    \right)^{-2}
    \\
    &\quad 
    +
    \left(
    \frac{R}{2x_1 \sqrt{
    R^2-x_1^2
    }}
    \arcsinh(x_2)
    \arcsinh
    \left(
    \sqrt{
    \left(\frac{R}{x_1}\right)^2-1
    }
    \right)^{-2}
    \right)^2
    \left(
    \frac{1}{\sqrt{x_1^2+x_2^2}} \frac{1}{2\pi R}
    \right)^{-2}
    \Bigg]
    \\
    &\quad
    +
    \log 
    \left(
    2\sqrt{x_1^2+x_2^2}\arcsinh
    \left(
    \sqrt{
    \left(\frac{R}{x_1}\right)^2-1
    }
    \right)
    \right)
    \\
    &\quad
    -
    \log 
    \left(
    2\pi R
    \right)
    -
    \frac{1}{2}
    \log 
    \left(
    x_1^2+x_2^2
    \right)
    d\sb
    \\
    &=
    \int 
    \frac{1}{2}
    \log 
    \Bigg[
    \pi^2 R^2
    \arcsinh
    \left(
    \sqrt{
    \left(\frac{R}{x_1}\right)^2-1
    }
    \right)^{-2}
    \\
    &\quad 
    +
    \frac{R^2}{4x_1^2
    (R^2-x_1^2)}
    \arcsinh(x_2)^2
    \arcsinh
    \left(
    \sqrt{
    \left(\frac{R}{x_1}\right)^2-1
    }
    \right)^{-4}
    (x_1^2+x_2^2)4\pi^2 R^2
    \Bigg]
    \\
    &\quad
    +
    \log(2)
    +
    \frac{1}{2}
    \log(x_1^2+x_2^2)
    +
    \log
    \left(
    \arcsinh
    \left(
    \sqrt{
    \left(\frac{R}{x_1}\right)^2-1
    }
    \right)
    \right)
    \\
    &\quad
    -
    \log(2)
    -
    \log
    \left(
    \pi R
    \right)
    -
    \frac{1}{2}
    \log 
    \left(
    x_1^2+x_2^2
    \right)
    d\sb
    \\
    &=
    \int 
    \frac{1}{2}
    \log 
    \Bigg[
    \pi^2 R^2
    \arcsinh
    \left(
    \sqrt{
    \left(\frac{R}{x_1}\right)^2-1
    }
    \right)^{-2}
    \\
    &\quad+
    \frac{\pi^2R^4(x_1^2+x_2^2)}{x_1^2
    (R^2-x_1^2)}
    \arcsinh(x_2)^2
    \arcsinh
    \left(
    \sqrt{
    \left(\frac{R}{x_1}\right)^2-1
    }
    \right)^{-4}
    \Bigg]
    \\
    &\quad
    +
    \log
    \left(
    \arcsinh
    \left(
    \sqrt{
    \left(\frac{R}{x_1}\right)^2-1
    }
    \right)
    \right)
    -
    \log
    \left(
    \pi R
    \right)
    d\sb
        \\
    &=
    \int
    \frac{1}{2}
    \log 
    \left(
    1
    +
    \frac{R^2(x_1^2+x_2^2)\arcsinh(x_2)^2}
    {x_1^2
    (R^2-x_1^2)
    \arcsinh
    \left(
    \sqrt{
    \left(\frac{R}{x_1}\right)^2-1
    }
    \right)^{2}
    }
    \right)
    d\sb
    >0
\end{align*}
where the strict inequality in the last step follows from the fact that the fraction inside the logarithm, and hence the entire integrand, is strictly positive within the disk of integration.

We have thus shown that for the example of an orthogonal coordinate transformation from polar to Cartesian coordinates, which is not a conformal map, the $C_\IMA$ os the true solution is zero and that of the Darmois construction is strictly greater than zero, hence the two can be distinguished based on the value of the $C_\IMA$ contrast.
\end{example}

%% file: appendix/E_experiments.tex
The code for our experiments (enclosed in the supplemental material) is in Python; we use Jax~\cite{jax2018github}, Distrax~\cite{distrax2021github} and Haiku~\cite{haiku2020github} to implement our models; the Jacobian and $C_\IMA$ computation and optimisation are performed with the automatic differentiation tools provided in Jax.

\subsection{Sampling random M\"obius transformations.} 

In order to generate mixing functions with $C_\IMA=0$, we use M\"obius transformations (see~\Cref{app:confmaps} and in particular~\Cref{thm:highdimmoebius}, for additional details on this kind of functions) with randomly sampled parameters, as specified below. 
A M\"obius transformation $\fb^{\text{M}} : \RR^n \rightarrow \RR^n$ is given by
\begin{equation}
\label{eq:moebius_transf}
\fb^{\text{M}}(\sb)=\tb+\frac{r \Ab (\sb-\bb)}{\norm{\sb-\bb}^{\epsilon}}\,,
\end{equation}
with parameters $\bb, \tb \in \RR^n$, $r \in \RR$, $\Ab$ is an orthogonal matrix and $\epsilon \in \{0,2\}$. The flow models we train have an diagonal affine layer at the top with fixed shift and scale set to the mean and standard deviation of the training data, thereby normalizing the inputs. Hence, without loss of generality, we can set the $\tb$ parameter to zero and $r$ to one. Since $\epsilon = 0$ corresponds to a linear transformation, we generally set $\epsilon = 2$ in our experiments unless otherwise specified. We sample the orthogonal matrix through the \texttt{ortho\_group} function in \texttt{scipy.stats}~\cite{virtanen2020scipy}. To avoid singularities given by a vanishing denominator in the second term on the RHS of~\eqref{eq:moebius_transf}, which would yield observed distributions with strong outliers and therefore hard to fit for our models, we restrict $\bb$ to lie outside the unit square $\sb$ is sampled from. We achieve this by sampling $\bb$ from a normal distribution and reject the sample until it is located outside of the unit square.

\subsection{How to implement the Darmois construction}
\label{app:darmois_flows}

In the following, we describe how the Darmois construction can be implemented based on normalising flow models~\cite{papamakarios2019normalizing}. The key idea is that the components $g_i^{\text{D}}$ of the Darmois construction~\eqref{eq:Darmois_construction} are conditional (cumulative) density functions corresponding to %
a given factorisation $p(\xb) = \prod_{i=1}^n p(x_i| \xb_{1:i-1})$ of the likelihood. A flow model with triangular Jacobian can be used to maximise the likelihood of the observations under a change of variable respecting said factorisation, and learning to map the observed variables onto a given (factorised) base distribution. After training, and provided that the model is expressive enough, the CDF of each component of the reconstructed sources should match that of the base distribution. By further transforming each reconstructed variable through said CDF, we achieve a global mapping of the observations onto a Uniform distribution on the $n$-dimensional hypercube, with a triangular Jacobian, matching the transformation operated by the Darmois construction (see also see~\cite{papamakarios2019normalizing}, section 2.2).
Note that, for the purpose of computing the $C_\IMA$ of the Darmois construction, this final step can be omitted due to~\Cref{prop:global_IMA_contrast_properties}, \textit{(ii)}, stating that the contrast is blind to element-wise reparametrisations of the sources.

We remark that, while the possibility of using normalising flows to ``learn'' the Darmois construction is mentioned in~\cite{papamakarios2019normalizing, huang2018neural}, where a similar construction is mentioned in a theoretical argument to prove ``universal approximation capacity for densities'' for normalising flow models with triangular Jacobian, it has to the best of our knowledge not been tested empirically, since autoregressive modules with triangular Jacobian are typically used in combination with permutation, shuffling or linear layers which overall lead to architectures with a non-triangular Jacobian.

\paragraph{Expressive normalising flow with triangular Jacobian.} To obtain an expressive normalizing flow with triagular Jacobian, we modify the residual flow model \cite{chen2019residualflows}.\footnote{We describe how to implement a function with upper triangular Jacobian, but the reasoning can be extended to implement functions whose Jacobian is lower triangular.} A residual flow is a residual network which is made invertible through spectral normalization. Each layer is given by
\begin{equation}
    \zb' = \zb + \gb(\zb),
\end{equation}
where $\zb', \zb \in \RR^n$ and $\gb: \RR^n \rightarrow \RR^n$ is a small neural network. Due to the chain rule, for the Jacobian of the overall flow model to be triangular, a sufficient condition is that all the layers have triangular Jacobian. Since the Jacobian of $\fb(\zb) = \zb$ is the identity matrix, we can restrict our attention to the neural network $\gb$. In our experiments, this is going to be a fully connected network. If it has $l$ layers and $h\geq n$ hidden units, it is given by
\begin{equation}
    \gb(\zb) = \bb_1 + \Wb_1 \phi(\bb_2 + \Wb_2\phi(\bb_3 + \Wb_3\cdots\phi(\bb_l + \Wb_l\zb)\cdots)),
\end{equation}
where $\phi: \RR^n \rightarrow \RR^n$ is an element-wise nonlinearity, $\bb_1\in\RR^n$, $\bb_2,...,\bb_l\in \RR^h$ are the biases, and $\Wb_1 \in \RR^{n\times h}$, $\Wb_2,...,\Wb_{l-1} \in \RR^{h\times h}$, $\Wb_l \in \RR^{h\times n}$ are the weight matrices. In order for the Jacobian of $\gb$ to be triangular, $g_n(\zb)$ should only depend on $z_{n}$, $g_{n-1}(\zb)$ should only depend on $z_{n}$ and $z_{n-1}$, and so on. To achieve this, we make the weight matrices block triangular as indicated in \eqref{equ:triresflow_weightmat_1}, \eqref{equ:triresflow_weightmat_l}, and \eqref{equ:triresflow_weightmat_2}.

\newcommand\undermat[2]{%
  \makebox[0pt][l]{$\smash{\underbrace{\phantom{%
    \begin{matrix}#2\end{matrix}}}_{\text{$#1$}}}$}#2}
    
\begin{equation}
    \Wb_1 = \left(
    \begin{array}{rrrr}
    * & * &  & * \\
    \vdots & \vdots & & \vdots \\
    * & * &  & * \\
    0 & * &  & * \\
    \vdots & \vdots & & \vdots \\
    0 & * &  & * \\
    & & \ddots & \\
    0 & 0 &  & * \\
    \vdots & \vdots & & \vdots \\
    0 & 0 &  & * 
    \end{array}
    \right)
    \hspace{-0.35cm}
    \begin{tabular}{l}
    $\left.\lefteqn{\phantom{\begin{matrix} *\\ \vdots\\ *\ \end{matrix}}}\right\}h_1$\\
    $\left.\lefteqn{\phantom{\begin{matrix} *\\ \vdots\\ *\ \end{matrix}}}\right\}h_2$\\
    $\phantom{\ddots}$ \\
    $\left.\lefteqn{\phantom{\begin{matrix} *\\ \vdots\\ *\ \end{matrix}}}\right\}h_n$\\
    \end{tabular} 
    \label{equ:triresflow_weightmat_1}
\end{equation}

\begin{equation}
    \Wb_l = \left(
    \begin{array}{rrrrrrrrrr}
    * & \cdots & * & * & \cdots & * & & * & \cdots & * \\
    0 & \cdots & 0 & * & \cdots & * & & * & \cdots & * \\
     &  &  &  & & & \ddots & & & \\
    \undermat{h_1}{0 & \cdots & 0} & \undermat{h_2}{0 & \cdots & 0} & & \undermat{h_n}{* & \cdots & *} \\
    \end{array}
    \right)
    \label{equ:triresflow_weightmat_l}
\end{equation}
\vspace{0.3cm}

\begin{equation}
    \Wb_i = \left(
    \begin{array}{rrrrrrrrrr}
    * & \cdots & * & * & \cdots & * & & * & \cdots & * \\
    \vdots & \ddots & \vdots & \vdots & \ddots & \vdots & & \vdots & \ddots & \vdots \\
    * & \cdots & * & * & \cdots & * & & * & \cdots & * \\
    0 & \cdots & 0 & * & \cdots & * & & * & \cdots & * \\
    \vdots & \ddots & \vdots & \vdots & \ddots & \vdots & & \vdots & \ddots & \vdots \\
    0 & \cdots & 0 & * & \cdots & * & & * & \cdots & * \\
     &  &  &  & & & \ddots & & & \\
     0 & \cdots & 0 & 0 & \cdots & 0 & & * & \cdots & * \\
    \vdots & \ddots & \vdots & \vdots & \ddots & \vdots & & \vdots & \ddots & \vdots \\
    \undermat{h_1}{0 & \cdots & 0} & \undermat{h_2}{0 & \cdots & 0} & & \undermat{h_n}{* & \cdots & *} \\
    \end{array}
    \right)
    \hspace{-0.35cm}
    \begin{tabular}{l}
    $\left.\lefteqn{\phantom{\begin{matrix} *\\ \vdots\\ *\ \end{matrix}}}\right\}h_1$\\
    $\left.\lefteqn{\phantom{\begin{matrix} *\\ \vdots\\ *\ \end{matrix}}}\right\}h_2$\\
    $\phantom{\ddots}$ \\
    $\left.\lefteqn{\phantom{\begin{matrix} *\\ \vdots\\ *\ \end{matrix}}}\right\}h_n$\\
    \end{tabular} 
    \hspace{0.4cm}\text{for}\,\, i\in\{2, ..., l-1\}
    \label{equ:triresflow_weightmat_2}
\end{equation}
\vspace{0.3cm}

Here, $h_i$ is the number of hidden units dedicated to transforming $\zb_{i}$ with the constraint $\sum_{i=1}^n h_i = h$. We perform an even split such that the $h_i$ and $h_j$ differ by at most 1 for $i, j\in\{1, ..., n\}$. The weight matrices are restricted to be block triangular during optimization by setting the respective matrix elements to zero after each iteration of the optimizer. The model can simply be made and kept invertible using the same spectral normalization as is used for dense residual flows~\cite{chen2019residualflows}.
We train our model to map onto a standard Normal base distribution.

\subsection{Generating random MLP mixing functions}

In order to generate random MLP mixing functions, we adopt the same initalisation as in~\cite{gresele2020relative}: we initialise the square weight matrices to be orthogonal,\footnote{Note that orthogonality of the weight matrices in a MLP does not guarantee satisfying~\Cref{principle:IMA}, due to the element-wise nonlinearities between the layers, which overall lead to a Jacobian whose columns are in general not orthogonal.}
and use the \texttt{leaky\_tanh} invertible nonlinearity.

\subsection{Maximum likelihood with low \texorpdfstring{$C_\IMA$}{}}
\label{app:modified_ml} 
The modified maximum likelihood objective described in~\Cref{sec:experiment2_learning} can be written as follows:\footnote{while the objective in~\Cref{sec:experiment2_learning} involves an expectation over $p_\xb$, we consider the loss for a single point $\xb$ here, $\Lcal(\gb; \xb)$.}
\begin{align}
\Lcal(\gb; \xb) =& \log p(\xb) - \lambda \cdot c_\IMA (\gb^{-1}, p_\yb) \nonumber \\
=& \sum_{i=1}^n \log p_{y_i}(\gb^i(\xb)) + \log | \Jb_{\gb}(\xb)|
- \lambda \cdot \left( \sum_{i = 1}^n \log \norm{[\Jb_{\gb^{-1}}(\gb(\xb))]_i} - \log \left|\Jb_{\gb^{-1}}(\gb(\xb))\right| \right) \nonumber \\
=& \sum_{i=1}^n \log p_{y_i}(\gb^i(\xb)) + \log | \Jb_{\gb}(\xb)|
- \lambda \cdot \left( \sum_{i = 1}^n \log \norm{[\Jb^{-1}_{\gb}(\xb)]_i} + \log \left|\Jb_{\gb}(\xb)\right| \right) \nonumber \\
=& \sum_{i=1}^n \log p_{y_i}(\gb^i(\xb)) + (1-\lambda)\log | \Jb_{\gb}(\xb)|  - \lambda \sum_i \log 
\norm{[\Jb^{-1}_{\gb}(\xb)]_i} \label{eq:objective_ml_cima} \,, 
\end{align}
where $[\Jb^{-1}_{\gb}(\xb)]_i$ represents the $i$-th column of the inverse of the Jacobian of $\gb$ computed at $\xb$.

We use the same model as the one described in~\Cref{app:darmois_flows}, but without the constraint that the Jacobian should be triangular, and train with a Logistic base distribution.

Note that the computational efficiency of optimising objective~\eqref{eq:objective_ml_cima} is cubic in the input size $n$, due to a number of operations (matrix inversion, Jacobian and determinant computation via automatic differentiation, etc.) which are $\mathcal{O}(n^3)$. However, similarly to what already observed in~\cite{halva2020hidden}, we found that for data of moderate dimensionality computing and optimising objective~\eqref{eq:objective_ml_cima} with automatic differentiation is feasible. For example, training a residual flow with 64 layers for $10^5$ iterations takes roughly 5.3 hours for $n=2$, 5.7 hours for $n=5$, and 6.3 hours for $n=7$ on the same hardware (see section \ref{sec:app_eval}).
An interesting direction for future work would be to find computationally efficient ways of optimising~\eqref{eq:objective_ml_cima}.

\begin{figure}[p]
    \centering
    \begin{subfigure}{0.24\textwidth}
        \begin{overpic}[width=\textwidth]{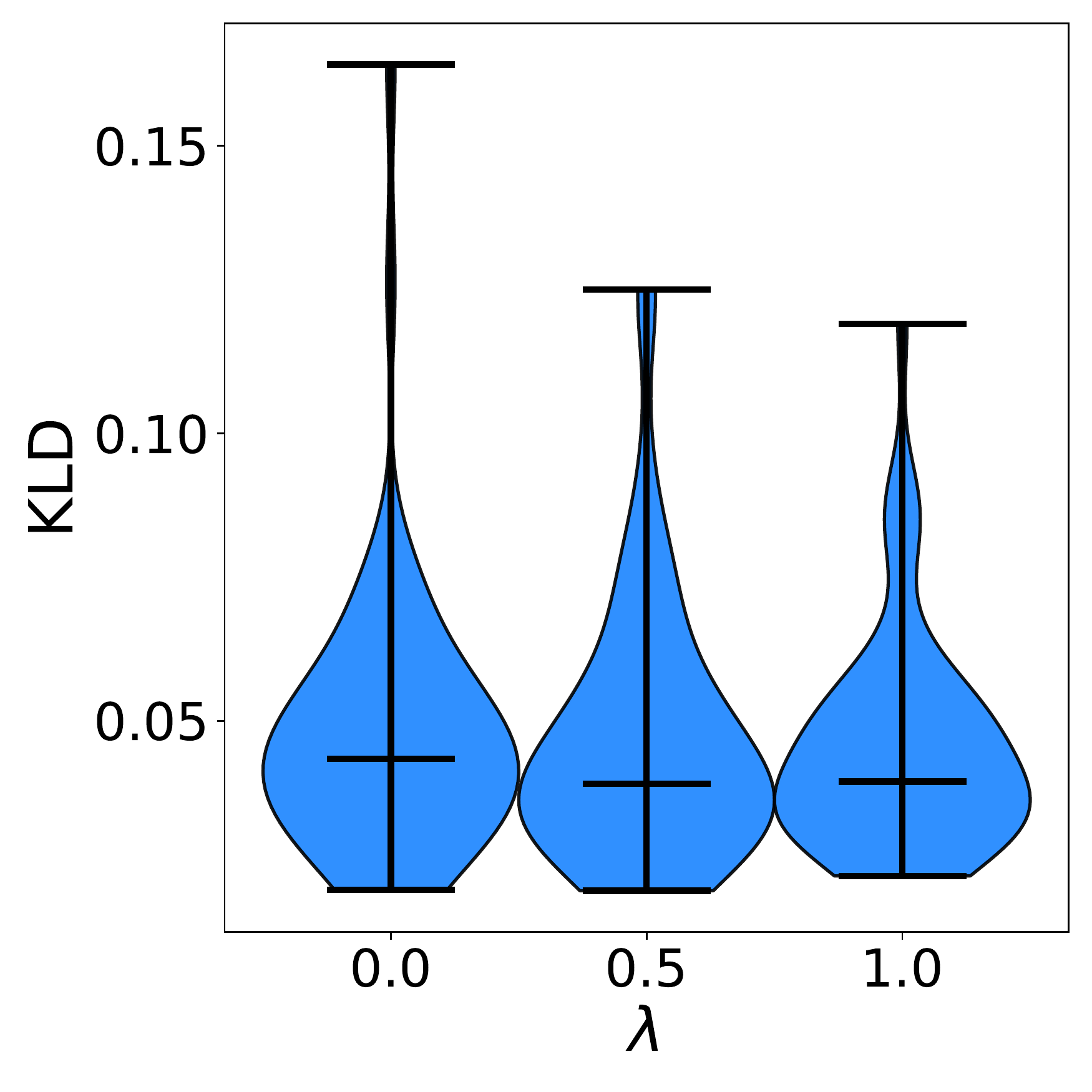}
            \put (10, 2) {\textbf{(a)}}
        \end{overpic}
    \end{subfigure}
    \begin{subfigure}{0.24\textwidth}
        \begin{overpic}[width=\textwidth]{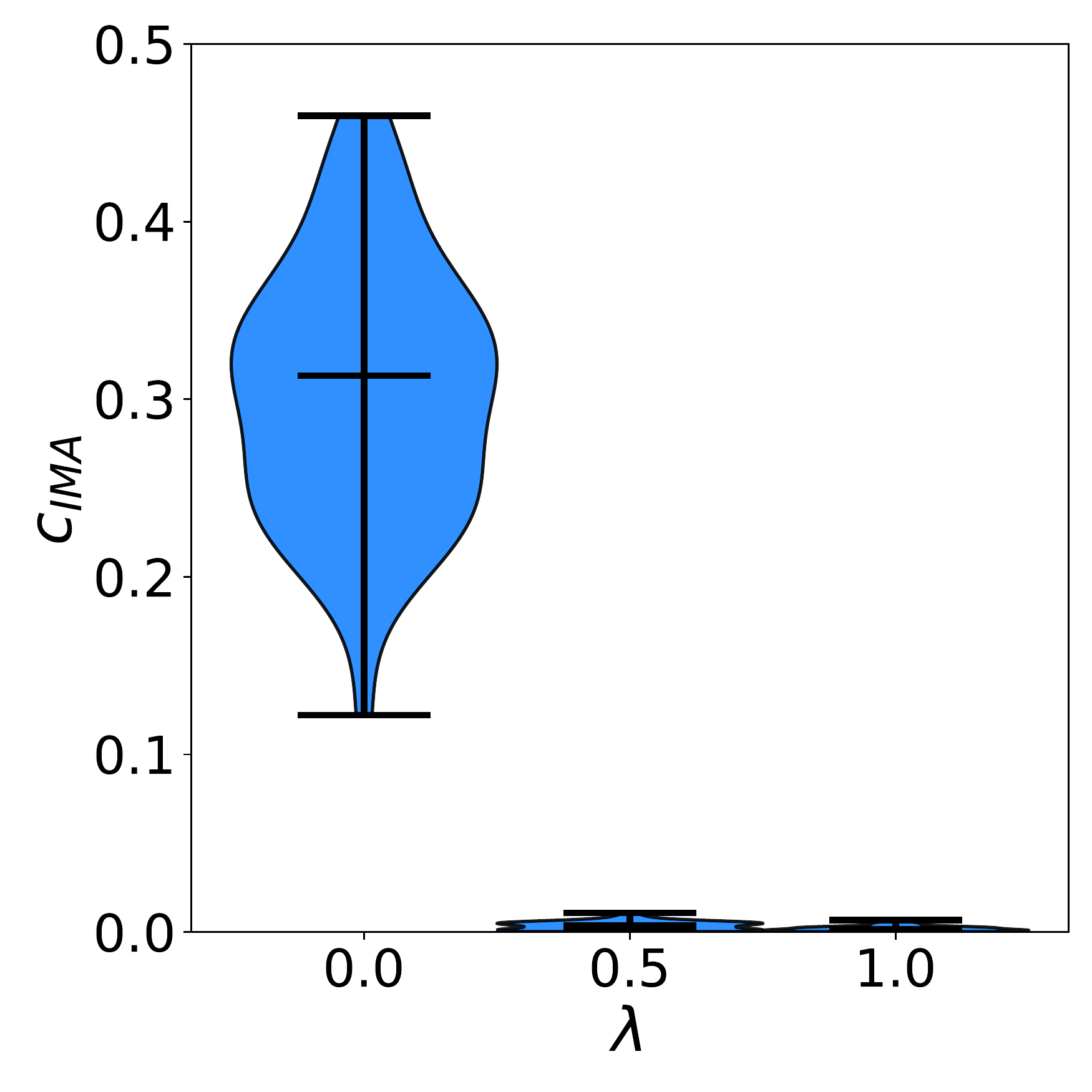}
            \put (10, 2) {\textbf{(b)}}
        \end{overpic}
    \end{subfigure}
    \begin{subfigure}{0.24\textwidth}
        \begin{overpic}[width=\textwidth]{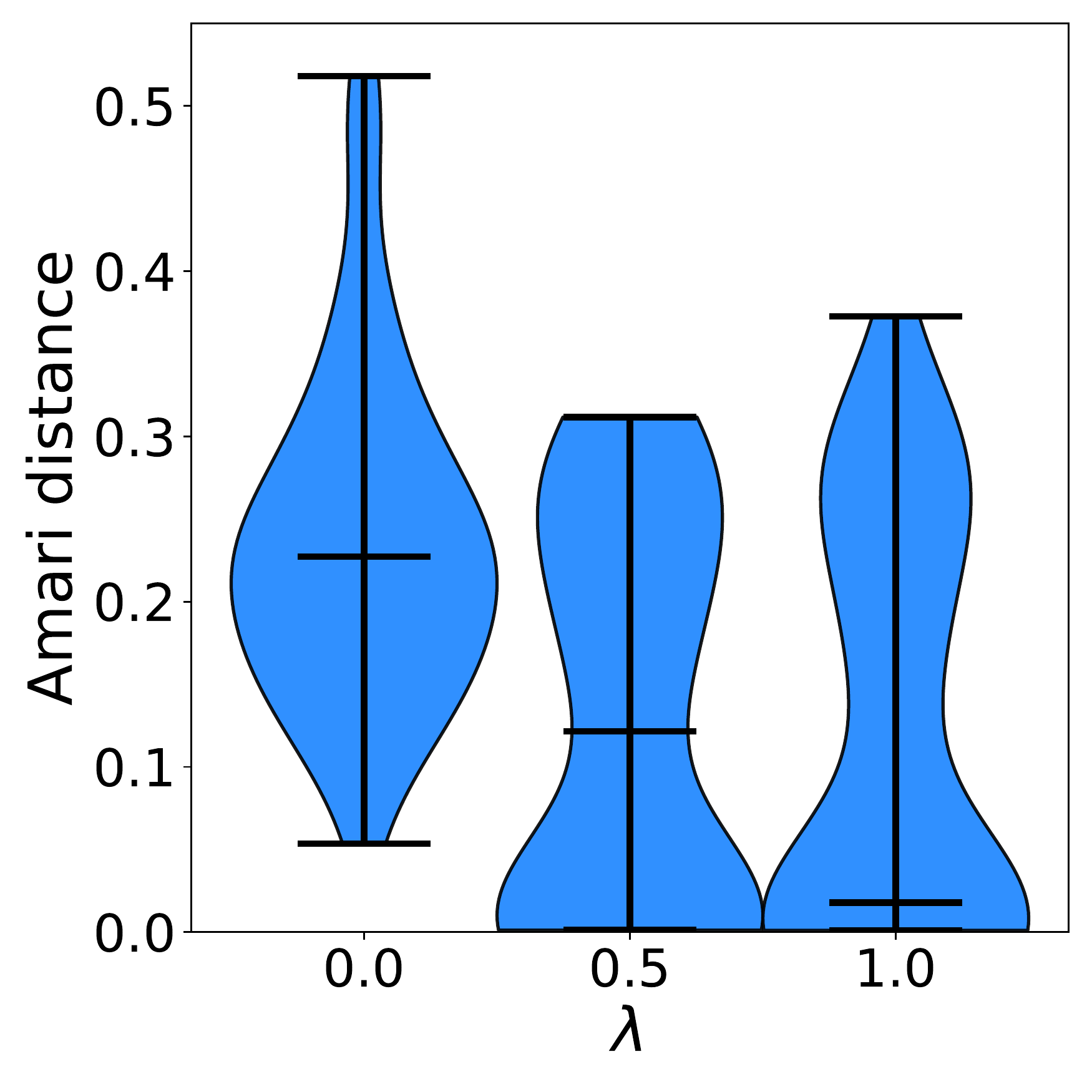}
            \put (10, 2) {\textbf{(c)}}
        \end{overpic}
    \end{subfigure}
    \begin{subfigure}{0.24\textwidth}
        \begin{overpic}[width=\textwidth]{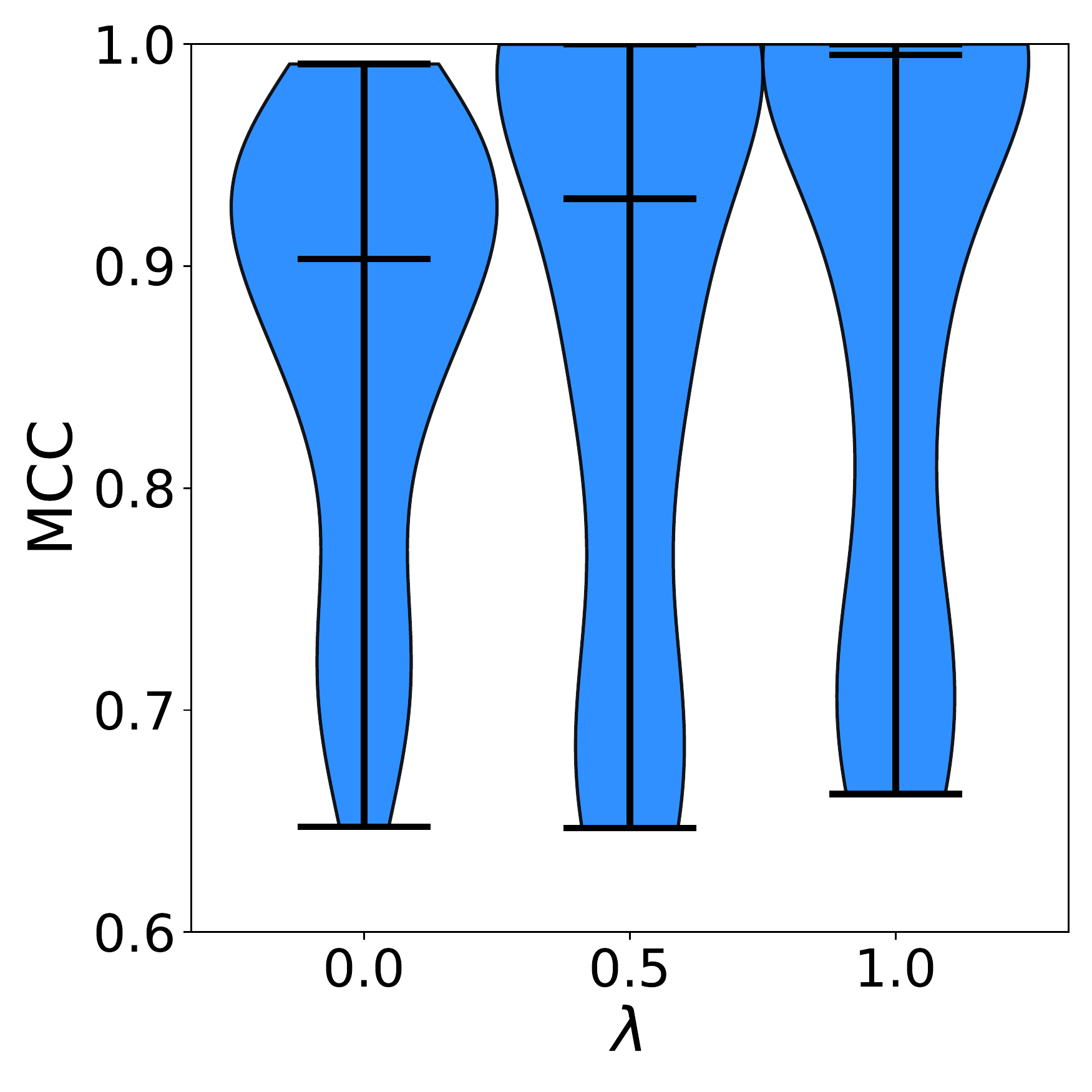}
            \put (10, 2) {\textbf{(d)}}
        \end{overpic}
    \end{subfigure}
    \\
    \begin{subfigure}{0.24\textwidth}
        \begin{overpic}[width=\textwidth]{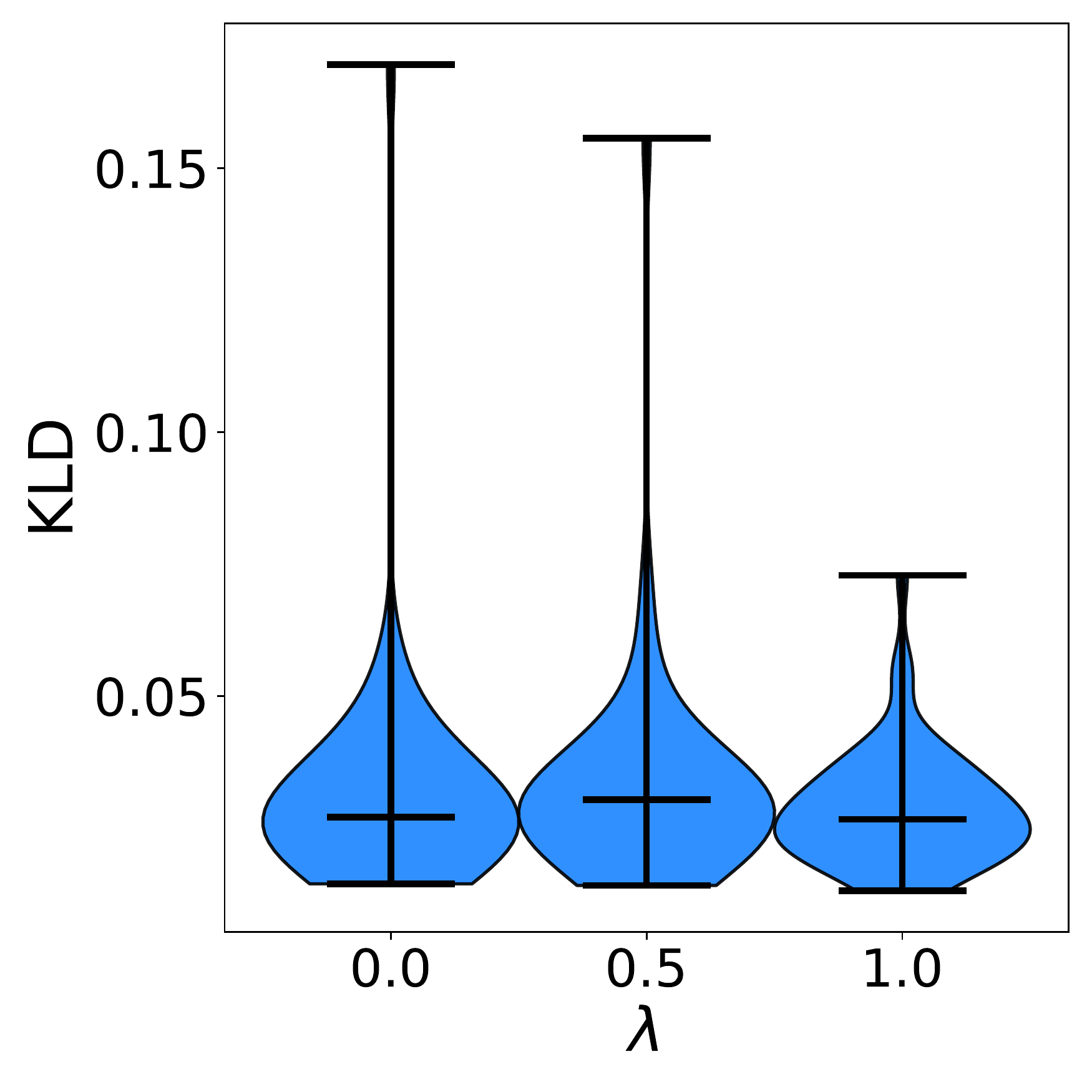}
            \put (10, 2) {\textbf{(e)}}
        \end{overpic}
    \end{subfigure}
    \begin{subfigure}{0.24\textwidth}
        \begin{overpic}[width=\textwidth]{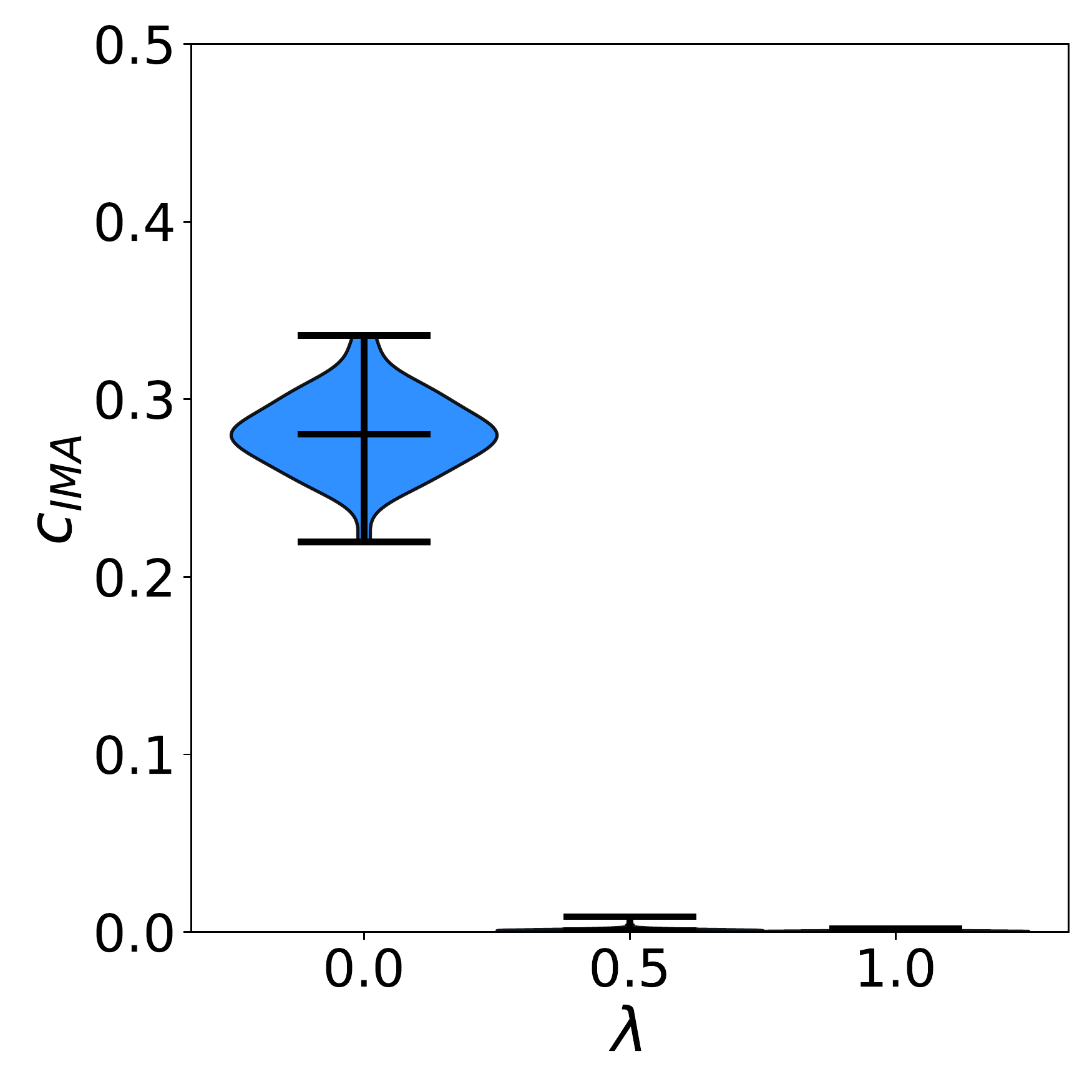}
            \put (10, 2) {\textbf{(f)}}
        \end{overpic}
    \end{subfigure}
    \begin{subfigure}{0.24\textwidth}
        \begin{overpic}[width=\textwidth]{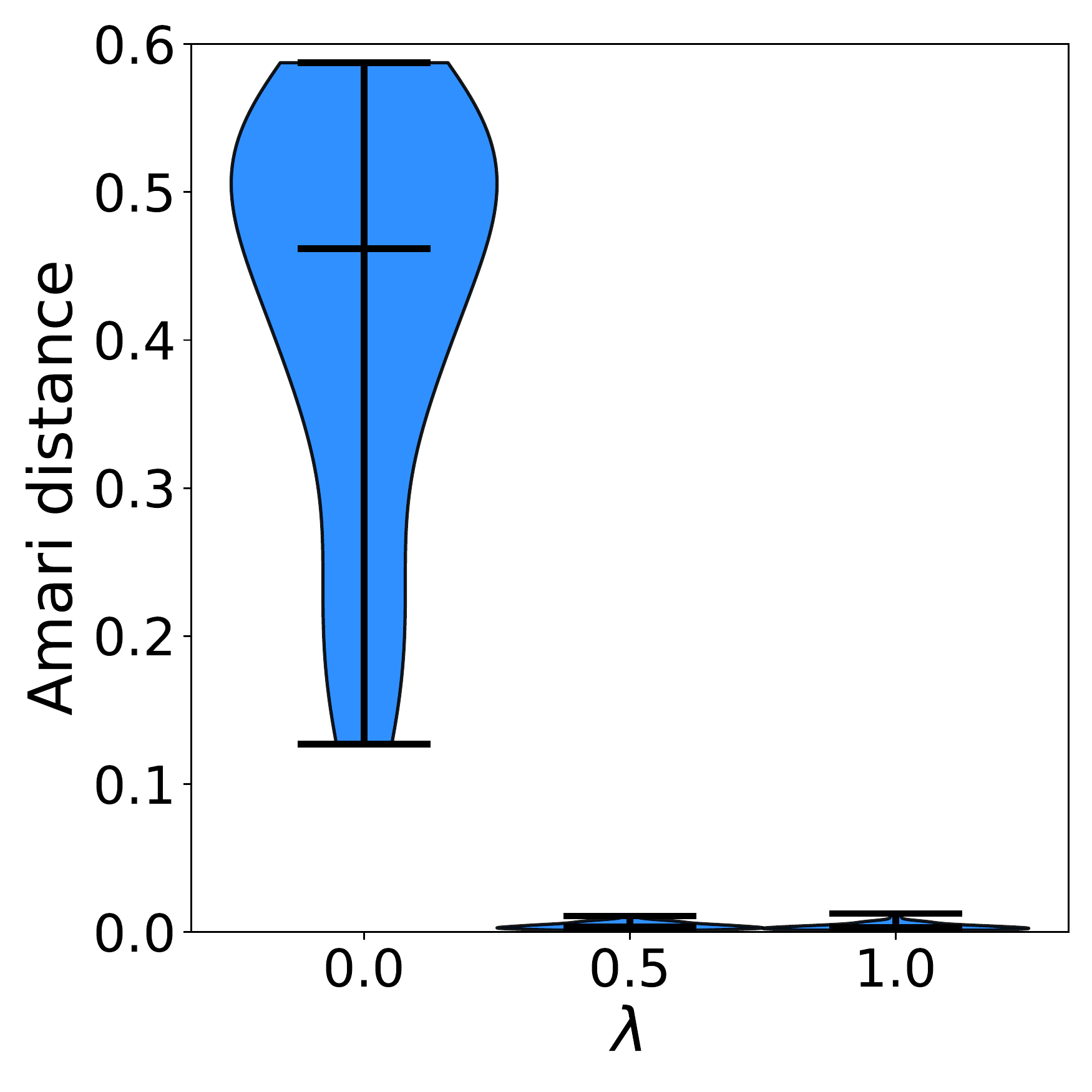}
            \put (10, 2) {\textbf{(g)}}
        \end{overpic}
    \end{subfigure}
    \begin{subfigure}{0.24\textwidth}
        \begin{overpic}[width=\textwidth]{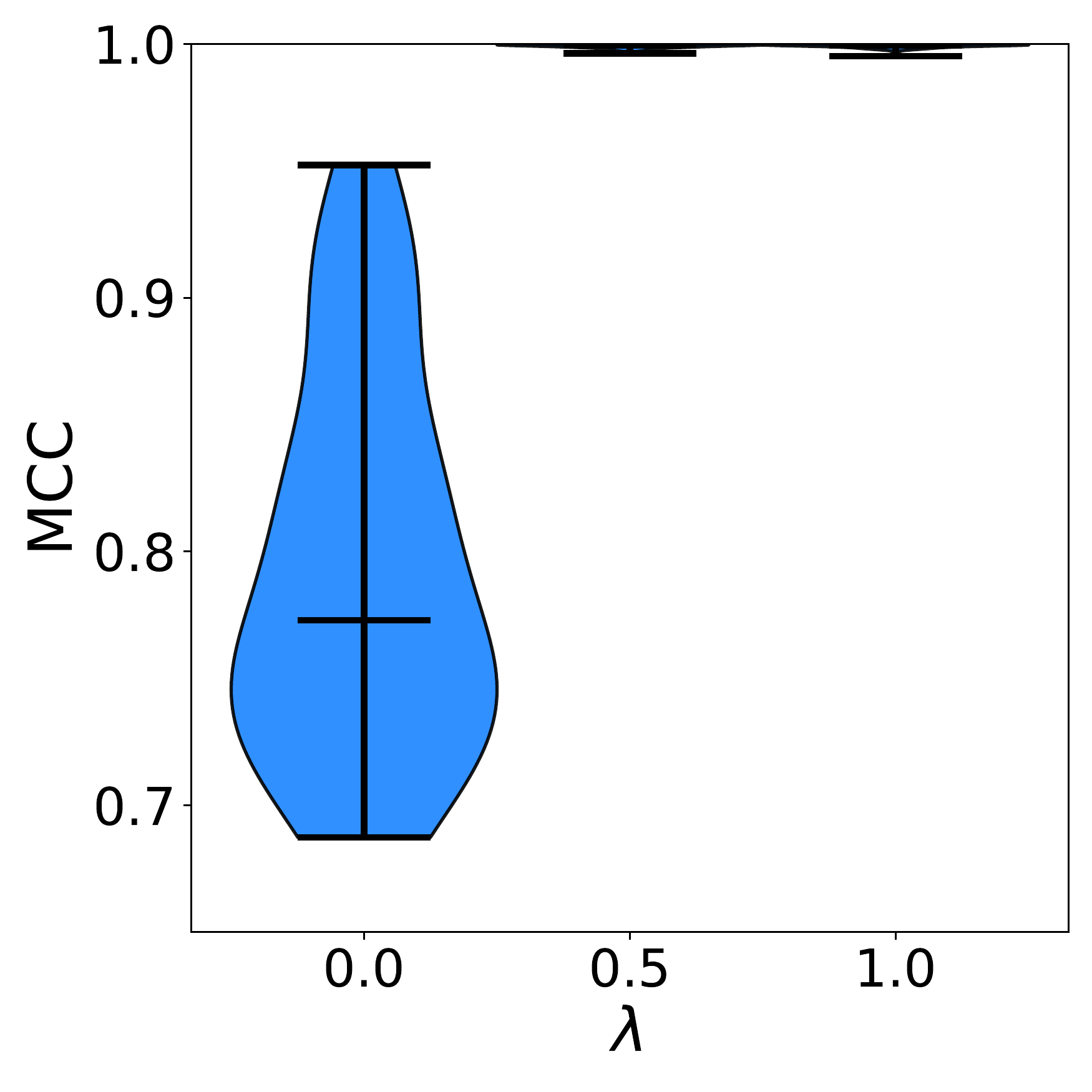}
            \put (10, 2) {\textbf{(h)}}
        \end{overpic}
    \end{subfigure}
    \caption{BSS via $C_\IMA$-regularised MLE for $n=2$ dimensions with $\lambda\in \{0.0,0.5,1.0\}$. The true mixing function is a randomly generated M\"obius transformation, nonlinear (with $\epsilon=2$) in \textbf{(a)}--\textbf{(d)} and linear (with $\epsilon=0$) transformation for \textbf{(e)}--\textbf{(h)}. For each type of transformation and $\lambda$, seeded runs are done. \textbf{(a)}, \textbf{(e)} KL-divergence between ground truth likelihood and learnt model; \textbf{(b)}, \textbf{(f)} $C_\IMA$ of the learnt models; \textbf{(c)}, \textbf{(g)} nonlinear Amari distance given true mixing and learnt unmixing; \textbf{(d)}, \textbf{(h)} MCC between true and reconstructed sources.}
    \label{fig:cima_obj_2d}
\end{figure}

\begin{figure}%
    \centering
    \begin{subfigure}{0.24\textwidth}
        \begin{overpic}[width=\textwidth]{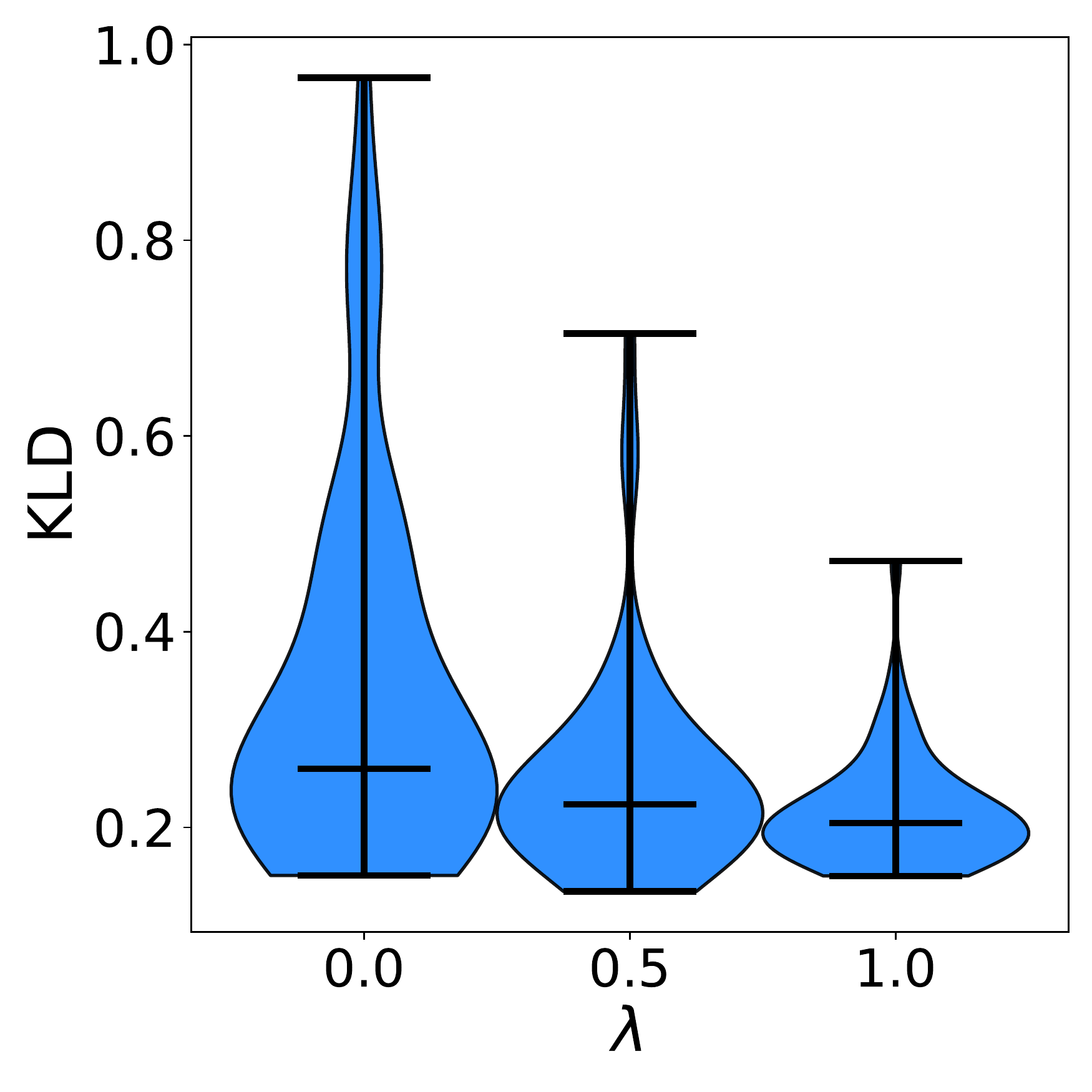}
            \put (10, 2) {\textbf{(a)}}
        \end{overpic}
    \end{subfigure}
    \begin{subfigure}{0.24\textwidth}
        \begin{overpic}[width=\textwidth]{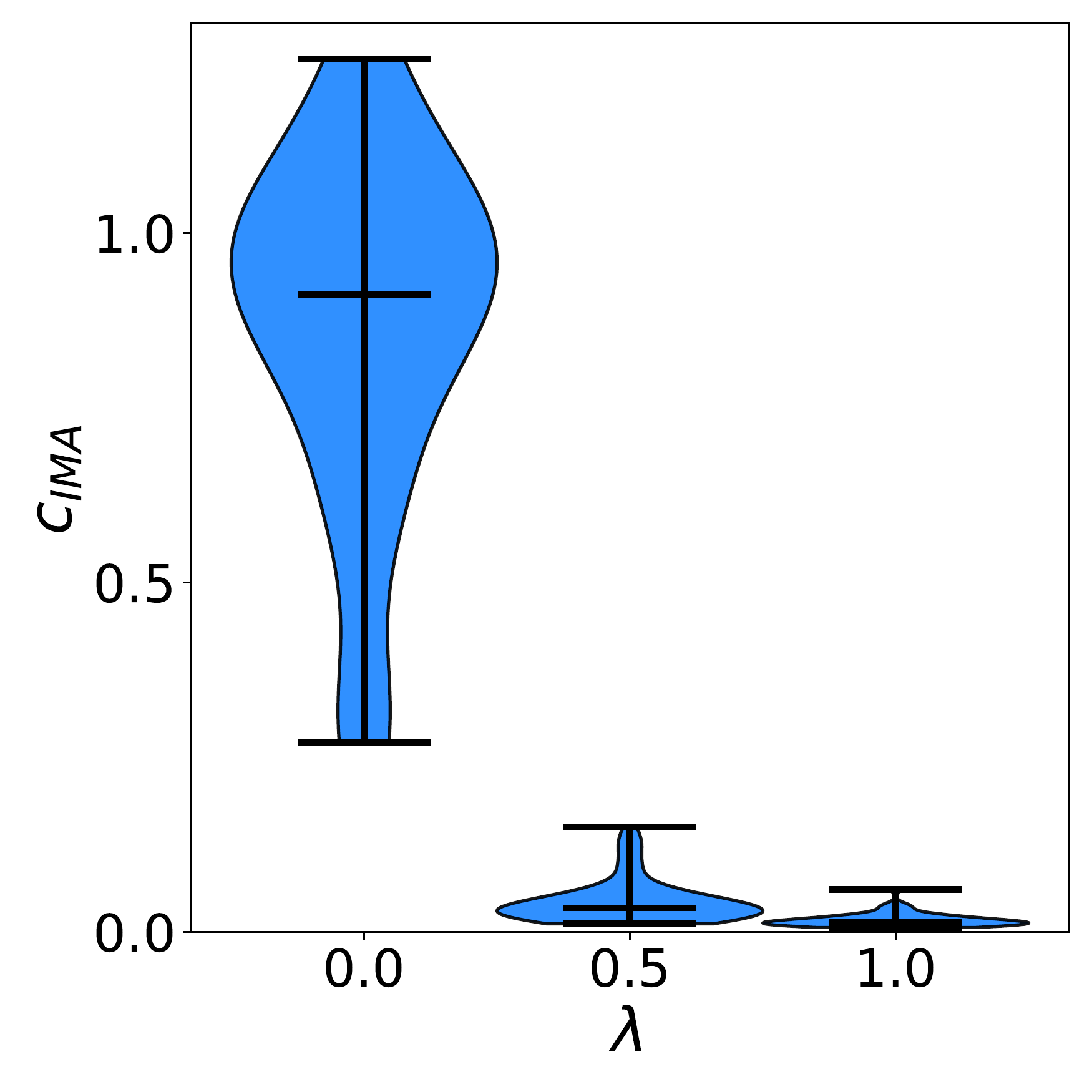}
            \put (10, 2) {\textbf{(b)}}
        \end{overpic}
    \end{subfigure}
    \begin{subfigure}{0.24\textwidth}
        \begin{overpic}[width=\textwidth]{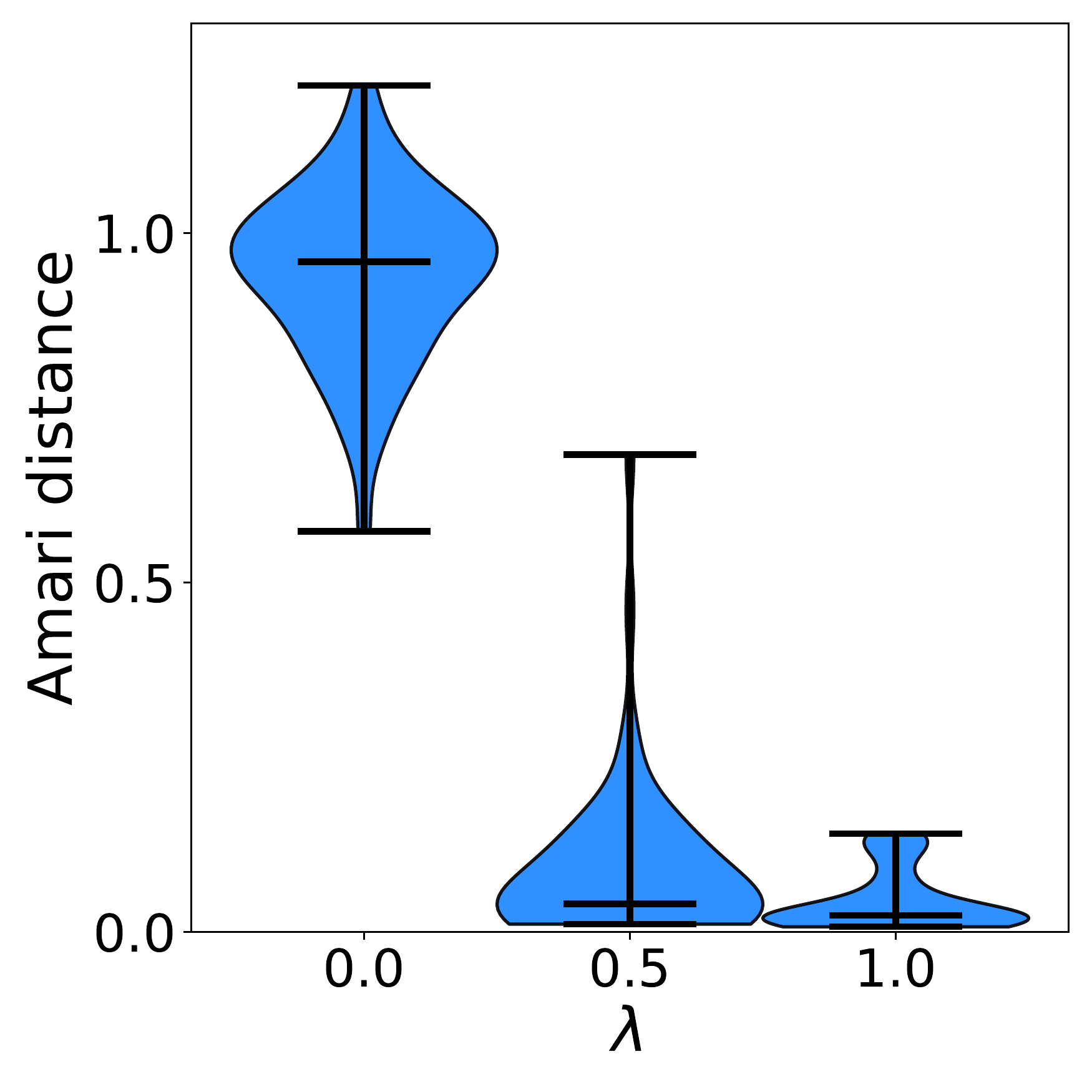}
            \put (10, 2) {\textbf{(c)}}
        \end{overpic}
    \end{subfigure}
    \begin{subfigure}{0.24\textwidth}
        \begin{overpic}[width=\textwidth]{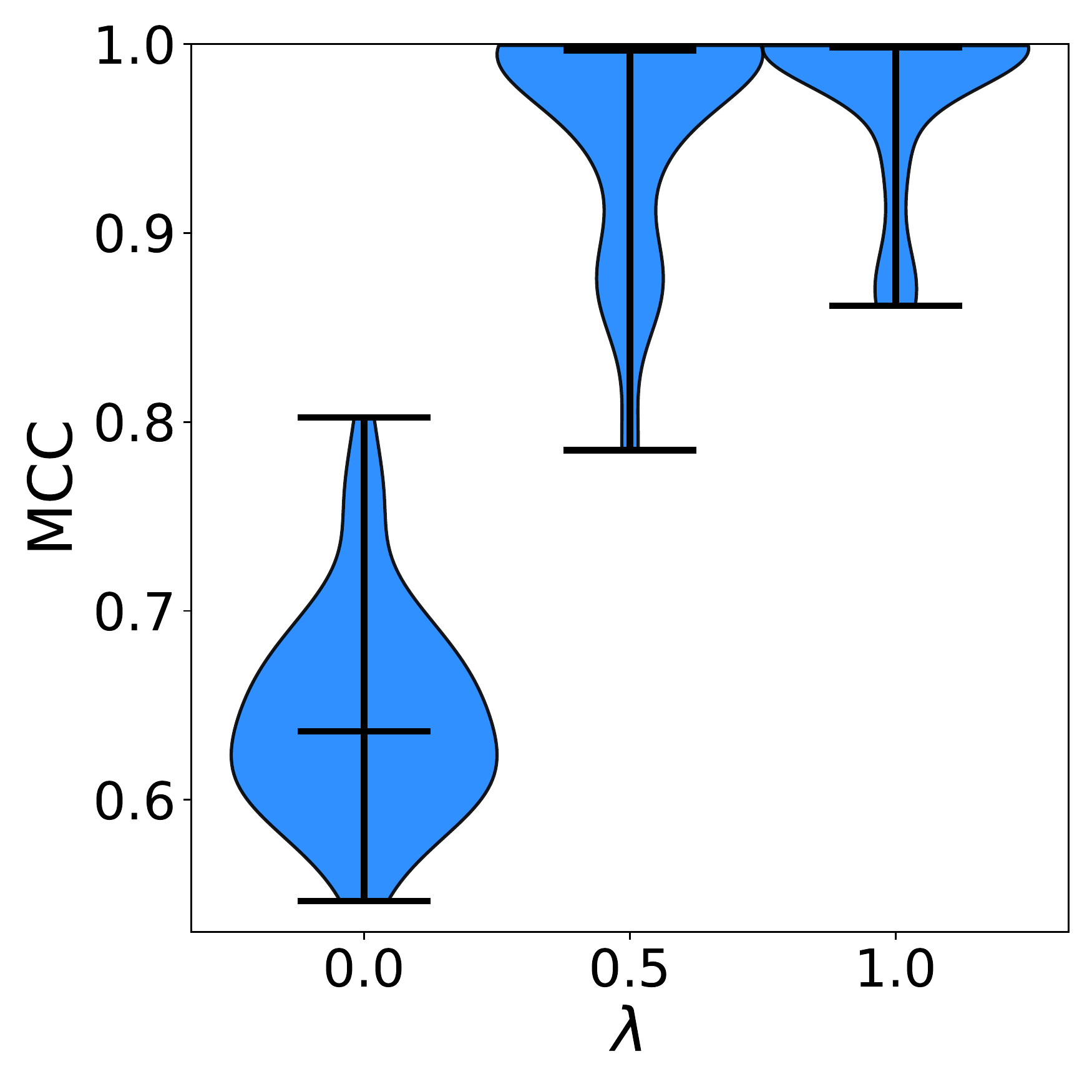}
            \put (10, 2) {\textbf{(d)}}
        \end{overpic}
    \end{subfigure}
    \\
    \begin{subfigure}{0.24\textwidth}
        \begin{overpic}[width=\textwidth]{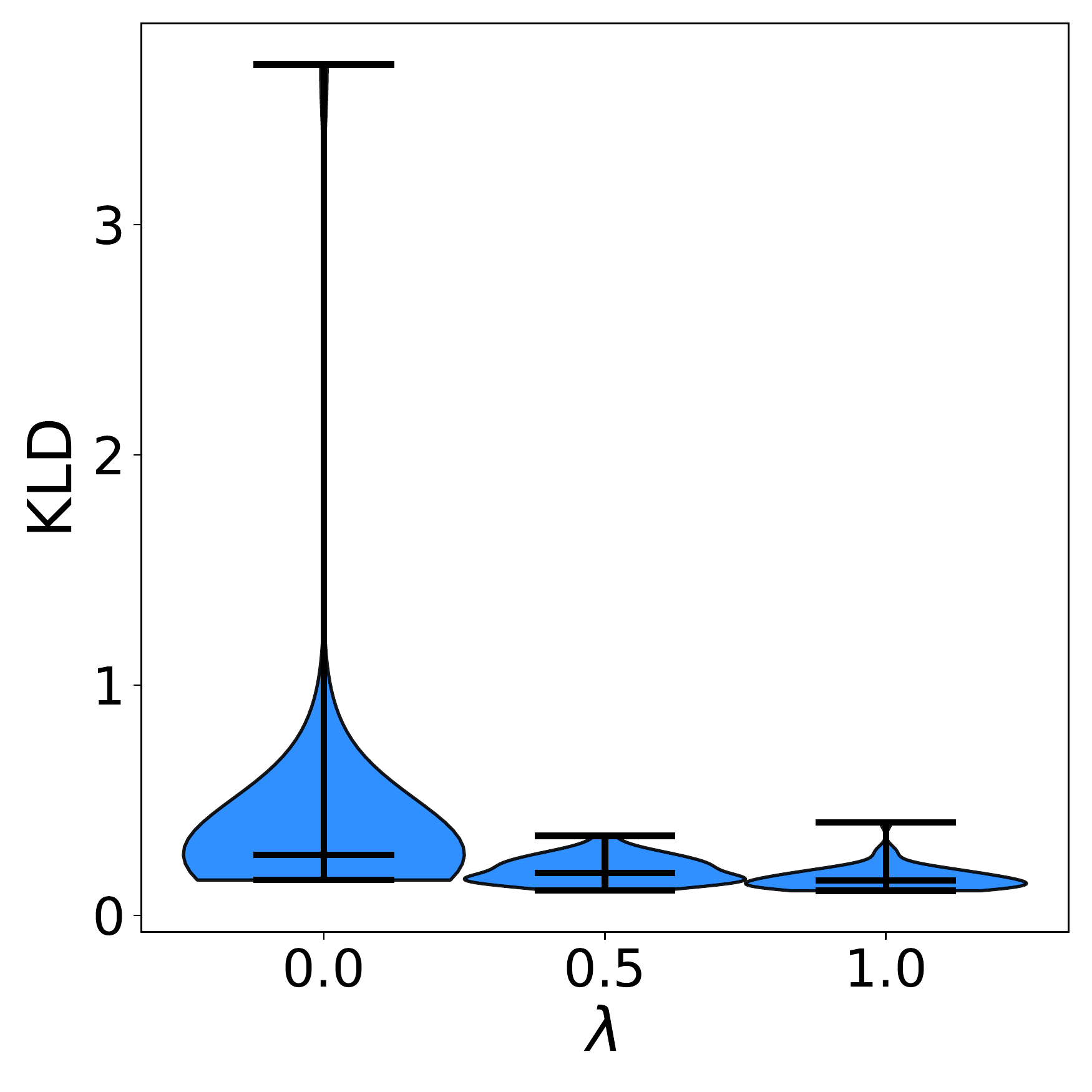}
            \put (10, 2) {\textbf{(e)}}
        \end{overpic}
    \end{subfigure}
    \begin{subfigure}{0.24\textwidth}
        \begin{overpic}[width=\textwidth]{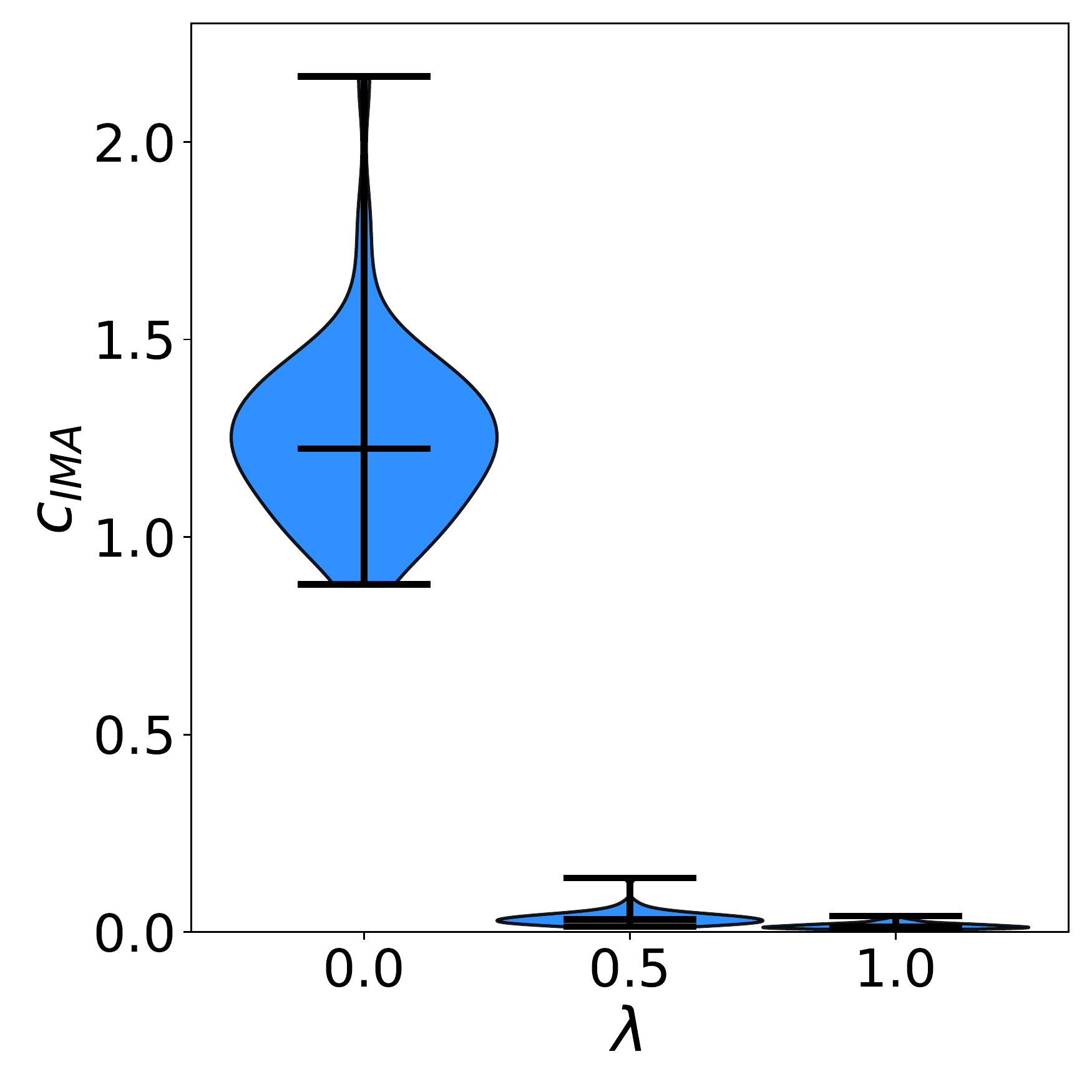}
            \put (10, 2) {\textbf{(f)}}
        \end{overpic}
    \end{subfigure}
    \begin{subfigure}{0.24\textwidth}
        \begin{overpic}[width=\textwidth]{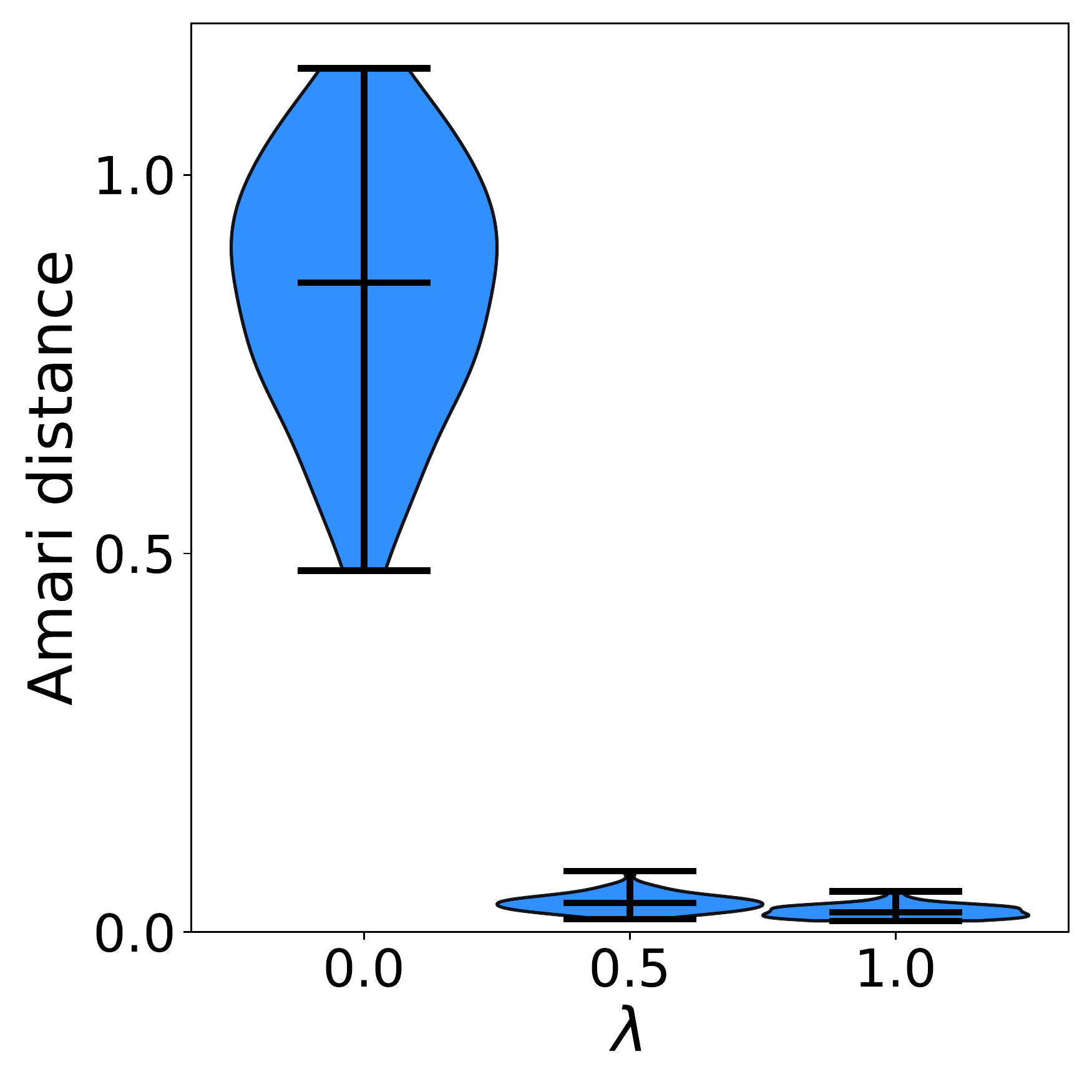}
            \put (10, 2) {\textbf{(g)}}
        \end{overpic}
    \end{subfigure}
    \begin{subfigure}{0.24\textwidth}
        \begin{overpic}[width=\textwidth]{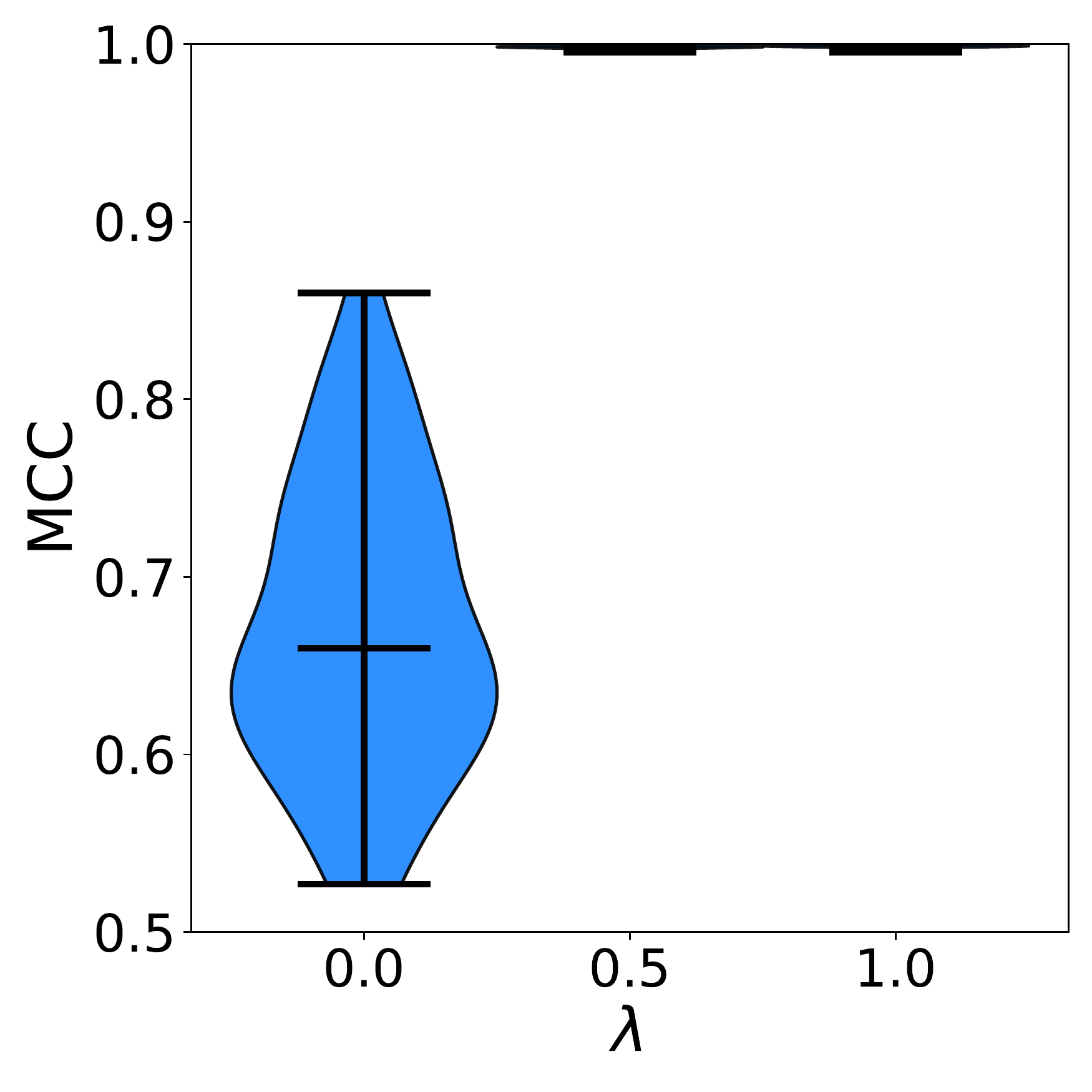}
            \put (10, 2) {\textbf{(h)}}
        \end{overpic}
    \end{subfigure}
    \caption{BSS via $C_\IMA$-regularised MLE for $n=5$ dimensions with $\lambda\in \{0.0,0.5,1.0\}$. The true mixing function is a randomly generated M\"obius transformation, nonlinear (with $\epsilon=2$) in \textbf{(a)}--\textbf{(d)} and linear (with $\epsilon=0$) transformation for \textbf{(e)}--\textbf{(h)}. For each type of transformation and $\lambda$, seeded runs are done. \textbf{(a)}, \textbf{(e)} KL-divergence between ground truth likelihood and learnt model; \textbf{(b)}, \textbf{(f)} $C_\IMA$ of the learnt models; \textbf{(c)}, \textbf{(g)} nonlinear Amari distance given true mixing and learnt unmixing; \textbf{(d)}, \textbf{(h)} MCC between true and reconstructed sources.}
    \label{fig:cima_obj_5d}
\end{figure}

\begin{figure}%
    \centering
    \begin{subfigure}{0.24\textwidth}
        \begin{overpic}[width=\textwidth]{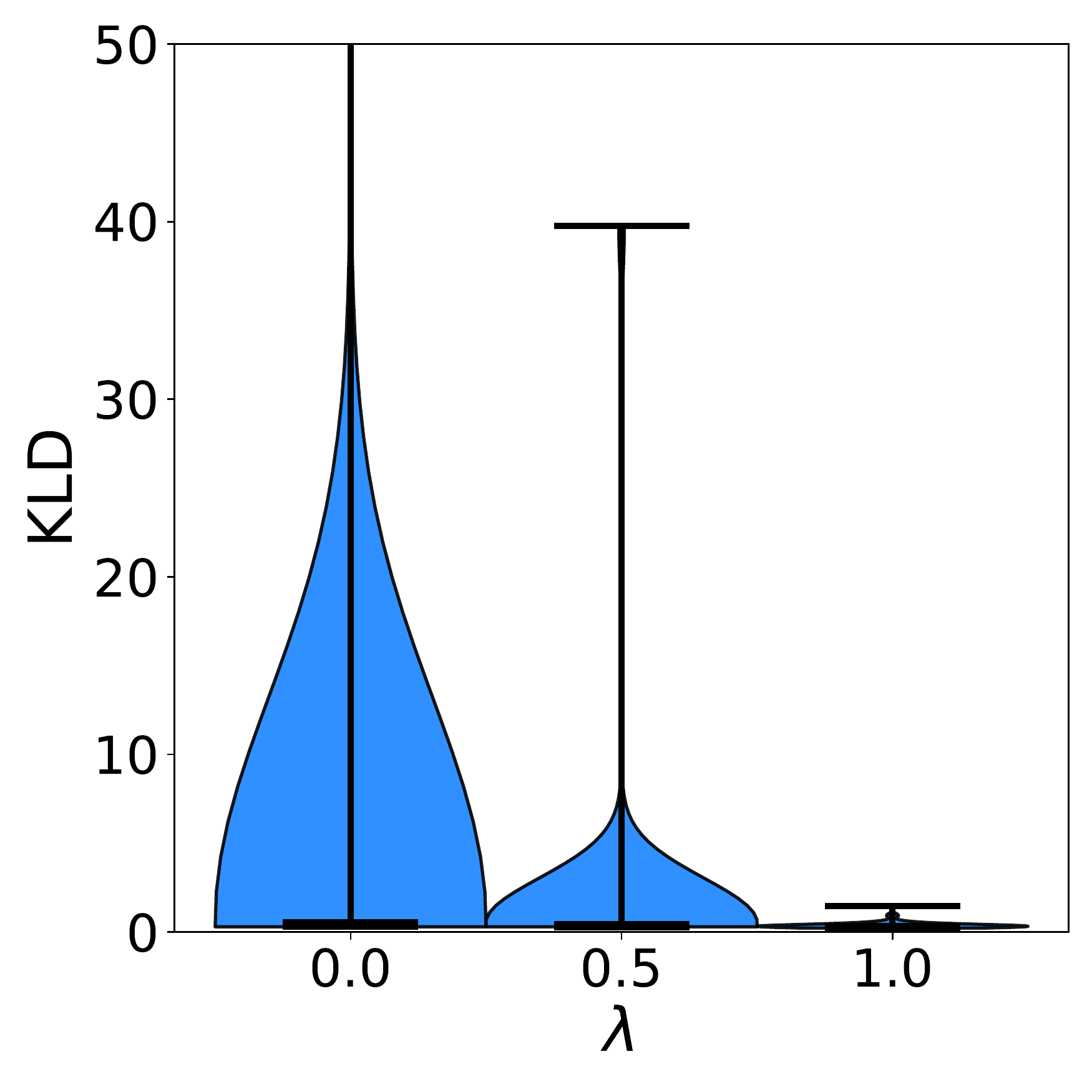}
            \put (10, 2) {\textbf{(a)}}
        \end{overpic}
    \end{subfigure}
    \begin{subfigure}{0.24\textwidth}
        \begin{overpic}[width=\textwidth]{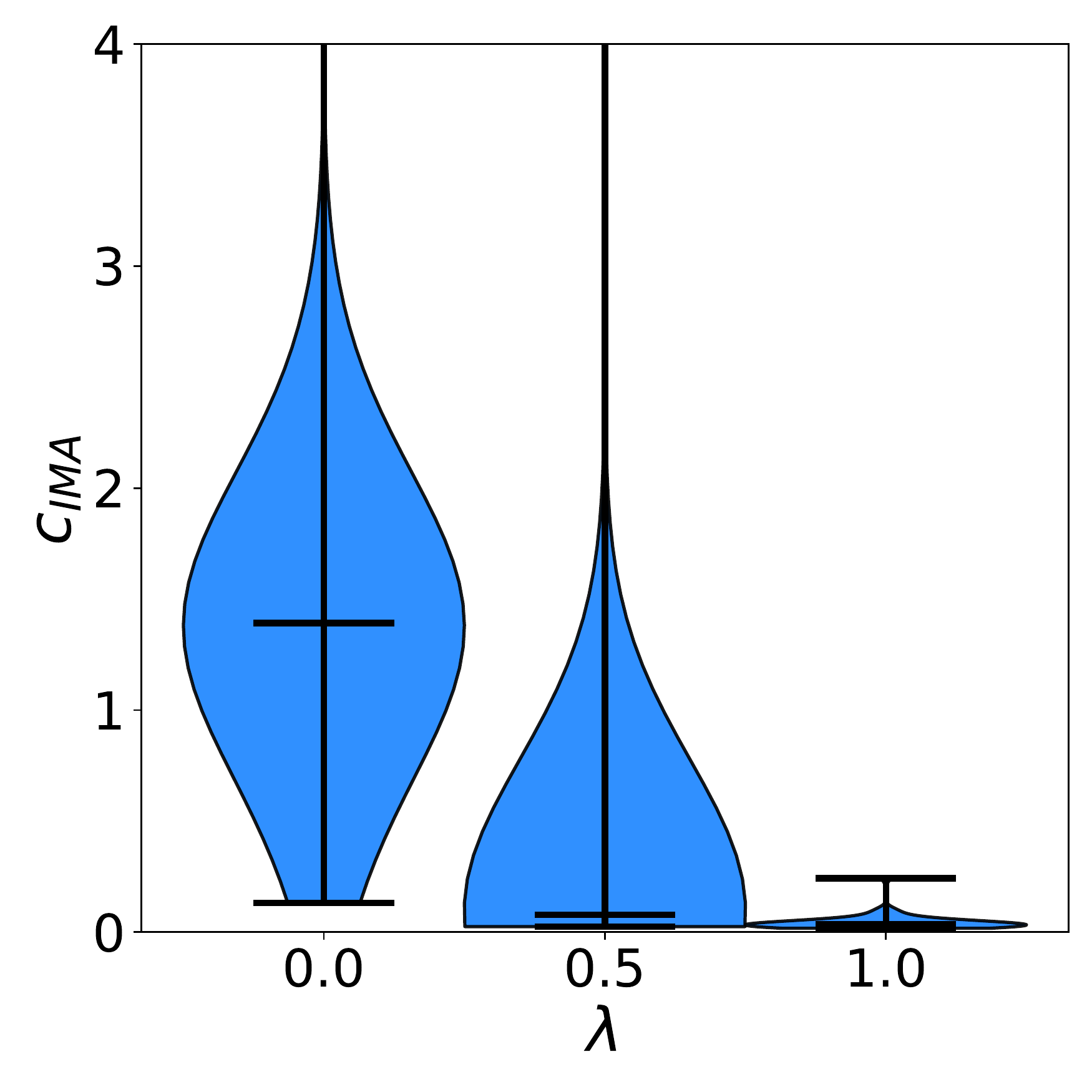}
            \put (10, 2) {\textbf{(b)}}
        \end{overpic}
    \end{subfigure}
    \begin{subfigure}{0.24\textwidth}
        \begin{overpic}[width=\textwidth]{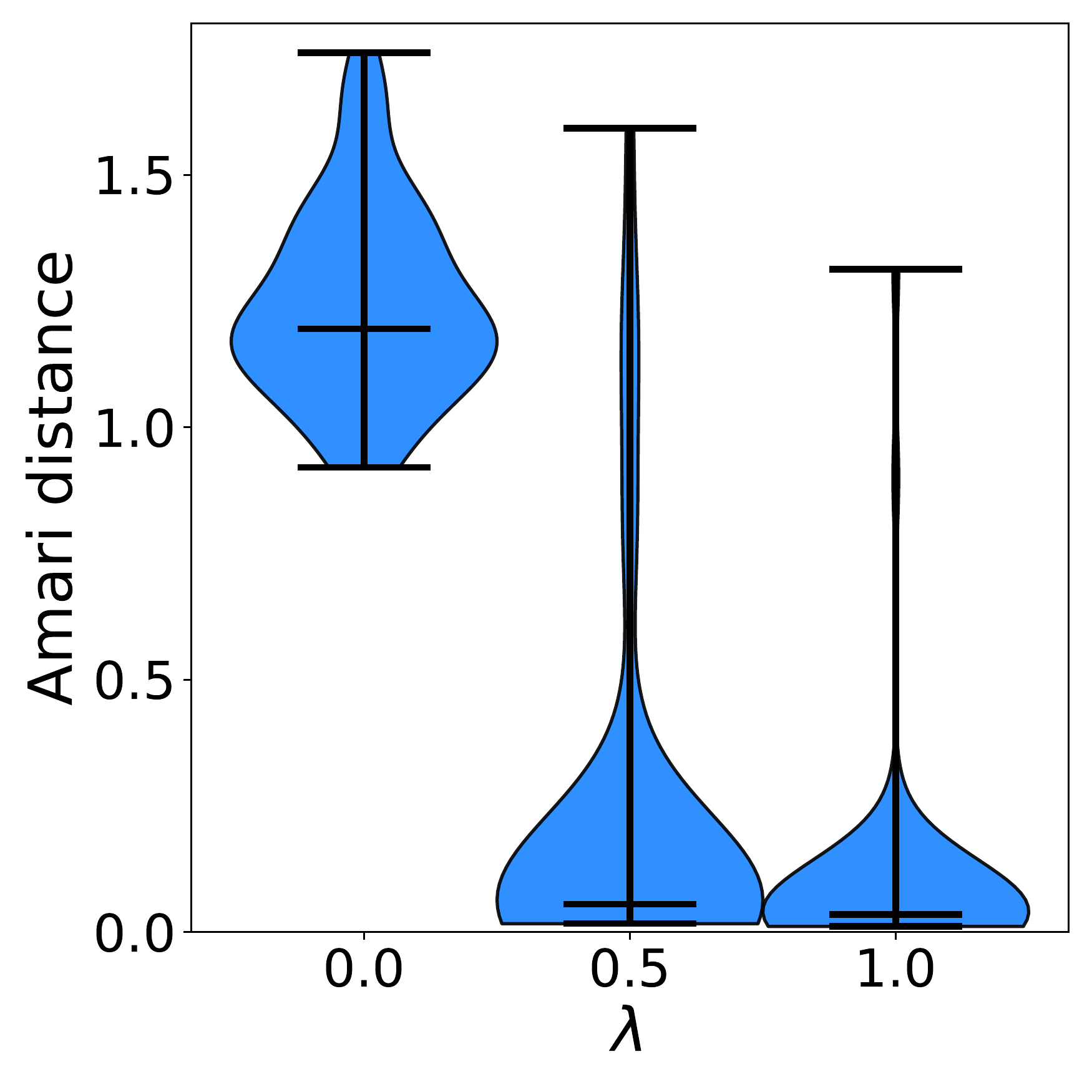}
            \put (10, 2) {\textbf{(c)}}
        \end{overpic}
    \end{subfigure}
    \begin{subfigure}{0.24\textwidth}
        \begin{overpic}[width=\textwidth]{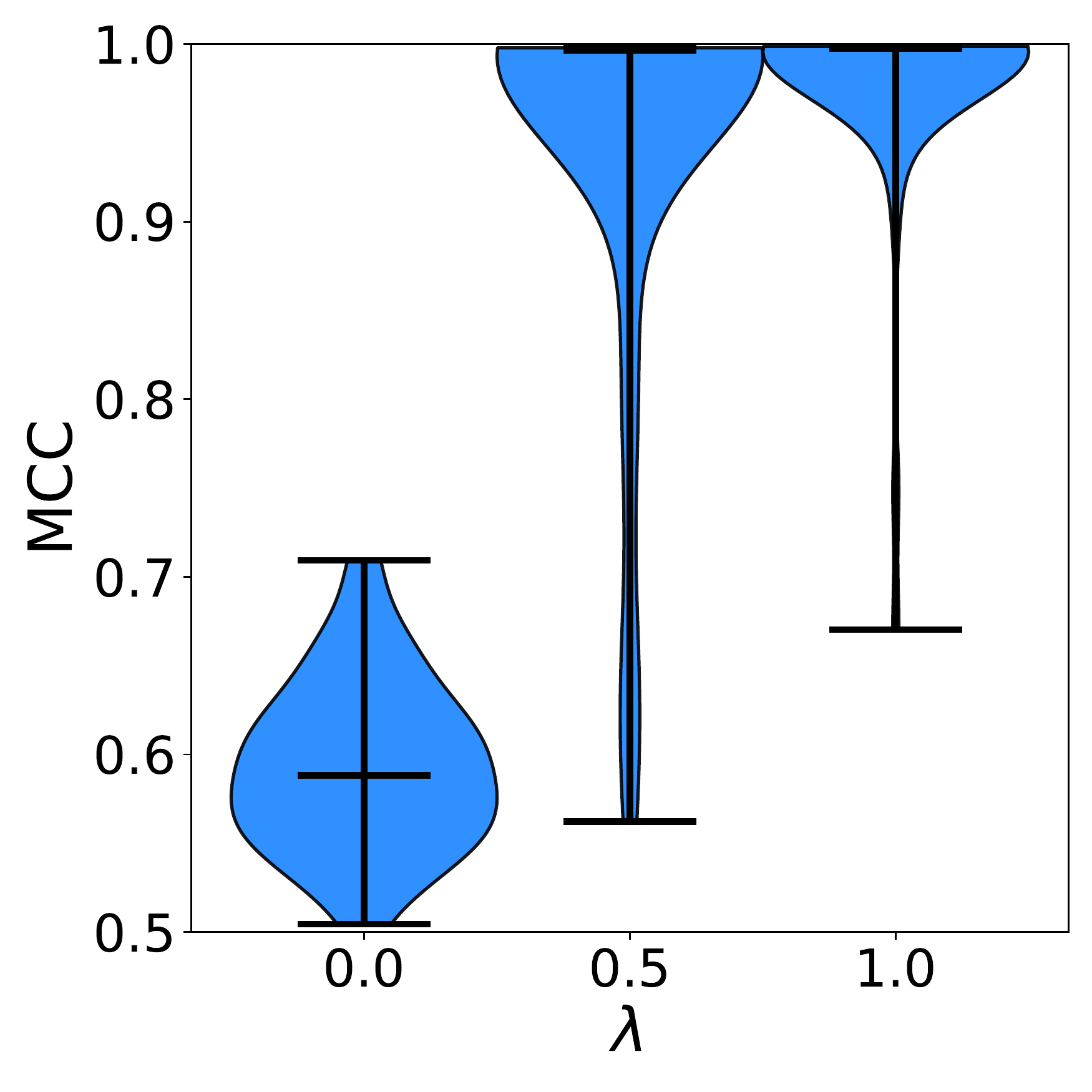}
            \put (10, 2) {\textbf{(d)}}
        \end{overpic}
    \end{subfigure}
    \caption{BSS via $C_\IMA$-regularised MLE for $n=7$ dimensions with $\lambda\in \{0.0,0.5,1.0\}$. The true mixing function is a randomly generated M\"obius transformation (with $\epsilon=2$). For each $\lambda$, seeded runs are done. \textbf{(a)} KL-divergence between ground truth likelihood and learnt model; \textbf{(b)} $C_\IMA$ of the learnt models; \textbf{(c)} nonlinear Amari distance given true mixing and learnt unmixing; \textbf{(d)} MCC between true and reconstructed sources.}
    \label{fig:cima_obj_7d}
\end{figure}

\renewcommand\folder{plots_perc_unif}
\begin{figure}[]

    \vspace{-1em}
    \centering
        \begin{minipage}{\minipagewidth \textwidth}
            \begin{subfigure}{1.0\textwidth}
            \centering
            \includegraphics[height=\height cm, keepaspectratio]{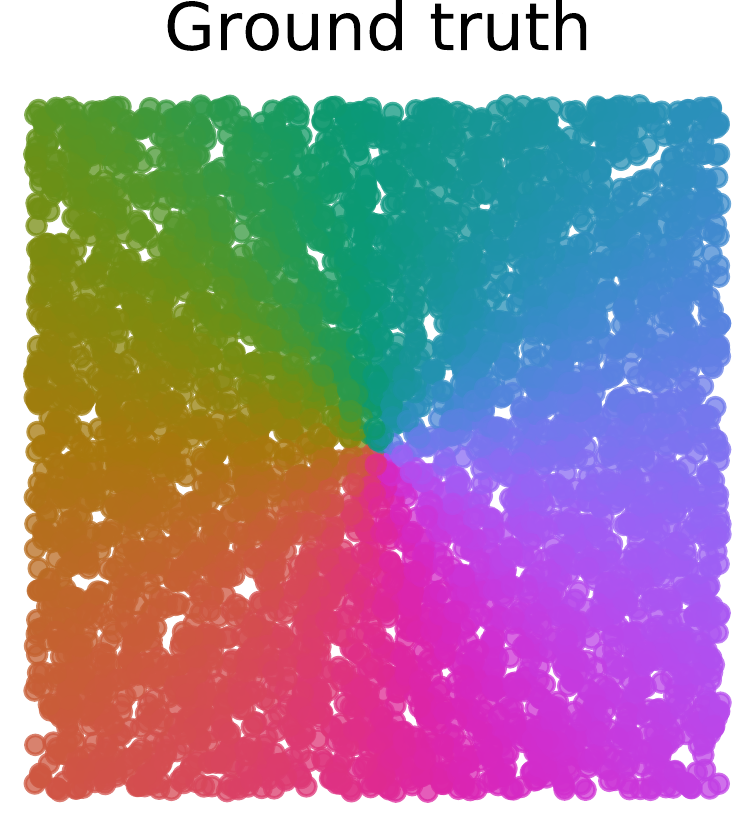}
            \end{subfigure}
        \end{minipage}%
        \hspace{\gap em}
        \begin{minipage}{\minipagewidth \textwidth}
            \begin{subfigure}{1.0\textwidth}
            \centering
            \includegraphics[height=\height cm, keepaspectratio]{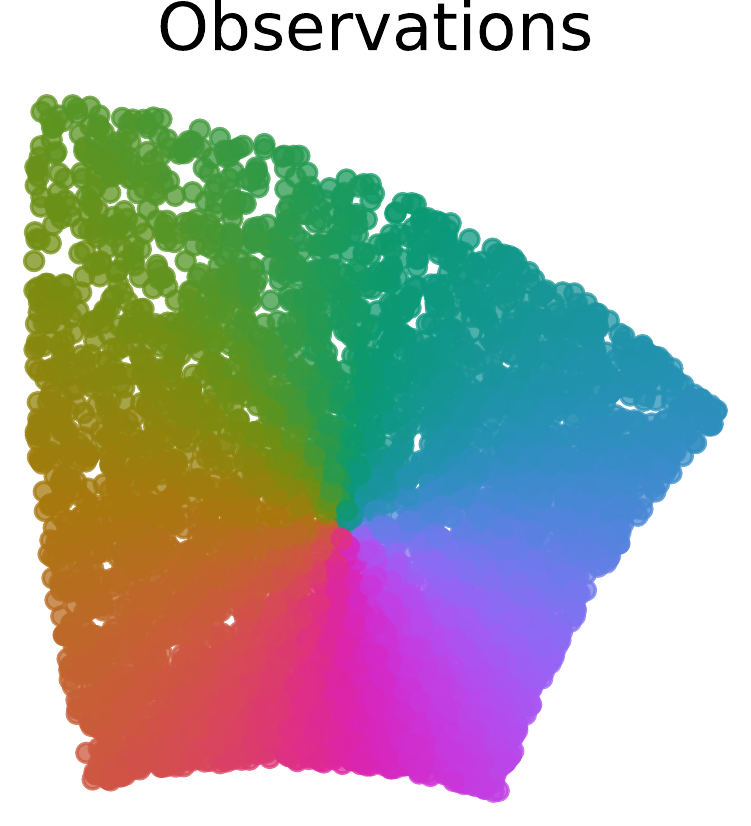}
            \end{subfigure}
        \end{minipage}%
        \hspace{\gap em}
        \begin{minipage}{\minipagewidth \textwidth}
            \begin{subfigure}{1.0\textwidth}
            \centering
            \includegraphics[height=\height cm, keepaspectratio]{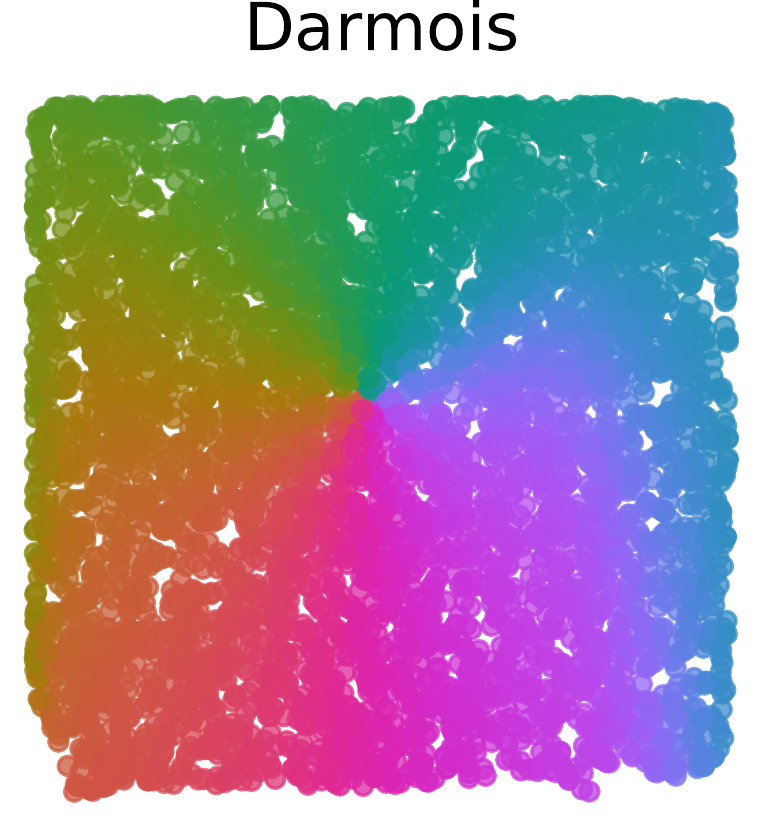}
            \end{subfigure}
        \end{minipage}%
        \hspace{\gap em}
        \begin{minipage}{\minipagewidth \textwidth}
            \begin{subfigure}{1.0\textwidth}
            \centering
            \includegraphics[height=\height cm, keepaspectratio]{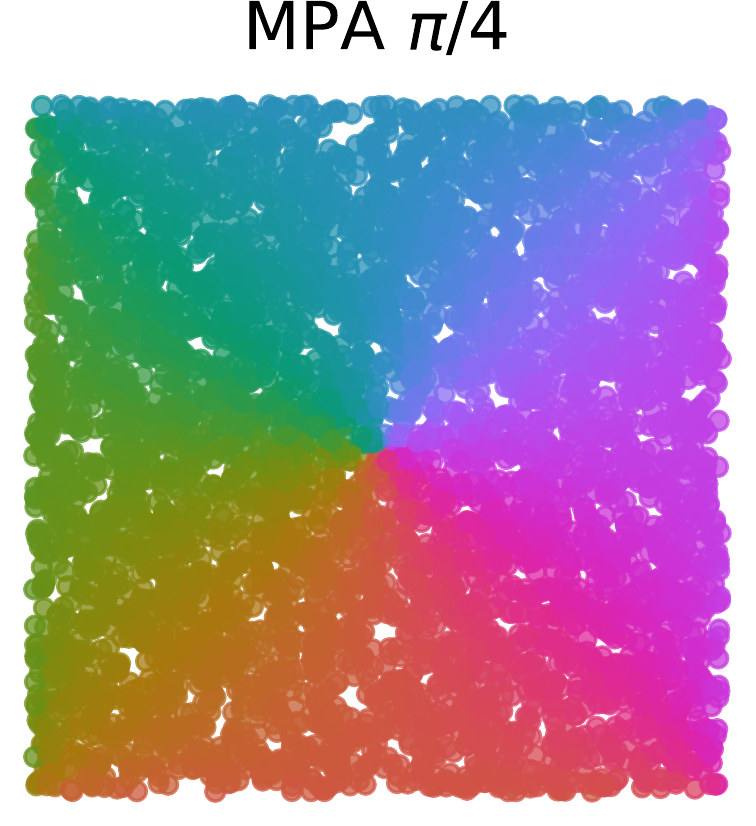}
            \end{subfigure}
        \end{minipage}%
        \hspace{\gap em}
        \begin{minipage}{\minipagewidth \textwidth}
            \begin{subfigure}{1.0\textwidth}
            \centering
            \includegraphics[height=\height cm, keepaspectratio]{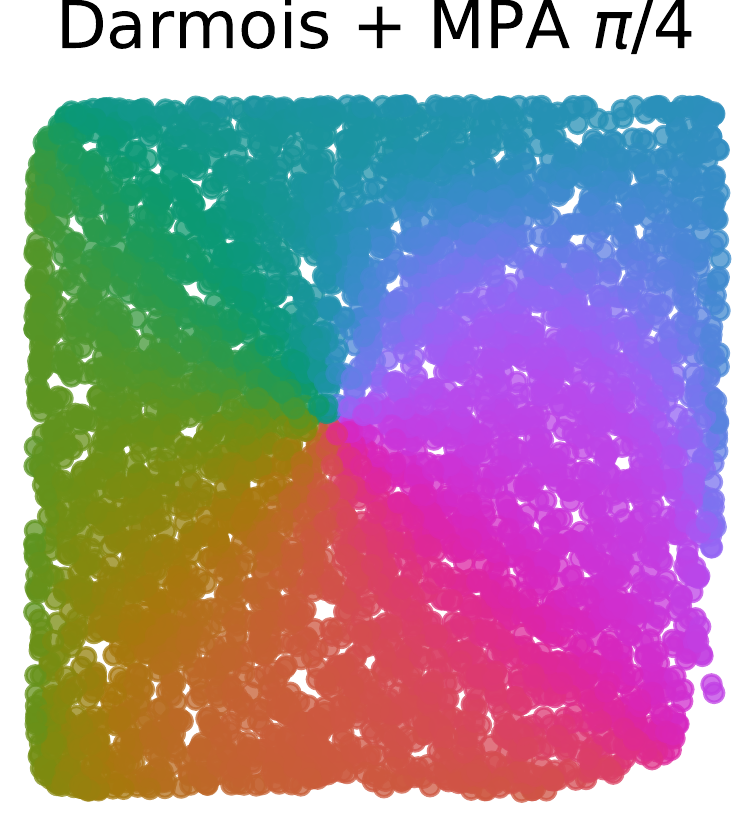}
            \end{subfigure}
        \end{minipage}%
        \hspace{\gap em}
        \begin{minipage}{\minipagewidth \textwidth}
            \begin{subfigure}{1.0\textwidth}
            \centering
            \includegraphics[height=\height cm, keepaspectratio]{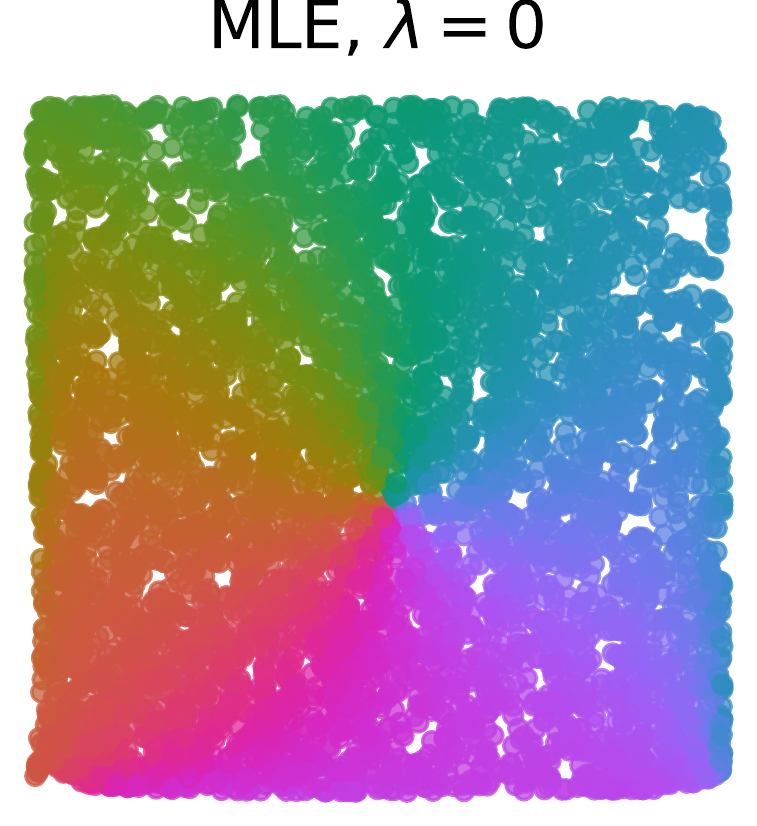}
            \end{subfigure}
        \end{minipage}
        \hspace{\gap em}
        \begin{minipage}{\minipagewidth \textwidth}
            \begin{subfigure}{1.0\textwidth}
            \centering
            \includegraphics[height=\height cm, keepaspectratio]{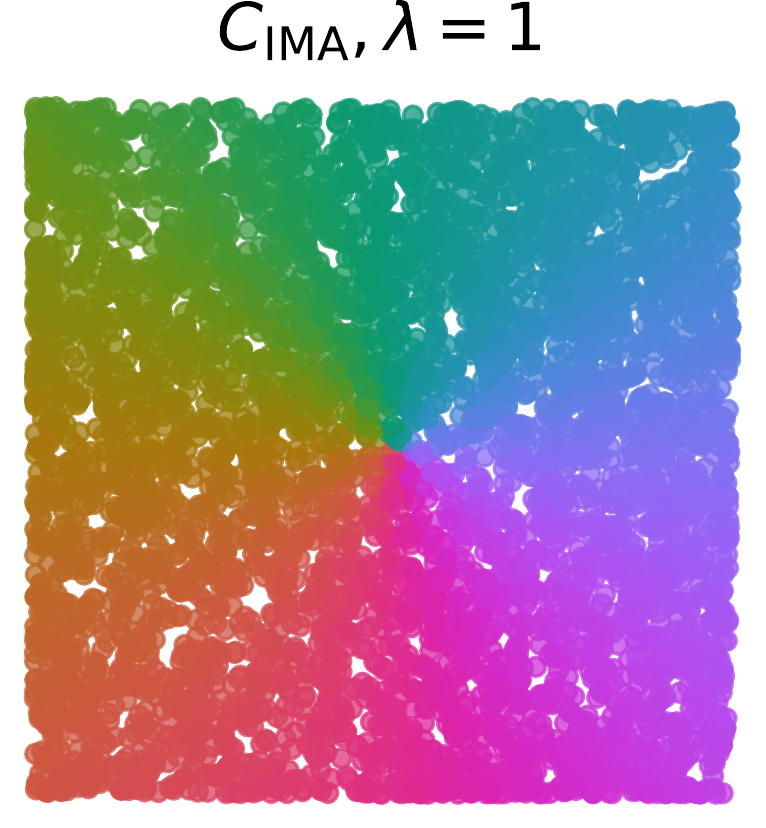}
            \end{subfigure}
        \end{minipage}
    \vspace{-0.5em}
    \caption{\small Visual comparison of different nonlinear ICA solutions for $n=2$: \textit{(left to right)} true sources; observed mixtures;  Darmois solution; true unmixing, composed with the measure preserving automorphism (MPA) from~\eqref{eq:measure_preserving_automorphism_Gaussian} (with rotation by~$\nicefrac{\pi}{4}$); Darmois solution composed with the same MPA; maximum likelihood~($\lambda=0$); %
    and $C_\IMA$-regularised approach~($\lambda=1$).
    }
    \label{fig:results_percuni}
    \vspace{-1.25em}
\end{figure}

When computing the $C_\IMA$ of the Darmois solutions of randomly generated functions, we restricted ourselves to Möbius transformations, i.e. conformal maps. However, there are also nonconformal maps satisfying $C_\IMA = 0$, e.g. the transformation of Cartesian to Polar coordinates, see~\Cref{app:examples}. To test whether the $C_\IMA$ of the Darmois solutions is actually bigger than $0$, we gener

\subsection{Evaluation}
\label{sec:app_eval}

\paragraph{Mean correlation coefficient.} %

To evaluate the performance of our method, we compute the mean correlation coefficient (MCC) between the original sources and the corresponding latents, see for example~\cite{khemakhem2020variational}. We first
compute the matrix of correlation coefficients between all pairs of ground truth and reconstructed sources. Then, we solve a linear
sum assignment problem (e.g. using the Hungarian algorithm) to match each reconstructed source to the ground truth one which has the highest correlation with it.
The MCC matrix contains the Spearman rank-order correlations between the ground truth and reconstructed sources, a measure which is blind to nonlinear invertible reparametrisations of the sources. 

\paragraph{Nonlinear Amari metric.}
While the MCC metric evaluates BSS by comparing ground truth and reconstructed sources, we propose an additional evaluation directly based on comparing the (Jacobians of the) true mixing and the learned unmixing. We take inspiration from an evaluation metric used in the context of linear ICA,
the Amari distance~\cite{amari1996new}: Given a learned unmixing $\Wb$ and the true mixing $\Ab$, and defining the matrix $\Rb = \Ab\Wb$, the Amari distance is defined as
\begin{equation}
d^{\text{Amari}}(\Rb) = \sum_{i=1}^{n}\left(\sum_{j=1}^{n} \frac{[\Rb]_{i j}^{2}}{\max _{l} [\Rb]_{i l}^{2}}-1\right)+\sum_{i=1}^{n}\left(\sum_{j=1}^{n} \frac{[\Rb]_{j i}^{2}}{\max _{l} [\Rb]_{l j}^{2}}-1\right)\,,
\label{eq:amari_distance}
\end{equation}
and is greater than or equal to zero, canceling if and only if $\Rb$ is a scale and permutation matrix, that is when the learned unmixing is matching the unresolvable ambiguities of linear ICA.

We extend this idea to the nonlinear setting:
Given a true mixing $\fb$ and a learned unmixing $\gb$, we define our nonlinear Amari distance as
\begin{equation}
d^{\text{n-Amari}}(\gb, \fb) = \EE_{\xb \sim p_\xb} \left[ d^{\text{Amari}}\left(\Jb_{\gb}(\xb) \Jb_{\fb}(\fb^{-1}(\xb))  \right) \right]\,.
\end{equation}
Then, according to the definition of Amari distance~\eqref{eq:amari_distance}, if the smooth function $\gb \circ \fb$ is a permutation composed with a scalar function, thus precisely matching the BSS equivalence class defined in~\Cref{def:bss_identifiability}, this would result in its Jacobian (that is, the product of the Jacobians $\Jb_{\gb}(\xb) \Jb_{\fb}(\fb^{-1}(\xb)) $) equalling the product of a diagonal matrix and a permutation matrix at every point $\xb$:
the quantity $d^{\text{n-Amari}}(\gb, \fb)$ would therefore be equal to zero. %

This metric can be of independent interest and potentially useful in contexts where the reconstructed sources might be a noisy version of the true ones, but the true unmixing is nevertheless identifiable. Our implementation is based on the one for the (linear) Amari distance provided in the code for~\cite{ablin2018faster}.

\paragraph{$C_\IMA$ of Darmois solutions for nonconformal maps satisfying the IMA principle.}

When computing the $C_\IMA$ of the Darmois solutions of randomly generated functions, we restricted ourselves to M\"obius transformations which are conformal maps. However, there are also nonconformal maps satisfying $C_\IMA = 0$, e.g., the transformation from polar to Cartesian coordinates with $n=2$, see~\Cref{app:examples}. To test whether the $C_\IMA$ of the Darmois solutions is actually bigger than $0$, we generate random radial transformations by imposing a random scale and shift before applying the radial transformation, compute the Darmois solution as we have done in~\Cref{sec:experiment1_evaluation}, and calculate its $C_\IMA$ on the test set. We did 50 runs and the results are shown in \cref{fig:darmois_rad_hist}.

\begin{figure}[h]
    \centering
    \includegraphics[width=0.4\textwidth]{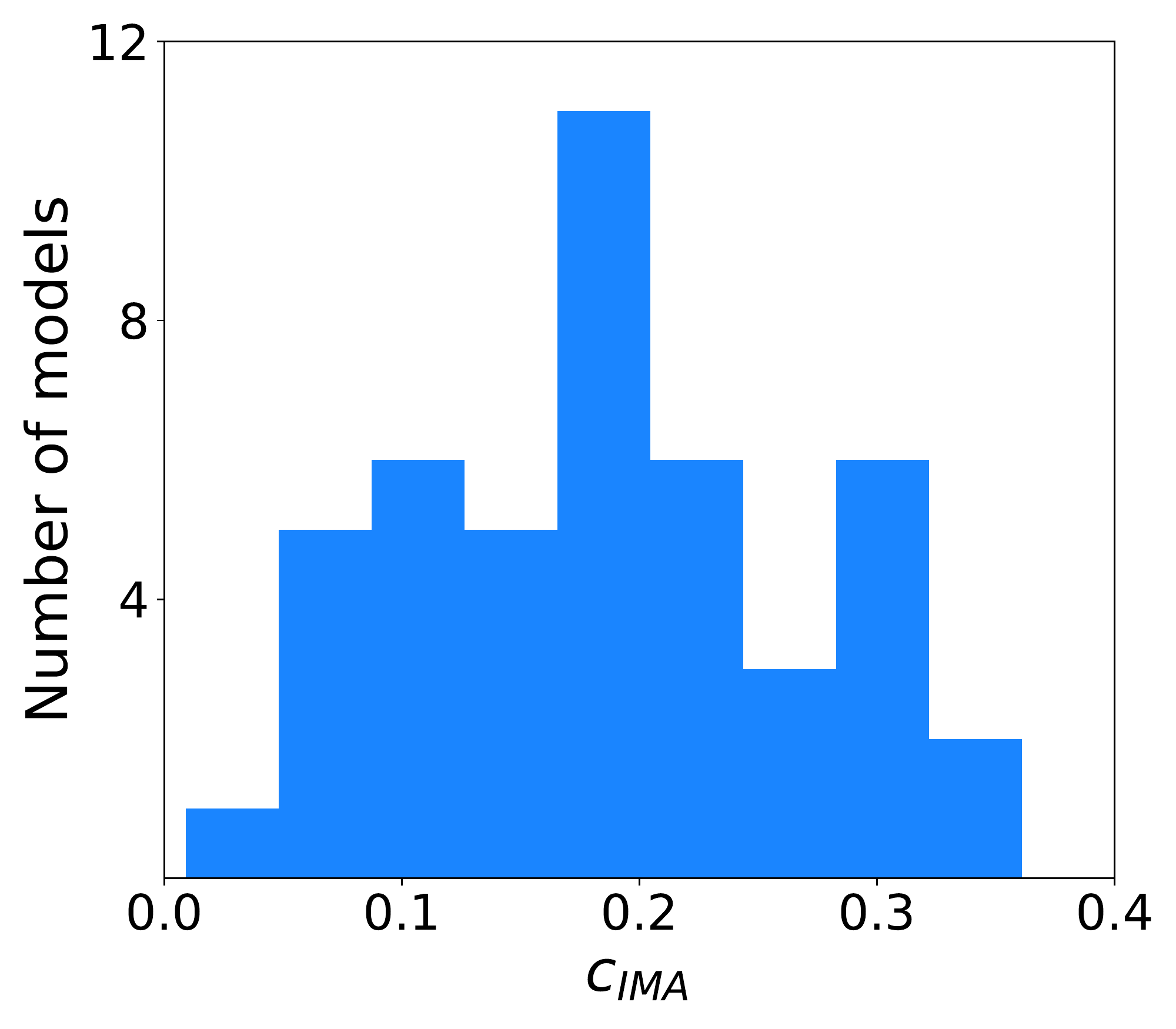}
    \caption{Histogram of the $C_\IMA$ values of the Darmois solutions of 50 randomly generated radial transformations.}
    \label{fig:darmois_rad_hist}
\end{figure}

Similar to \cref{fig:results1} \textbf{(a)} we can clearly see that all $C_\IMA$ values of the final models are larger than 0, with the smallest value being $0.01$. This confirms the result we have already shown theoretically.

\paragraph{Additional plots for~\Cref{sec:experiment2_learning}.}

We show additional plots for the quantitative experiments involving training with the objective described in~\eqref{eq:objective_ml_cima}, see ~\Cref{fig:cima_obj_2d}, ~\Cref{fig:cima_obj_5d} and~\Cref{fig:cima_obj_7d}.

For $\epsilon =0 $ (that is, ground truth mixing linear), there appears to be an almost perfect recovery of the ground truth sources (resp. unmixing function) for  $\lambda \in \{0.5, 1.0\}$, as can be seen by the high (resp. low) values of the MCC  (resp. nonlinear Amari distance) evaluations ; this is in stark contrast with the distribution of the MCC (resp. nonlinear Amari distance) values for models trained with $\lambda=0$, which are typically much higher (resp. lower), indicating that the learned solutions do not achieve blind source separation (see $n=2$, ~\Cref{fig:cima_obj_2d} \textbf{ (g), (h)}; $n=5$,~\Cref{fig:cima_obj_5d} \textbf{ (g), (h)}). All models achieve a comparably good fit, reflected in the KL-divergence values ($n=2$, ~\Cref{fig:cima_obj_2d} \textbf{ (e)}; $n=5$,~\Cref{fig:cima_obj_5d} \textbf{ (e)}). %

The trend is confirmed when the true mixing is nonlinear ($\epsilon = 2$), with slightly lower (resp. higher) values achieved with $C_\IMA$ regularisation for the MCC (resp. nonlinear Amari) metrics; this possibly due to the increased difficulty of fitting observations generated by a nonlinear mixing, as can be seen from the higher values of the KL-divergence ($n=2$, ~\Cref{fig:cima_obj_2d} \textbf{(a)}; $n=5$, ~\Cref{fig:cima_obj_5d} \textbf{(a)}; $n=7$, ~\Cref{fig:cima_obj_7d} \textbf{(a)});\footnote{The distribution of the KL values contains outliers, and seemingly more strongly for lower values of $\lambda$.} still, the beneficial effect of  $\lambda \in \{0.5, 1.0 \}$ with respect to models trained with $\lambda=0$ is clear, and is apparently stronger for $\lambda=1.0$  and with higher data dimensionality $n$ ($n=2$, ~\Cref{fig:cima_obj_2d} \textbf{ (c), (d)}; $n=5$,~\Cref{fig:cima_obj_5d} \textbf{ (c), (d)}; $n=7$,~\Cref{fig:cima_obj_7d} \textbf{ (c), (d)}).

We additionally plot the $C_\IMA$ values for the all trained models, for all values of $\lambda$. It can be seen that solutions found by unregularised maximum likelihood estimation typically learn functions with relatively high values of $C_\IMA$, while as expected the regularised version achieves low values ($n=2$, ~\Cref{fig:cima_obj_2d} \textbf{(b), (f)}; $n=5$, ~\Cref{fig:cima_obj_5d} \textbf{(b), (f)}; $n=7$, ~\Cref{fig:cima_obj_7d} \textbf{(b)}).

Finally, in figure \ref{fig:results_percuni}, we report the same plot as in \ref{fig:results1}, top row, but with a perceptually uniform colormap.

\paragraph{Comparison to FastICA.} 

We compared the performance of our proposed regularised maximum likelihood procedure to a state of the art method for linear ICA, FastICA~\cite{hyvarinen1999fast}, 
in the implementation from the Scikit-learn package~\cite{pedregosa2011scikit}, over $50$ repetitions. Our experiments show that our regularised method ($\lambda=0.5$, and particularly $\lambda=1.0$; $\lambda=0.0$ provides the unregularised nonlinear baseline) is superior in learning the true unmixing and reconstructing the sources. This indicates that the linearity assumption of FastICA does not allow enough flexibility to solve blind source separation in our setting, whereas our criterion does (see~\cref{fig:fastica2},~\cref{fig:fastica5} and~\cref{fig:fastica7}).\footnote{the experimental setting and the plots for the normalising flow models correspond to those already shown in the paper, but here we modified the $y$-axis scale to facilitate the comparison of all methods}
While the spread in the distributions of MCC and Amari distance can be largely attributed to the brittleness of neural networks, the median values for the MCC (resp. nonlinear Amari distance) are consistently higher (resp. lower) for our regularised method than for FastICA. In contrast, the performance of FastICA is consistently better than the unregularised baseline.

\begin{figure}[h!]
\centering
\begin{subfigure}{.5\textwidth}
  \centering
  \includegraphics[width=\linewidth]{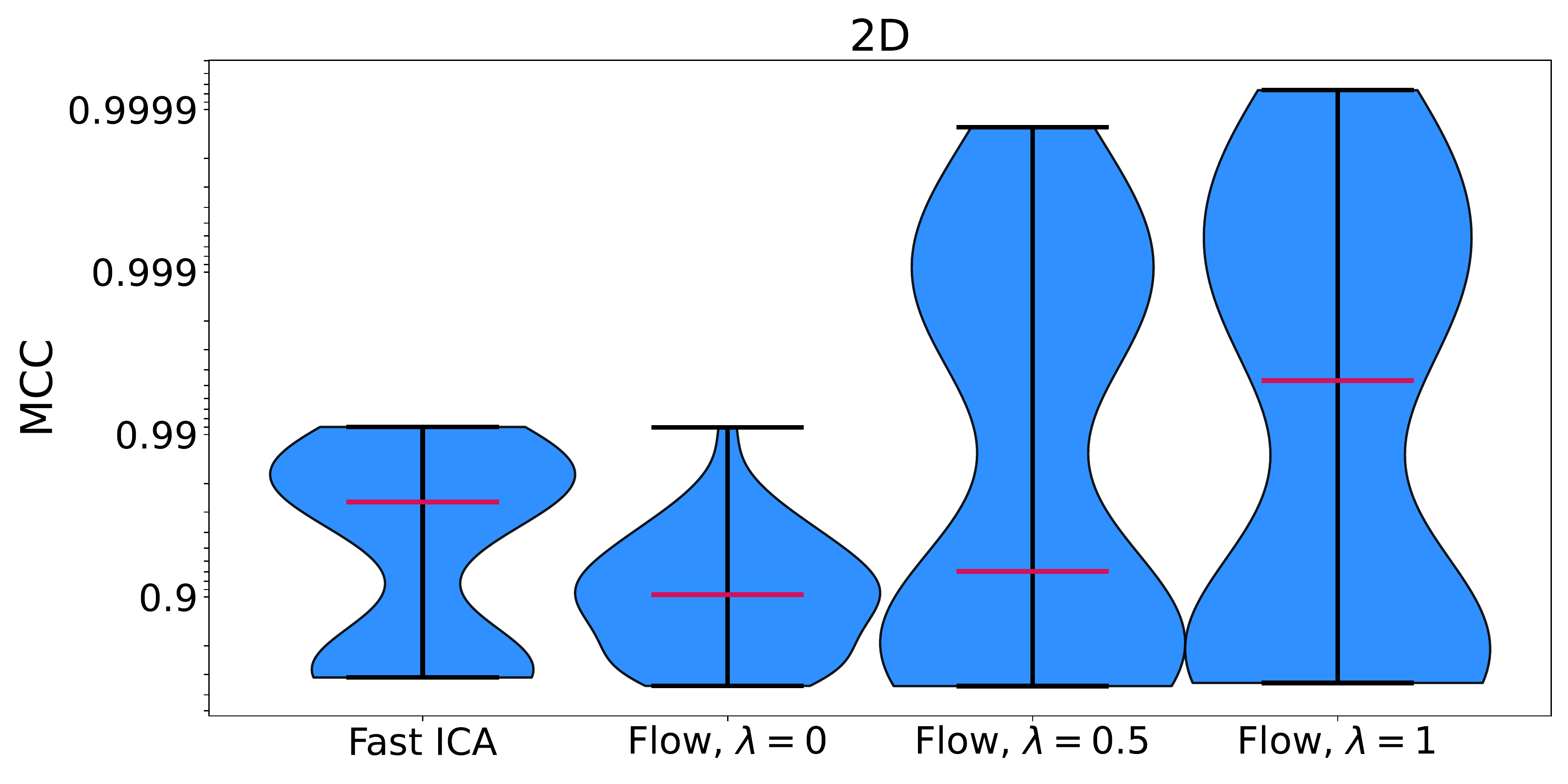}
  \caption{}
\end{subfigure}%
\begin{subfigure}{.5\textwidth}
  \centering
  \includegraphics[width=\linewidth]{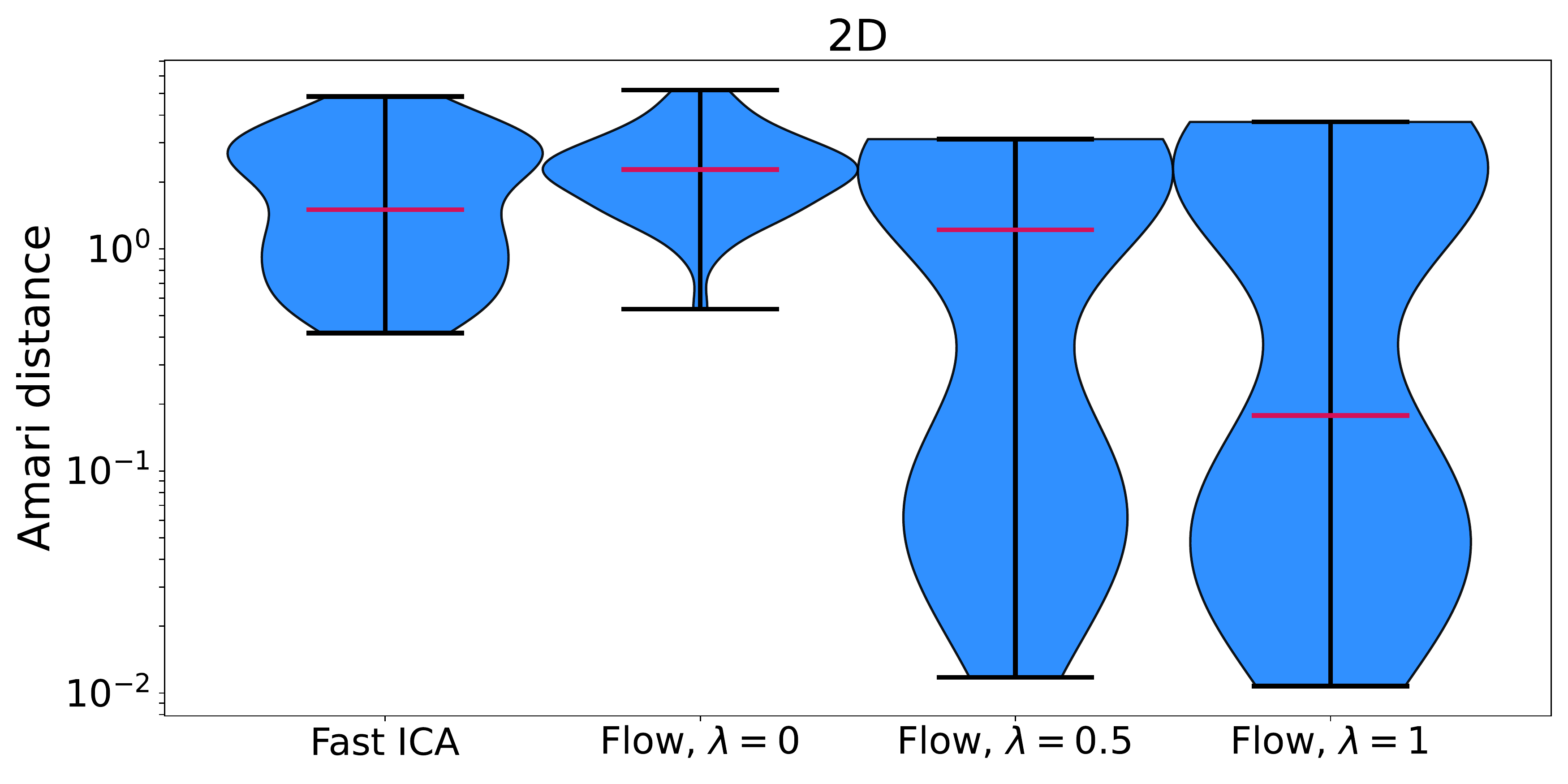}
  \caption{}
\end{subfigure}%
\caption{Comparison between FastICA and our normalising flow method with $\lambda \in \{ 0.0, 0.5, 1.0 \}$, $n=2$. (a) MCC; (b) Amari distance.}
\label{fig:fastica2}
\end{figure}

\begin{figure}[h!]
\centering
\begin{subfigure}{.5\textwidth}
  \centering
  \includegraphics[width=\linewidth]{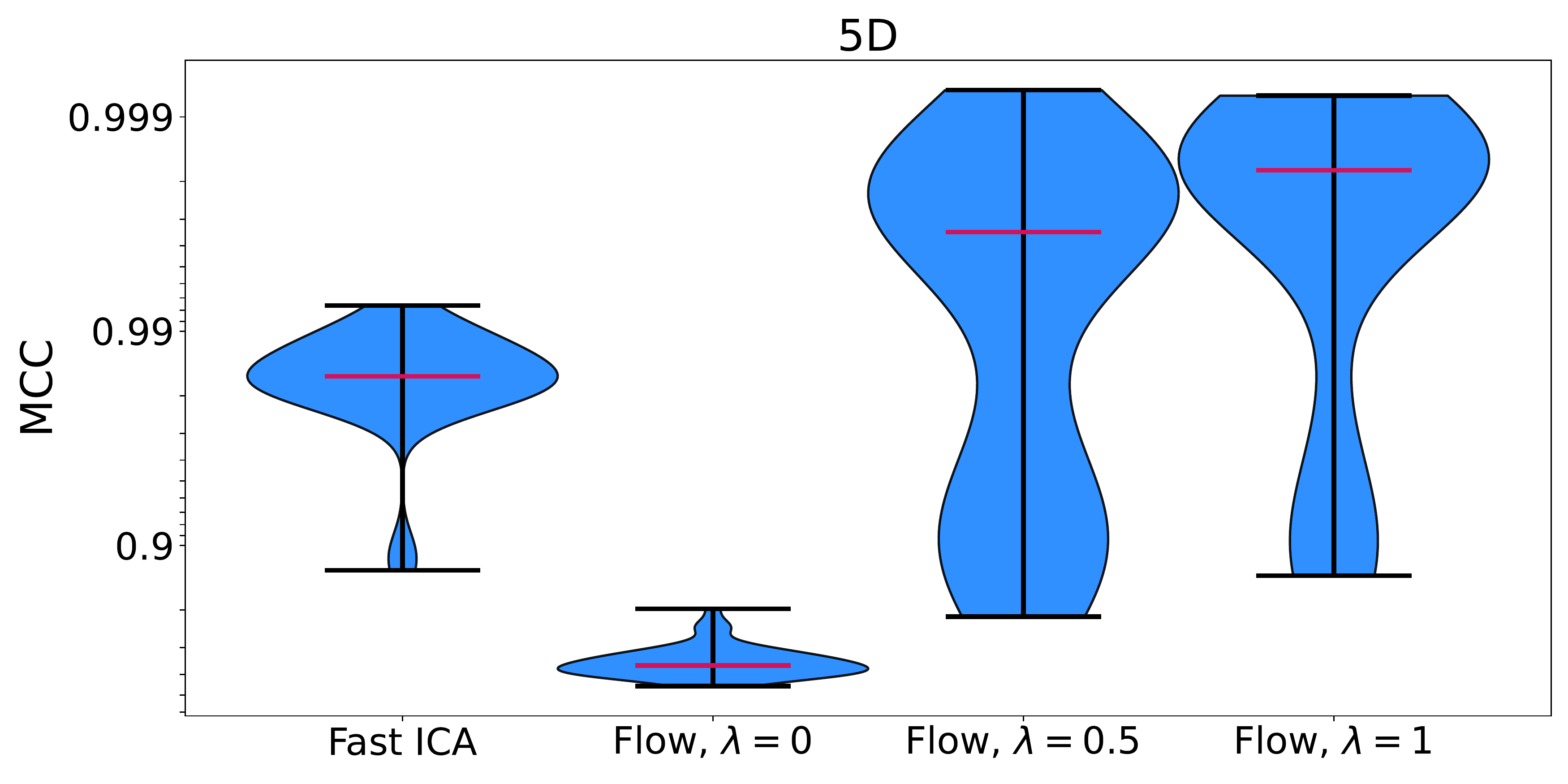}
  \caption{}
\end{subfigure}%
\begin{subfigure}{.5\textwidth}
  \centering
  \includegraphics[width=\linewidth]{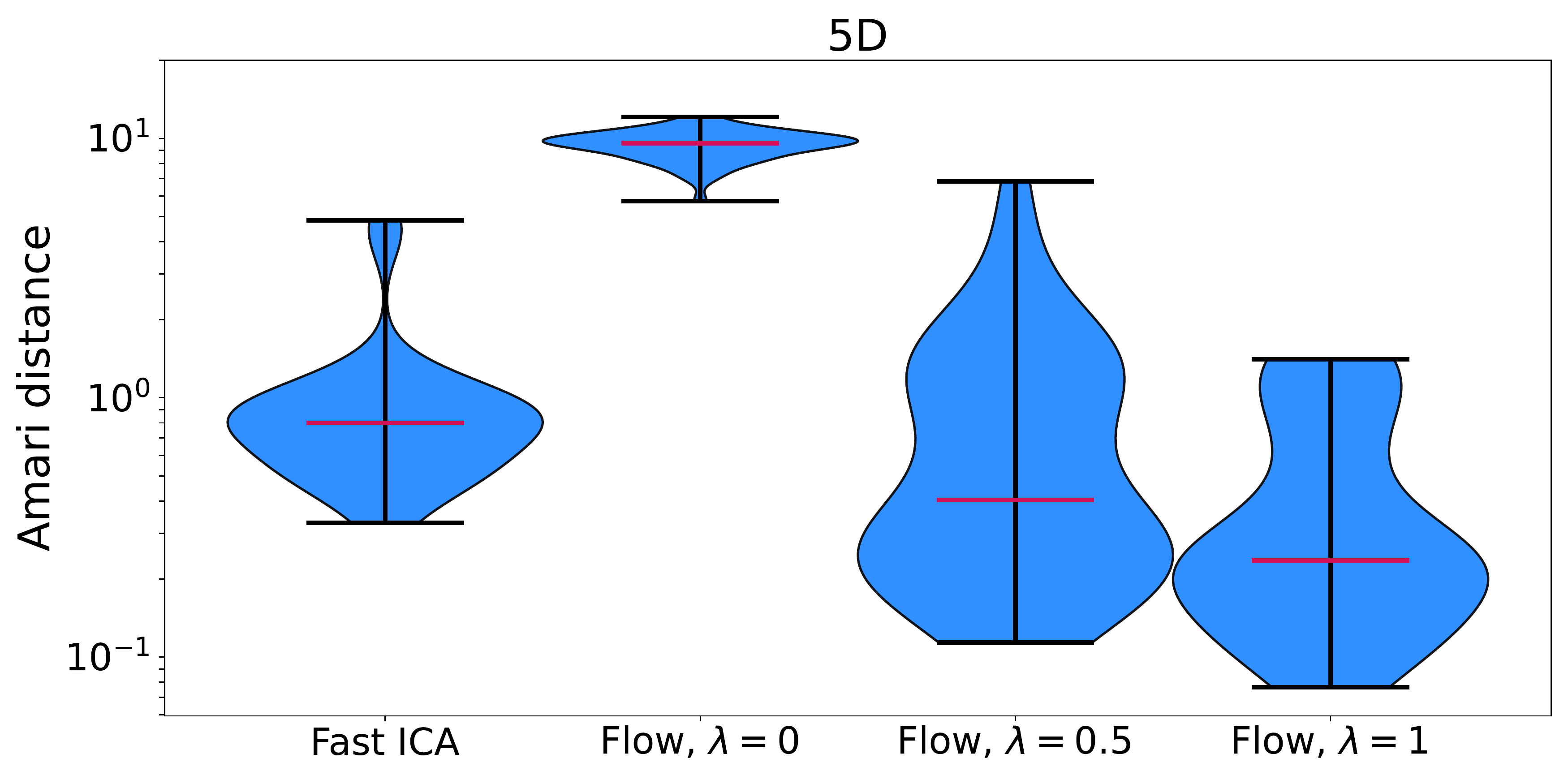}
  \caption{}
\end{subfigure}%
\caption{Comparison between FastICA and our normalising flow method with $\lambda \in \{ 0.0, 0.5, 1.0 \}$, $n=5$. (a) MCC; (b) Amari distance.}
\label{fig:fastica5}
\end{figure}

\begin{figure}[h!]
\centering
\begin{subfigure}{.5\textwidth}
  \centering
  \includegraphics[width=\linewidth]{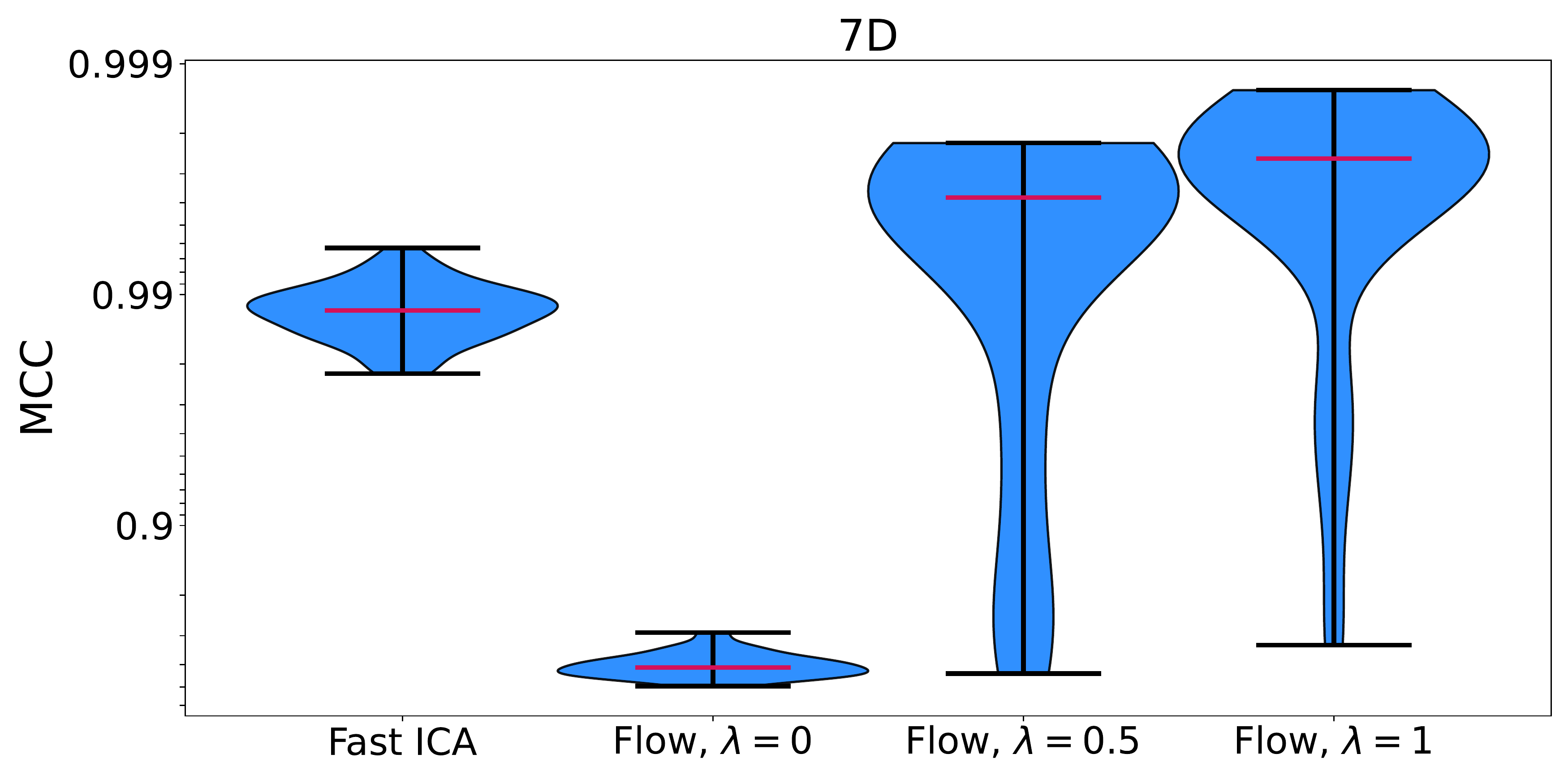}
  \caption{}
\end{subfigure}%
\begin{subfigure}{.5\textwidth}
  \centering
  \includegraphics[width=\linewidth]{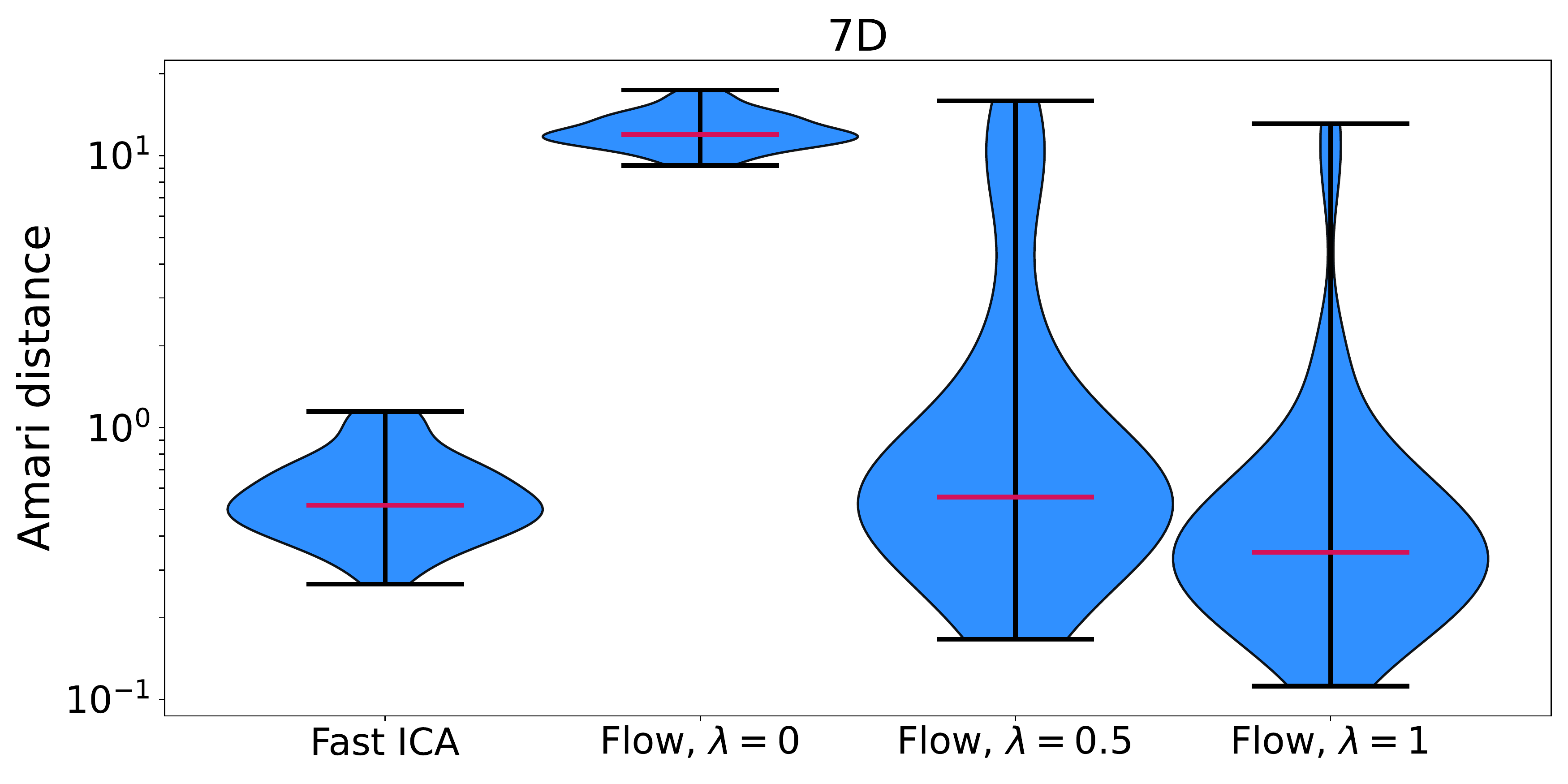}
  \caption{}
\end{subfigure}%
\caption{Comparison between FastICA and our normalising flow method with $\lambda \in \{ 0.0, 0.5, 1.0 \}$, $n=7$. (a) MCC; (b) Amari distance.}
\label{fig:fastica7}
\end{figure}

\paragraph{Details on resources used.} 
All models were trained on compute instances with 16 Intel Xeon E5-2698 CPUs and a Nvidia Geforce GTX980 GPU. The cluster we used has 204 thereof. Training the models took between 4 and 16 hours depending mainly on the dimensionality $n$ and number of samples in the dataset, and on the number of iterations used for training. Overall, we trained around 2000 models, amounting to roughly 18000 GPU hours.

%% file: appendix/F_background_conf_map.tex
\paragraph{Similarities.} A \textit{similarity} of a Euclidean space is a bijection $\fb$ from the space onto itself that multiplies all distances by the same positive real number $r$, so that for any two points $\xb$ and $\yb$ we have
\[
    d ( \fb ( \xb ) , \fb ( \yb ) ) = r d ( \xb , \yb ) ,  \, 
\]
where $d(\xb,\yb)$ is the Euclidean distance from $\xb$ to $\yb$~\cite{smart1998modern}%
. The scalar $r$ is sometimes termed the ratio of similarity, the stretching factor and the similarity coefficient. When $r = 1$ a similarity is called an isometry (rigid transformation). Two sets are called similar if one is the image of the other under a similarity.

As a map $\fb : \RR^n \rightarrow \RR^n$, a similarity of ratio $r$ takes the form
\[
    \fb ( \xb ) = r \Ab \xb + \tb ,
\]
where $\Ab$ %
is a orthogonal matrixn $n \times n$ and $\tb \in \mathbb{R}^n$ is a translation vector.

Note that such a similarity $\fb$ has Jacobian $\Jb_\fb(\xb)=r\Ab$ for any $\xb$.

\paragraph{Conformal maps.} Conformal maps are angle preserving transformation, and in this sense, are a generalization of similarities. In short, let $U$ be an open subset of $\RR^n$, $\varphi:U\rightarrow \RR^n$ is a conformal map if, for two arbitrary curves $\gamma_1(t)$ and $\gamma_2 (t)$
on $\mathbb{R}^n$, where these curves intersect each other with angle $\theta$ in point $\pb\in U$, then  $\varphi \circ \gamma_1 (t)$ and $\varphi \circ \gamma_2 (t)$ intersect each other with the same
angle $\theta$ in the point $\varphi(\pb)$.

A characterisation of conformal maps directly related to orthogonal coordinate systems is the following.
\begin{proposition}[See e.g. \cite{soeten2011conformal}]
Let $U$ be an open subset of $\RR^n$ with a $C ^1$-function $\varphi  : U \rightarrow
\mathbb{R}^n$. Then $\varphi$ is conformal iff there exists a scalar function $\lambda : U \rightarrow \RR$ such that
$\lambda(\xb)^{-1} \Jb_{\varphi}(\xb)$ is an orthogonal matrix for all $\xb$ in $U$. We call $\lambda$ the scale factor of $\varphi$.
\end{proposition} 

While it can be shown that \textit{linear} conformal maps are similarities, an interesting class of \textit{nonlinear} conformal maps are the unit radius sphere inversion (restriction to unit radius is only to avoid unnecessary notational complexity):
\begin{eqnarray*}
I_\bb:  \mathbb{R}^n\setminus \{0\} &\rightarrow &\mathbb{R}^n\setminus\{0\}\\
 \xb&\mapsto & \frac{\xb-\bb}{\|\xb-\bb\|^2}+\bb
\end{eqnarray*}
We can notice that such transformation leaves the hypersphere of center $\bb$ and radius 1 invariant, while the points outside of the unit ball are mapped to the interior of the unit ball, and vice-versa.

Interestingly, conformal maps in Euclidean spaces of dimension superior or equal to 3 can be restricted to two kinds according to the following result from Liouville.

\begin{theorem}[see e.g. \cite{tojeiro2007liouville}]\label{thm:highdimmoebius}
Let $f : U \rightarrow \mathbb{R}^n$ be a conformal map defined on a connected open subset
of Euclidean space $\mathbb{R}^n$ of dimension $n \geq 3$. Then $f = L_{|U}$ can be written either as the restriction of a 
similarity $L$ to $U$, or as the composition $f = I \circ L_{|U}$ of such a map with an inversion with respect
to a hypersphere of unit radius, centered at the origin.
\end{theorem}
The class of function described in~\Cref{thm:highdimmoebius} corresponds exactly to the M\"obius transformations described in~\eqref{eq:moebius_transf}. These transformation can as well be defined in dimension $2$, with the specificity that they are only a subset of the class conformal maps in this dimension.

\paragraph{Properties of sphere inversion.}
We characterize the properties of the unit sphere centered at zero, that we denote $I$
\begin{eqnarray*}
I:  \mathbb{R}^n\setminus \{0\} &\rightarrow &\mathbb{R}^n\setminus\{0\}\\
 \xb&\mapsto & \frac{\xb}{\|\xb\|^2}
\end{eqnarray*}

\begin{comment}
First, let us compute the scale factor of this transformation, which, due to the above property, is given by
\[
\kappa(x)=\sqrt[n]{|J_I(x)|}\,
\]
This can be derived by noticing for $r>0$:
\[
I(rx)=\frac{1}{r}I(x)
\]
thus
\[
rJ_I(rx)=\frac{1}{r}J_I(x)
\]
which leads to 
\[
r^n|J_I(rx)|=\frac{1}{r^n}|J_I(x)|
\]
thus taking the n-th root,
\[
r\kappa(rx)=\frac{1}{r}\kappa(x)\,.
\]

By symmetry, we also notice $\kappa (x)$ should be constant on the unit sphere, with value $\kappa_1=1$ (this can be easily checked).
Then at any point at distance $r$ from the origin
\[
\kappa(x)=\frac{1}{r^{2}}\kappa_1=\frac{1}{\|x\|^{2}}\,.
\]
\end{comment}

Now let us derive the Jacobian of $I$. A straightforward computation leads to
\[
\Jb_I(\xb) =\frac{1}{\|\xb\|^2}\left( \Ib_n-2\frac{\xb \xb^\top}{\|\xb\|^2} \right)
\]
where $\Ib_n$ denote the identity matrix.

By noticing that  $\frac{\xb \xb^\top}{\|\xb\|^2}$ is rank one symmetric with eigenvalue 1 associated with unit norm eigenvector $\frac{\xb }{\|\xb\|} $, we can diagonalize this matrix in any (space dependent) orthogonal basis that has $\frac{\xb }{\|\xb\|} $ as the first basis vector. 

Let us thus pick the unit vectors associated to the hyperspherical coordinates (which satisfy this condition by definition), and consider the orthogonal matrix $\Bb(\frac{\xb }{\|\xb\|})$ gathering these basis vectors as its columns (it is parameterized by the unit vector $\frac{\xb }{\|\xb\|}$, as this basis is radially invariant.
Then we can write 
\[
\frac{\xb \xb^\top}{\|\xb\|^2}
=
\Bb\left(\frac{\xb }{\|\xb\|}\right)\Db \Bb\left(\frac{\xb }{\|\xb\|}\right)^\top
\]
and thus
\[
\Jb_I(\xb) =\frac{1}{\|\xb\|^2}\left( \Ib_n-2\Bb\left(\frac{\xb }{\|\xb\|}\right)\Db \Bb\left(\frac{\xb }{\|\xb\|}\right)^\top \right)
=\frac{1}{\|\xb\|^2}\Bb\left(\frac{\xb }{\|\xb\|}\right)\left( \Ib_n-2\Db \right)\Bb\left(\frac{\xb }{\|\xb\|}\right)^\top
\]
with $\Db$ a diagonal matrix with diagonal elements $[1,0,\dots,0]$.
This leads to
\[
\Jb_I(\xb) =\frac{1}{\|\xb\|^2}\Bb\left(\frac{\xb }{\|\xb\|}\right)\Db_I \Bb\left(\frac{\xb }{\|\xb\|}\right)^\top 
\]
with $\Db_I=\Ib_n-2\Db$ a diagonal matrix with diagonal elements $[-1,1,\dots,1]$.
The Jacobian thus takes the form predicted by the above proposition for conformal maps
\[
\Jb_I(\xb) =\lambda (\xb)\Ob\left(\frac{\xb }{\|\xb\|}\right) 
\]
with scale factor $
\lambda(\xb)=\frac{1}{\|\xb\|^{2}}$
and $\Ob(\frac{\xb }{\|\xb\|})=\Bb\left(\frac{\xb }{\|\xb\|}\right)\Db_I \Bb\left(\frac{\xb }{\|\xb\|}\right)^\top $ a space dependent orthogonal matrix%
, which has the additional property to be radially invariant for the specific case of sphere inversions.